\documentclass[12pt]{article}

\usepackage{natbib}
\usepackage{abstract}
\usepackage{latexsym}
\usepackage{amsmath}
\usepackage{amsfonts}
\usepackage{amssymb}
\usepackage{amsthm}
\usepackage{dsfont}
\usepackage{psfrag}
\usepackage{graphicx}
\usepackage{subcaption}
\usepackage[pdfstartview={FitH},bookmarks=false,colorlinks,citecolor=blue,urlcolor=blue,linkcolor=blue,breaklinks=true]{hyperref}
\usepackage{breakcites}
\usepackage{microtype}
\usepackage[letterpaper,centering]{geometry}
\usepackage{comment}
\usepackage{algorithm}
\usepackage{algcompatible}
\usepackage{algpseudocode}
\usepackage{multirow}
\usepackage{soul}
\graphicspath{ {./figures/} }

\title{Restricted Tweedie Stochastic Block Models}%\query{Q1}
\date{}
\author{Jie Jian \thanks{Corresponding author: j5jian@uwaterloo.ca} \hspace{1em}  Mu Zhu \hspace{1em}   Peijun Sang \\
    Department of Statistics and Actuarial Science, University of Waterloo}

\newcommand{\var}{\textnormal{var}}

\newcommand\bm[1]{\boldsymbol #1}
\DeclareMathAlphabet\mathbfcal{OMS}{cmsy}{b}{n}

\def\E{\mathbb{E}}
\def\P{\mathbb{P}}
\def\real{\mathbb{R}}

\def\scsum#1#2{\sum\limits_{#1}^{#2}}%JJ
\def\scprod#1#2{\prod\limits_{#1}^{#2}}%JJ

\DeclareMathOperator*{\argmax}{arg\,max} %JJ
\def\ind{\mathds{1}(z_i=k,z_j=l)}

\def\indCi{\mathds{1}(c_i=k)}
\def\indCiCj{\mathds{1}(c_i=k,c_j=l)}

\newtheorem{mycondition}{Condition}[section]
\newtheorem{Lemma}{Lemma}
\newtheorem{Remark}{Remark}
\newtheorem{theorem}{Theorem}

\usepackage{xcolor}
\newcommand{\myred}[1]{\textcolor{black}{#1}}
\newcommand{\JJ}[1]{\textcolor{black}{#1}}

\usepackage{setspace}
\usepackage{caption}
\begin{document}
\maketitle

% Abstract, keywords, and classification codes
\begin{abstract}
The stochastic block model (SBM) is a widely used framework for community detection in networks, \JJ{where the network structure is typically represented by an adjacency matrix}. However, conventional SBMs \JJ{are not directly applicable} \myred{to an adjacency matrix that consists of non-negative zero-inflated continuous edge weights}. To model the  international trading network, where edge weights represent trading values between countries, we propose an innovative SBM based on a restricted Tweedie distribution.
Additionally, we incorporate nodal information, \JJ{such as} the geographical distance between countries, and account for its dynamic effect on edge weights. Notably, we show that given a sufficiently large number of nodes, estimating \JJ{this covariate} effect \JJ{becomes} independent of community labels of each node when computing the maximum likelihood estimator of parameters in our model. This result enables \JJ{the development of} an efficient two-step algorithm that separates the estimation of covariate effects from other parameters.
\JJ{We demonstrate the effectiveness of our proposed method through extensive simulation studies and an application to real-world international trading data.}

%Extensive simulation studies are conducted to demonstrate the performance of our proposed model. 
%Furthermore, we apply our model to the international apple trading data. 

    \textbf{Keywords:} Stochastic block model, community detection, network analysis, compound Poisson-Gamma distributions, dynamic effects.
\end{abstract}

\section{Introduction}
\label{sec-sbm-intro}

\subsection{Background}
A community can be conceptualized as a collection of nodes that exhibit similar connection patterns in a network. 
Community detection is a fundamental problem in network analysis, with wide applications in social network \citep{bedi2016community}, marketing \citep{bakhthemmat2021communities}, recommendation systems \citep{gasparetti2021community}, and political polarization detection \citep{guerrero2017community}. 
%, whereby connectivity between two nodes is determined solely by their respective community affiliations.
%that are heavily connected within groups but sparsely connected between groups. 
Identifying communities in a network not only enables nodes to be clustered according to their connections with each other, but also reveals the hierarchical structure that many real-world networks exhibit. Furthermore, it can facilitate network data processing, analysis, and storage \citep{lu2018community}.

\JJ{Among the various methods for detecting communities in a network, the Stochastic Block Model (SBM) stands out as a probabilistic graph model. It is founded based on the \textit{stochastic equivalence} assumption, positing that the connecting probability between node $i$ and node $j$ depends solely on their community memberships \citep{holland1983stochastic}.}
%Among various approaches of detecting communities in a network, the Stochastic Block Model (SBM) is represented as a probabilistic graph that is built upon the \textit{stochastic equivalence} assumption that the connecting probability between node $i$ and node $j$ solely depends on their community memberships \citep{holland1983stochastic}. 
If we assume that given the community memberships of two nodes $i$ and $j$, denoted by $c_i$ and $c_j$, the edge weight between them is Bernoulli distributed. 
In particular, letting $Y_{ij}$ denote this weight, the adjacency matrix $Y = (Y_{ij})$ is generated as
\begin{equation} \label{eq-VSBM}
    Y_{ij}\mid c_i=k,c_j=l \sim \text{Bernoulli}(B_{kl}),
\end{equation}
where $B_{kl}$ denotes the probability of connectivity between the nodes from the $k$th and $l$th communities. 

%As indicated in the above expression, the SBM is built upon a probabilistic framework, which makes rigorous statistical inference feasible, as we can quantify the uncertainty in the estimated communities. 
%and make principled decisions based on the data. 
%Research on theoretical performances of the SBM becomes increasingly popular; see \cite{lee2019review} for example. 
As indicated in \eqref{eq-VSBM}, an SBM provides an interpretable representation of the network's community structure. %The estimated parameters of the SBM can be used to identify the most important nodes or communities, and to understand the relationships between them. 
Moreover, an SBM can be efficiently fitted with various algorithms, such as maximum likelihood estimation and Bayesian inference \citep{lee2019review}. In recent few years, there has been extensive research on theoretical properties of the estimators obtained from these algorithms \citep{lee2019review}. 

%% JJ add the motivating example here:
% \JJ{Given the remarkable capability of the SBM in detecting the latent community structure, we are motivated to leverage its strengths to tackle an interesting problem: clustering countries into different groups based on their international trading patterns. However, when we apply the existing SBM in international trading application, there are still three fundamental challenges.}

\JJ{In this paper, we are motivated to leverage the remarkable capability of the SBM in detecting latent community structures to tackle an interesting problem---clustering countries into different groups based on their international trading patterns. However, in this application, we encounter three fundamental challenges that can not be addressed by existing SBM models.}

\subsection{Three main challenges}
\label{sec:three}

\subsubsection{Edge Weights}

\JJ{The classical SBM, as originally proposed by \cite{holland1983stochastic}, is primarily designed for binary networks, as indicated in \eqref{eq-VSBM}. However, in the context of the international trading network, we are presented with richer data, encompassing not only the presence or absence of trading relations between countries but also the specific trading volumes in dollars. These trading volumes serve as the intensity and strength of the trading relationships between countries. In such cases, thresholding the data to form a binary network would inevitably result in a loss of valuable information. }

\JJ{In the literature, several methods have been developed to extend the modelling of edge weights beyond the binary range. 
%These methods can be broadly categorized into three types. The first category employs Bayesian frameworks. \cite{aicher2013adapting, aicher2015learning} adopt a Bayesian approach, modeling edge weights using distributions from the exponential family. They express the prior distribution as a product of parameterized conjugate distributions. \cite{ludkin2020inference} takes a different approach by allowing for arbitrary distributions in modeling edge weights. They sample the posterior distribution using a reversible jump Markov Chain Monte Carlo (MCMC) method.
Some methods leverages distributions capable of handling edge weights. 
%In contrast, the second type of methods leverages distributions capable of handling valued edges. 
For instance, \cite{aicher2013adapting, aicher2015learning} adopt a Bayesian approach to model edge weights using distributions from the exponential family. \cite{ludkin2020inference} allows for arbitrary distributions in modeling edge weights and sample the posterior distribution using a reversible jump Markov Chain Monte Carlo (MCMC) method. \cite{ng2021weighted} and \cite{NathanielStevens2021} use a compound Bernoulli-Gamma distribution and a Hurdle model to represent edge weights respectively. \cite{haj2022estimation} apply the binomial distribution to networks with integer-valued edge weights that are bounded from above. 
In contrast, there is a growing interest in multilayer networks, where edge weights are aggregated across network layers. Notable examples of research in this area include the work by \cite{macdonald2022latent} and \cite{chen2022community}.}

\JJ{However, the above approaches cannot properly deal with financial data that involve non-negative continuous random variables with a large number of zeros and a right-skewed distribution. }

\subsubsection{Incorporating nodal information}

\JJ{Many SBMs assume that nodes within the same community exhibit stochastic equivalence. However, this assumption can be restrictive and unrealistic, as real-world networks are influenced by environmental factors, individual node characteristics, and edge properties, leading to heterogeneity among community members that affects network formation.
%Many SBMs commonly employ a strong assumption that all nodes within the same community exhibit stochastic equivalence. This presumption, though widely adopted, can be considered somewhat restrictive and unrealistic. This is because in reality the formation of a network is influenced not solely by community labels but also by environmental factors, individual node characteristics, or edge properties, which can introduce heterogeneity among community members, subsequently impacting the formation of network. 
Depending on the relationship between communities and covariates, there are generally three classes of models, as shown in Figure~\ref{fig:sbm_3relationships}. Models (b) and (c) have been previously discussed by \cite{huang2023pairwise}. 
We are also particularly interested in model (c), where latent community labels and covariates jointly shape the network structure. In our study on international trading networks, factors such as the geographical distance between countries, along with community labels, play critical roles in shaping trading relations. Neglecting these influential factors can significantly compromise the accuracy of SBM estimations.}

\begin{figure}[ht]
     \centering
     \begin{subfigure}[b]{0.3\textwidth}
         \centering
         \includegraphics[width=0.9\textwidth]{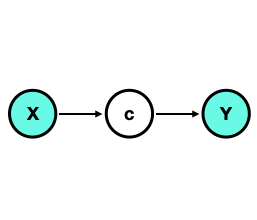}
        \caption{Covariates-driven}
        \end{subfigure}
     \hfill
     \begin{subfigure}[b]{0.3\textwidth}
         \centering
\includegraphics[width=0.8\textwidth]{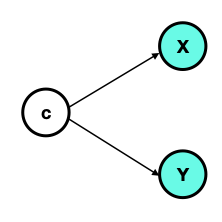}
         \caption{Covariates-confounding}
     \end{subfigure}
     \hfill
     \begin{subfigure}[b]{0.3\textwidth}
         \centering
\includegraphics[width=0.8\textwidth]{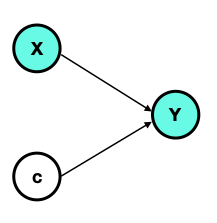}
         \caption{Covariates-adjusted}
     \end{subfigure}
        \caption{Three network models with covariates. The symbols $X$, $Y$ and $c$ represent covariates, network connection and community memberships, respectively. A shaded/unshaded cell means the corresponding quantity is observable/latent.} 
         \label{fig:sbm_3relationships}
\end{figure}

%The strong presumption that all nodes in the same community are stochastically equivalent is adopted in most SBMs. 
% However, this assumption ignores heterogeneity in community members caused by nodal information, which could be predictive of the community assignment. \JJ{(?)}

Various works in the past have considered the incorporation of nodal information. For instance, \cite{roy2019likelihood} and \cite{choi2012stochastic} considered a pairwise covariate effect in the logistic link function when modelling the edge between two nodes. In contrast, \cite{ma2017exploration} and \cite{hoff2002latent} incorporated the pairwise covariate effect but with a latent space model. Other research considering covariates in an SBM includes \cite{tallberg2004bayesian}, \cite{vu2013model} and \cite{peixoto2018nonparametric}. Moreover, \cite{mariadassou2010uncovering} and \cite{huang2023pairwise} addressed the dual challenge of incorporating the covariates and modeling the edge weights by assuming that each integer-valued edge weight follows a Poisson distribution and accounting for the pairwise covariates into the mean. 
 %To better understand how changes in the environment or from individual traits impact the network structures, incorporating nodal information in SBM breaks the restriction brought by only using the connection information as inputs. 

 \JJ{While the aforementioned literature has made significant progress in incorporating covariate information into network modeling, the complexity escalates when we confront the third challenge — the observed network is changing over time. This challenge necessitates a deeper exploration of how covariates influence network formation dynamically — a facet that remains unaddressed in the existing literature.}

\subsubsection{Dynamic network}

\JJ{Recent advances in capturing temporal network data demand the extension of classic SBMs to dynamic settings, as previous research predominantly focused on static networks. }

\JJ{Researchers have attempted to adapt SBMs to dynamic settings, employing various strategies such as state-space models, hidden Markov chains, and change point detection.
% Given an adjacency matrix that evolves over time, recent model-based approaches to identifying the community structures include state-space models, Markov chains and change point detection. 
 \cite{fu2009dynamic} and \cite{xing2010state} extended a mixed membership SBM for static networks to dynamic networks by characterizing the evolving community memberships and block connection probabilities with a state space model. 
Both \cite{yang2011detecting} and \cite{xu2014dynamic} studied a sequence of SBMs, where the parameters were dynamically linked by a hidden Markov chain. \cite{Matias2017} applied Markov chains to the evolution of the node community labels over time. \cite{Bhattacharjee2020} proposed a method to detect a single change point such that the community connection probabilities are different constants within the two intervals separated by it. 
\cite{xin2017continuous} characterized the occurrence of a connection between any two nodes in an SBM using an inhomogeneous Poisson process. 
\cite{zhang2020TVN} proposed a regularization method for estimating the network parameters at adjacent time points to achieve smoothness. }

\subsection{Our contributions}

The main contribution of this paper is to extend the classical SBM to address the three challenges mentioned above.
%%BEGIN MZ rewrite: MZ rewrites below ...
%%We consider generalizing the assumption that given the node communities each edge follows a Bernoulli distribution to a compound Poisson-Gamma distribution and incorporating the nodal information in a dynamical manner. This generalized distribution can accommodate edges that can take zero or any positive continuous values. Then we apply the proposed model to an international trading network, where each edge between two countries represents their trading value in dollar, rendering the proposed model more appropriate than the classical one. %Successful answers to the latter two individual challenges have already been provided separately. However, if we assume that the effect of the covariate in the network is time-varying, then these two seemingly independent challenges can be solved simultaneously. 
%%to ...
Given the community membership of each node, we generalize the assumption that edges in the network follow Bernoulli distributions to that they follow compound Poisson-Gamma distributions instead (Section~\ref{sec-TWsbm-meth}). This allows us to model edges that can take on any non-negative real value, including exactly zero itself. Later in Section~\ref{sec:application}, we apply the proposed model to an international trading network, where each edge between two countries represents the dollar amount of their trading values, for which our model is more appropriate than the classical one.
%%END MZ rewrite
%%BEGIN MZ rewrite: 
%%To address the last challenges introduced in Section \ref{sec:three}, we take the nodal information into consideration and assume that the effect of the nodal covariates in the network is time-varying.
%%By making this assumption, we achieve a time-varying network through the dynamic effect of nodal covariates, and employ smoothing splines to estimate the dynamic coefficients. %We will combine their solutions and apply them to a weighted time-varying network where each link has a either zero or positive continuous value associated with it and the value changes over time.
Moreover, not only do we incorporate nodal information in the form of covariates, we also allow the effects of these covariates to be time-varying (Section~\ref{sec-TWsbm-meth}). 
%%END MZ rewrite

%MZ minor edits below, without specific notes
We use a variational approach (Section~\ref{sec:estimation}) to conduct statistical inference for such a time-varying network. We also prove an interesting result (Section~\ref{sec:theory}) that, asymptotically, the covariate effects in our model can be estimated irrespective of how community labels are assigned to each node. This result also allows us to use an efficient two-step algorithm (Section~\ref{sec:estimation}), separating the estimation of the covariate effects and that of the other parameters---including the unknown community labels. A similar two-step procedure is also used by \cite{huang2023pairwise}. 

%In this paper, we develop statistical methods to characterize the dynamic changes in discrete time-varying weighted networks via the stochastic block model by incorporating the pairwise nodal covariates information. We consider a dynamic network under the discrete setting where the community label of each node is fixed along the whole time interval, while the connection probability of two nodes is changing over time. We assume that the connection probability is composed of two parts, a constant part whose value only depends on the communities of the dyad and a term including the node-pair covariate which captures the evolution of the network. The effect of the node-pair covariate on the connection probability is time-varying, and is modeled by the varying-coefficient framework. The objective of our model is to infer the community structure of the networks and estimate the model parameters.

%%BEGIN MZ comment out: MZ find most of what follows has already been said above
%%The rest of the paper is organized as follows. In Section~\ref{sec-TWsbm-meth}, we develop the restricted Tweedie SBM. In Section~\ref{sec:theory}, we \red{prove} that, asymptotically as the sample size goes to infinity, the estimation of covariate coefficients does not depend on community labels, which leads us to an efficient, two-step algorithm in Section~\ref{sec:estimation}. In Section~\ref{sec:simulation}, we study the performance of model with simulation. Finally in Section~\ref{sec:application}, we apply our model to analyze some international trading data .
%%END MZ comment out

\section{Methodology}
\label{sec-TWsbm-meth}

%%MZ edits below ... various, but only highlghting key ones with red color
In this section, we first give a brief review of a rarely-used distribution, the Tweedie distribution, which can be used to model network edges with zero or positive continuous weights. Next, we propose \myred{a general SBM} using the Tweedie distribution \myred{in three successive steps}, each addressing a challenge mentioned in Section \ref{sec:three}. 
%Because Tweedie has good ways to be reparameterized which makes it suitable for generalized linear model, and it has well-established computational techniques. These makes Tweedie distribution a [good] tool to adjust/loose the restriction in the classic SBM, based on which we propose three extensions of classic SBM.
%Using Tweedie distribution, we can model the stochastic block models when the edge values in the network is either 0 or positive continuous.
More specifically, 
we start with a vanilla model, a variation of the classic SBM where each edge value between two nodes now follows the Tweedie distribution rather than the Bernoulli distribution. We then incorporate covariate terms into the model, before we finally arrive at a time-varying version of the model by allowing the covariates to have dynamic effects that change over time.

\subsection{Tweedie distribution}
\label{subsec-TW}

Let $N$ be a random variable following the Poisson distribution with mean $\lambda$. Conditional on $N = n$, $Z_1, \ldots, Z_n  \overset{iid}{\sim} \text{Gamma} (\alpha,\gamma)$. Define 
% $$
%     Y=\begin{cases}
%         0, & \rm{if}  N=0,\\
%         Z_1+Z_2+\cdots+Z_N, & \rm{if} N=1,2,3,\cdots.\\
%     \end{cases}
% $$
\begin{equation*}
    Y = \left\{\begin{aligned}
0, \quad\quad\quad\quad\quad\quad\quad\,\,\,\,&\quad \rm {if}\quad N=0,\\
        Z_1+Z_2+\cdots+Z_N, &\quad \rm{if}\quad N=1,2,3,\cdots.
\end{aligned}\right.
\end{equation*}
Then, $Y$ has a compound Poisson-gamma distribution, with a nonzero probability mass at $0$. As $Y=0$ if and only if $N=0$, $\P(Y=0)=\P(N=0)=\exp{(-\lambda)}$. Conditional on $N=n>0$, $Y$ follows a gamma distribution with mean $n\alpha\gamma$ and variance $n\alpha\gamma^2$. 
%%MZ further condensed the above, and added some motivating context below
\myred{In the context of international trading (also see Section~\ref{sec:application} below), $N$ may be the number of trades in a given year; $Z_1, \dots, Z_N$ may be the dollar amount of each trade; then, $Y$ is the simply total trading amount from that year.}
%%END MZ addition

The compound Poisson-gamma distribution, known as a special case of the Tweedie distribution \citep{tweedie1984index}, is related to an exponential dispersion (ED) family. If $Y$ follows an ED family distribution with mean $\mu$ and variance function $V$, then $Y$ satisfies $\var (Y)=\phi V(\mu)$ for some dispersion parameter $\phi$. The Tweedie distribution belongs to the ED family with $V(\mu)=\mu^\rho$ for some constant $\rho$. Specified by different values of $\rho$, the Tweedie distribution includes the normal ($\rho=0$), the gamma ($\rho=2$) and the inverse Gaussian distribution ($\rho=3$), and the scaled Poisson distribution ($\rho=1$). Tweedie distributions exist for all values of $\rho$ outside the interval $(0,1)$. Of special interest to us here is the \myred{restricted} Tweedie distribution with $1<\rho<2$, which is the aforementioned compound Poisson–gamma distribution with a positive mass at zero but \JJ{a continuous distribution of positive values elsewhere.}
%continuously-valued elsewhere. 
\myred{We add the word ``restricted'' to describe the Tweedie distribution when $\rho$ is constrained to lie on the interval $(1,2)$; it will become clearer later in Section~\ref{sec:estimation} that this particular restriction also simplifies the overall estimation procedure somewhat.}

%\blue{JJ: Most paper on the compound Poisson-gamma model first introduces the exponential dispersion family, and then claims that Tweedie is a special member of ED family, and Poisson-gamma distribution is a special case of Tweedie distribution when $\rho\in(1,2)$. I don't understand why ED family is necessary information here (maybe due to the special expectation and variance relation?) rather than giving the formula of Tweedie distribution directly, but I still introduce it.}

\myred{Specifically, the aforementioned compound Poisson–gamma distribution with parameters $(\lambda,\alpha,\gamma)$ can be reparameterized as a restricted Tweedie distribution, with parameters $(\mu,\phi,\rho)$ satisfying $1 < \rho < 2$ and the following relationships:}
\begin{align*}
   \lambda=\dfrac{\mu^{2-\rho}}{\phi(2-\rho)},
   \quad \alpha=\dfrac{2-\rho}{\rho-1},
   \quad \gamma=\phi(\rho-1)\mu^{\rho-1}.
\end{align*}
That is, the marginal \myred{distribution} of $Y$, defined above, can be expressed as 
\begin{align}
    f(y|\mu,\phi,\rho) = a(y,\phi,\rho)\cdot \exp\left\{\frac{1}{\phi} \left( \dfrac{y\mu^{1-\rho}}{1-\rho} - \dfrac{\mu^{2-\rho}}{2-\rho} \right)\right\},
    \quad 1 < \rho < 2, \label{Twdist}
\end{align}
where
% \begin{align*}
% a(y,\phi,\rho) = \begin{cases}
%     \dfrac{1}{y} \scsum{j=1}{\infty} \dfrac{y^{j\alpha}}{(\rho-1)^{j\alpha} \phi^{j(1+\alpha)} (2-\rho)^j j! \Gamma(j\alpha) }, \text{ for }y>0,\\
%     1, \text{ for }y=0.
% \end{cases}
% \end{align*}
\begin{equation*}
 a(y,\phi,\rho) = \left\{\begin{aligned}
\dfrac{1}{y} \scsum{j=1}{\infty} \dfrac{y^{j\alpha}}{(\rho-1)^{j\alpha} \phi^{j(1+\alpha)} (2-\rho)^j j! \Gamma(j\alpha) },  &\quad \rm { for }\quad y>0,\\
1,\quad\quad\quad\quad\quad\quad\quad\quad\quad\quad\quad\quad\quad\quad\quad\,\, &\quad \rm{ for } \quad y=0.
\end{aligned}\right.
\end{equation*}

%\subsection{Vanilla Tweedie SBM}
\subsection{Vanilla model}
\label{subsec-vanillaTWsbm}

%%MZ deleted this paragraph below.
%%We propose a new model for the weighted stochastic block model using the Tweedie distribution, henceforth referred to as the \textbf{vanilla Tweedie SBM}. Before discussing the model, we introduce some notation first. 

%%Other than various minor edits, MZ also find the original  presentation logic confusing and changed the orders of things below.
Let $\mathcal{G}=(V,E)$ denote a weighted graph, where $V$ denotes a set of nodes with cardinality $|V|=n$ and $E$ denotes the set of edges between two nodes. 
For SBMs, each node in the network can belong to one of $K$ groups. Let $c_{i}\in \{1,\cdots,K \}$ denote the unobserved community membership of node $i$ and $c_{i}$ follows a multinomial distribution with the probability $\pi=(\pi_1,\cdots,\pi_K)$.

Usually, the set $E$ is represented by an $n \times n$ matrix $Y=[y_{ij}]\in \mathbb{R}^{n\times n}$. In classical SBMs, each $y_{ij}$ is modelled either as a Bernoulli random variable taking on binary values of 0 or 1, or as a Poisson random variable taking on non-negative \emph{integer} values. We first relax this restriction by allowing $y_{ij}$ to take on non-negative \emph{real} values. Since we focus on an undirected weighted network without self-loops, $Y$ is a (for us, non-negative) real-valued symmetric matrix with zero diagonal entries. 

Given the observed data set $D=\{ y_{ij} \}_{1\leq i<j\leq n}$, we assume that each edge value $y_{ij}$ follows a \myred{restricted} Tweedie distribution with power \myred{$\rho \in (1,2)$} and dispersion $\phi$: 
\begin{align}
    y_{ij}\sim \text{Tw}(\mu_{ij},\phi,\rho), 
    \quad \myred{1 < \rho < 2}, \label{TwModel}
\end{align}
where the mean $\mu_{ij}$ is modelled as a positive constant determined by the latent community label of nodes $i$ and $j$ through a log-link function, i.e.,
\begin{align}
\log (\mu_{ij})=\beta_0^{kl},
\quad\text{if}\quad c_i=k \quad\text{and}\quad c_j=l, \label{TSBMV_mu}
\end{align}
where $\beta_0 = [\beta_0^{kl}] \in \real^{K\times K }$ is a symmetric matrix. 
For a constant model, the log-link may not appear to be necessary, but it will become more useful later on as we incorporate covariates into this baseline model. 
\subsection{Model with covariates}
\label{subsec-TWsbm}

%%MZ also clarified there are multiple covariate matrices---this aspect was not at all clear in the old notation!
In many real-life situations, we observe additional information about the network. For example, in addition to the relative existence or importance of each edge, a collection of $p$ symmetric covariate matrices $\myred{X^{(1)},...,X^{(p)}} \in \real^{n\times n}$ may also be available, where the $(i,j)$-th entry \myred{$x_{ij}^{(u)}$ of each $X^{(u)}$} represents a pair-wise covariate containing some information about the connection between node $i$ and node $j$, and $x_{ii}^{(u)}=0$ for all $1\leq i \leq n$ and $u=1,...,p$. 
Given a data set $D=\{Y, \myred{X^{(1)}, ..., X^{(p)}}\}$, the vanilla model from Section~\ref{subsec-vanillaTWsbm} above can be easily extended by replacing \eqref{TSBMV_mu} with  
\begin{align}
\log (\mu_{ij})=\beta_0^{kl} + \bm{x}_{ij}^\top \bm{\beta}, 
\quad\text{if}\quad c_i=k \quad\text{and}\quad c_j=l, \label{TSBMC_mu}
\end{align}
so that $\mu_{ij}$ is affected not only by the community labels $c_i, c_j$ but also by the covariates contained in $\bm{x}_{ij}$. Here, both \myred{$\bm{x}_{ij} \equiv (x_{ij}^{(1)},...,x_{ij}^{(p)})^{\top}$} and $\bm{\beta}$ are $p$-dimensional vectors.

\subsection{Time-varying model}
\label{subsec-TVTWsbm}

%%MZ notes some bad word usages before: evolutionary vs evolving, serial vs series.  
Now suppose we observe an evolving network at a series of $T$ discrete time points $\{t_1,\cdots,t_T \}$, with a common set of $n$ nodes. Specifically, our data set is of the form $D=\{ Y(t_1), \dots, Y(t_T); X^{(1)}, \dots, X^{(p)}\}$. Without loss of generality, we may assume each $t_{\nu} \in [0, 1]$. 

%%Begin MZ comment out
%%Similar to the two static models mentioned in Section 2.3 and 2.4, each node $i$ is associated with a time-invariant latent community label $c_i$, and each node pair $i-j$ has a $p-$dimensional time-invariant covariate $x_{ij}$. The data we observe include the temporal networks represented by their edge weights matrices and the time-invariant covariate $D=\{ Y(t), t=t_1,\cdots,t_T, X \}$. 
%%END MZ comment out

To model such data, we assume in this paper that the latent community labels $c_1, \dots, c_n$ are fixed over time but allow the covariate effects to change over time by incorporating a varying-coefficient model. \myred{In reality, the community labels may also change over time, but a fundamentally different set of tools will be required to model these changes and we will study them separately---not in this paper. Here, we simply assume that model \eqref{TwModel} holds pointwise at every time point $t$}, i.e.,
\begin{align}
    y_{ij}(t) \sim \text{Tw}(\mu_{ij}(t),\phi,\rho),
    \quad \myred{1 < \rho < 2},
    \label{TVTwModel}
\end{align}
%%BEGIN MZ comments: MZ again finds most of this below to be unnecessary ...
%%We adopt varying-coefficient models to characterize the time-varying effect of the covariate $X$ and assume the intercept as fixed over time. \blue{For $u=1,\cdots,p$, the $u^{th}$ component in the static $p-$dimensional coefficient $\beta$ in Section~\ref{subsec-TWsbm} is extended to a smooth function $\beta_u (t)$ over $t \in [0, 1]$, which represents the dynamic coefficient of the $u$th covariate. To enhance clarity in our notation, we differentiate the vector form and function form by employing an overbar exclusively for the vector representation. In the subsequent sections of this paper, $\overline{\beta(t)}$ represents a vector of p scalars, with each element signifying the value of $\beta_u(t)$ at time point $t$.}
%%END MZ comments
and
\begin{align}
%%    \log \{\mu_{ij} (t)\} =\beta_0^{kl} + x_{ij}^\top \overline{\beta(t)} ,\text{ if }c_i=k\text{ and }c_j=l
%%MZ dislikes overbar notation above, currently experimenting with the following:
%    \log \{\mu_{ij} (t)\} =\beta_0^{kl} + \underbrace{%
%    \sum_{u=1}^p x_{ij}^{(u)} \beta_u(t)}_{\equiv h_{ij}(t)}, 
    \log \{\mu_{ij} (t)\} =\beta_0^{kl} + \bm{x}_{ij}^{\top} \bm{\beta}(t), 
    \quad\text{if}\quad c_i=k \quad\text{and}\quad c_j=l, \label{TSBMTV_mu}
\end{align}
where $\bm{\beta}(t) \equiv (\beta_1(t),\dots,\beta_p(t))^{\top}$ and each $\beta_{u}(t)$ is a smooth function of time. The full likelihood function corresponding to our time-varying model \eqref{TVTwModel}--\eqref{TSBMTV_mu} is given by
\begin{multline}
L(\beta_0,\bm{\beta}(t),\pi,\phi,\rho; D,c) = \scprod{i=1}{n} \scprod{k=1}{K} \pi_k^{\indCi} \scprod{\nu =1}{T} \scprod{1\leq i < j\leq n}{} \scprod{k,l=1}{K} 
\biggl[ a(y_{ij}(t_{\nu}),\phi,\rho) \times \\ 
\exp\left\{\dfrac{1}{\phi} \left( \dfrac{y_{ij}(t_{\nu}) \exp[(1-\rho)\{\beta_0^{kl}+ \bm{x}_{ij}^\top \bm{\beta}(t_{\nu})\} ]}{1-\rho} - \right.\right. \\
\left.\left.
 \dfrac{\exp[(2-\rho)\{\beta_0^{kl} + \bm{x}_{ij}^\top \bm{\beta}(t_{\nu})\} ]}{2-\rho}\right)\right\} \biggr]^{\indCiCj}. \label{TVlikelihood}
\end{multline}
The likelihood functions for the earlier, simpler models---namely, the vanilla model in Section~\ref{subsec-vanillaTWsbm} and the static model with covairates in Section~\ref{subsec-TWsbm}---are simply special cases of \eqref{TVlikelihood}.

%%BEGIN MZ edits
%%\section{Theory of estimating the covariates coefficients}
\section{Theory}
\label{sec:theory}

%\myred{[MZ has revised this section]}

The resulting log-likelihood based on \eqref{TVlikelihood} contains three additive terms: the first involves only $\pi$; the second involves only $(\phi,\rho)$; and the third is the only one that involves both $\beta_0$ and $\bm{\beta}(t)$. 
Define
\begin{multline}
\ell_n(\bm{\beta}(t), \phi_0, \rho_0; D, z) \equiv \frac{1}{\binom{n}{2}}
\scsum{\nu =1}{T} \scsum{1\leq i < j\leq n}{} \scsum{k,l=1}{K} \dfrac{\ind}{\phi_0} \times \\
\left( 
\dfrac{y_{ij}(t_{\nu}) \exp[(1-\rho_0)\{\hat{\beta}_0^{kl}(\bm{\beta}(t_{\nu}))+ \bm{x}_{ij}^\top \bm{\beta}(t_{\nu})\} ]}{1-\rho_0} - \right. \\
\left.
\dfrac{\exp[(2-\rho_0)\{\hat{\beta}_0^{kl}(\bm{\beta}(t_{\nu})) + \bm{x}_{ij}^\top \bm{\beta}(t_{\nu})\} ]}{2-\rho_0}\right) \label{eq:profileloglik}
\end{multline}
to be the aforementioned third term \emph{after having}
\begin{itemize}
    \item replaced the unknown labels $c=(c_1,\dots,c_n)$ with an arbitrary set of labels $z=(z_1,\dots,z_n)$, where each $z_i$ is independently multinomial$(p_1,\dots,p_K)$;
    \item profiled out the parameter $\beta_0$ by replacing it with $\hat{\beta}_0(\bm{\beta}(t))$, while presuming $\phi=\phi_0$ and $\rho=\rho_0$ to be known and fixed; and
    \item re-scaled it by the total number of pairs, $\binom{n}{2}$.
\end{itemize}
This quantity turns out to be very interesting. Not only does $\hat{\beta}_0(\bm{\beta}(t))$ have an explicit expression, but \eqref{eq:profileloglik} can also be shown to converge to a quantity \emph{not} dependent on $z$ as $n$ tends to infinity. 

In other words, it does \emph{not} matter that $z$ is a set of \emph{arbitrarily} assigned labels! This has immediate computational implications (see Section~\ref{sec:estimation}).
Some high-level details of this theory are spelled out below in Section~\ref{subsec:theory-detail}, while actual proofs are given in the Appendix. 

\subsection{Details}
\label{subsec:theory-detail}

%\myred{[For this subsection, MZ simply extracted various materials---selectively---from the old section, with some slight reorganization and edits.]}

To simplify the notation, we first define two population parameters,
\begin{align*}
  \theta  = \scsum{\nu =1}{T} \E [y_{ij} (t_{\nu}) \exp \{(1-\rho_0)\bm{x}_{ij}^\top \bm{\beta}(t_{\nu})\} ] \quad\text{and}\quad
  \gamma = \scsum{\nu =1}{T} \E [\exp \{(2-\rho_0) \bm{x}_{ij}^\top \bm{\beta}(t_{\nu})\} ].
\end{align*}
For these to be properly defined, we require the following two conditions, which are fairly standard and not fundamentally restrictive. 

\begin{mycondition}\label{cond1}
The covariates $\{ \bm{x}_{ij}, 1\leq i<j \leq n \}$ are i.i.d., and there exists some $\alpha >0$ such that $\P(\exp\{\bm{x}_{ij}^\top \bm{u}\} \geq \delta ) \leq 2\exp(-\delta^2/\alpha)$ 
for any  $\delta > 0$, $i \neq j$ and $\bm{u} \in \mathbb{R}^p$ satisfying $\|\bm{u} \|_2 = \sqrt{u_1^2 + \dots + u_p^2} = 1$. 
\end{mycondition}
 
\begin{mycondition}\label{cond2}
The function $\beta_u(t)$ is continuous on $[0, 1]$, for all $u = 1, \ldots, p$. 
\end{mycondition}

The corresponding empirical versions of $\theta$ and $\gamma$ between any two groups, $k$ and $l$, according to an arbitrary community label assignment, $z$, are given by
\begin{align*}
 \hat{\theta}_{kl}  &=\dfrac{1}{\binom{n}{2}} \scsum{\nu =1}{T} \scsum{1\leq i < j \leq n}{}y_{ij}(t_{\nu}) \exp [(1-\rho_0)\bm{x}_{ij}^\top \bm{\beta}(t_{\nu})]\ind,\\
  \hat{\gamma}_{kl} &=\dfrac{1}{\binom{n}{2}} \scsum{\nu =1}{T} \scsum{1\leq i < j \leq n}{} \exp [(2-\rho_0) \bm{x}_{ij}^\top \bm{\beta}(t_{\nu})]\ind.
\end{align*}
We can then establish the following main theorem.

\begin{theorem}{Theorem 1.}{}\label{theorem1_1}
As $n \rightarrow \infty$ while $K$ remains constant, %\myred{[MZ asks: does $K$ really need to be a constant, or is $K=o(n), o(n^{1/2})$, ... enough?]} 
\begin{multline}
\ell_n (\bm{\beta} (t), \phi_0, \rho_0; D, z) = 
\dfrac{1}{\phi_0} \dfrac{1}{(1-\rho_0)(2-\rho_0)}  \scsum{k,l=1}{K} \hat{\theta}_{kl}^{2-\rho_0}\cdot \hat{\gamma}_{kl}^{\rho_0-1} \\
= \dfrac{1}{\phi_0} \dfrac{1}{(1-\rho_0)(2-\rho_0)} \theta^{2-\rho_0}\cdot \gamma^{\rho_0-1}+o_p(1). \label{eq:theorem}
\end{multline}
\end{theorem}

\begin{Remark}
So far, we have simply written 
$\hat{\theta}_{kl}$, $\hat{\gamma}_{kl}$, $\theta$ and $\gamma$
in order to keep the notation short. To better appreciate the conclusion of the theorem, however, it is perhaps important for us to emphasize here that these quantities are more properly written as 
$\hat{\theta}_{kl}(\bm{\beta}(t),\rho_0;D,z)$, 
$\hat{\gamma}_{kl}(\bm{\beta}(t),\rho_0;D,z)$, 
$\theta(\bm{\beta}(t),\rho_0;D)$, and
$\gamma(\bm{\beta}(t),\rho_0;D)$.
\end{Remark}

The implication here is that, asymptotically, our inference about $\bm{\beta}(t)$ is not affected by the community labels---nor is it affected by the total number of communities, $K$, since $z$ can follow \emph{any} multinomial$(p_1,\dots,p_K)$ distribution, including those with some $p_k=0$. Thus, even if we got $K$ wrong, our inference about $\bm{\beta}(t)$ would still be correct.

\section{Estimation Method}
\label{sec:estimation}

\subsection{Two-step Estimation}
\label{subsec:computation}

%%BEGIN MZ edit
In this section, we outline an algorithm to fit the restricted Tweedie SBM.
%Because the TV-TSBM is deemed as the most encompassing model, and the vanilla Tweedie SBM and the static Tweedie SBM with covairates are special cases of the TV-TSBM, their estimations can be done by slightly adjusting the estimation presenting in this section. 
%%BEGIN MZ comment out 
%%Since the vanilla Tweedie SBM and the static Tweedie SBM with covairates are special cases of the more comprehensive model TV-TSBM, their estimations can be derived by making slight adjustments to the estimation of TV-TSBM. 
%%END MZ comment out
%%
%%MZ moved materials about \rho here (Doesn't this make more logical sense?) and did some rewriting
Since, for us, the parameter $\rho$ is restricted to the interval $(1,2)$, we find it sufficient to simply perform a grid search~\citep[e.g.,][]{dunn2005series, dunn2008evaluation, lian2023tweedie} over an equally-spaced sequence, say, $1<\rho_1<\dots<\rho_m<2$, to determine its ``optimal'' value. However, our empirical experiences also indicate that a sufficiently accurate estimate of $\rho$ is important for making correct inferences on other quantities of interest, including the latent community labels $c$.

%In this section, we provide a comprehensive exposition on the estimation of the TV-TSBM, which is deemed as the most encompassing model. Furthermore, the estimations of the vanilla Tweedie SBM and the static Tweedie SBM with covairates can be illustrated as special cases.

For any given $\rho_0$ in a pre-specified sequence/grid, we propose an efficient two-step algorithm to estimate the other parameters. 
%%BEGIN MZ rewrite ... ``illustrate the Step 1'', ``demonstrate the Step 2'' are really terrible English
%%In Section~\ref{subsubsec:step1}, we illustrate the Step 1 to estimate the coefficients $\bm{\beta}(t)$ for covariates. In Section~\ref{subsubsec:step2}, we demonstrate the Step 2 to estimate the SBM parameters $c$, $\beta_0$, $\pi$, and the Tweedie parameter $\phi$. 
In Step 1 (Section~\ref{subsubsec:step1}), we obtain an estimate
$\hat{\bm{\beta}}_{\rho_0}(t)$ of $\bm{\beta}(t)$ using an arbitrary set of community labels. This is made possible by the theoretical result earlier in Section~\ref{sec:theory}. 
In Step 2 (Section~\ref{subsubsec:step2}), we obtain estimates of 
the remaining parameters parameters---$\hat{\beta}_0(\rho_0), \hat{\pi}(\rho_o), \hat{\phi}(\rho_0)$---while keeping $\hat{\bm{\beta}}_{\rho_0}(t)$ fixed. 
%%END MZ rewrite
The optimal $\rho$ is then chosen to be
\begin{align*}
\hat{\rho}=\argmax\limits_{\rho_0\in \{\rho_1,\cdots,\rho_m\}} L (\hat{\beta}_0 (\rho_0), \hat{\bm{\beta}}_{\rho_0}(t),
 \hat{\pi} (\rho_0), \hat{\phi} (\rho_0), \rho_0; D,c). 
\end{align*}

\subsubsection{Step 1: Estimation of Covariates Coefficients}
\label{subsubsec:step1}

%\myred{[MZ working here now ... even though PS has not yet endorsed MZ's revision in Sec 3]}

It is clear from our earlier theoretical result in Section~\ref{sec:theory} that, when $\rho=\rho_0$ is given and fixed, the quantity \eqref{eq:profileloglik} can be used directly as a criterion to estimate $\bm{\beta}(t)$. To begin, here one can fix the parameter $\phi$ at $\phi_0=1$, since it only appears as a scaling constant in \eqref{eq:profileloglik} and does not affect the optimum. The main computational saving afforded by Theorem~1 is that we can use an \emph{arbitrary} set of labels $z$ to carry out this step, estimating $\bm{\beta}(t)$ separately without simultaneously concerning ourselves with $\beta_0$ or having to make inference on $c$. Both of those tasks can be temporarily delayed until after $\bm{\beta}(t)$ is estimated.

For our static model (Section~\ref{subsec-TWsbm}), we use the \texttt{optim} function in R to maximize \eqref{eq:profileloglik} directly over $\bm{\beta}$, with $T=1$. For our time-varying model (Section~\ref{subsec-TVTWsbm}), we add (component-wise) smoothness penalties to \eqref{eq:profileloglik} 
%\JJ{to enhance stability and robustness [MZ does not like adding this extra ``reason'' to justify adding the penalty ... he thinks it's not even the right ``reason'']} 
and estimate $\bm{\beta}(t)$ as
\begin{equation}
\hat{\bm{\beta}}_{\rho_0}(t) = \argmax_{\bm{\beta}(t)} \ \ell_n(\bm{\beta}(t),1,\rho_0; D,z)  
-\frac{1}{2} \scsum{u=1}{p} \lambda_u \cdot \int \{ \bm{\beta}_u'' (t) \}^2 dt. \label{eq:hat_betat}
\end{equation}
The penalty parameters $\lambda_1, \dots, \lambda_p$ are chosen by cross-validation (see Section~\ref{subsec:tuning} below). With given penalty parameters, technical details for calculating~\eqref{eq:hat_betat} are provided in the Appendix. 

%\myred{[MZ thinks we may even be able to end this section here ... everything below is the original writing which MZ did not touch ... in fact, if JJ feels strongly about giving details of B-splines (i.e., the $\bm{B}$ and $\bm{\eta}$ part), those details should perhaps be in the supplementary materials rather than the main text?]}

\subsubsection{Step 2: Variational Inference}
\label{subsubsec:step2}
%%\myred{[MZ has worked from here to right after equation~\eqref{ELBO-TWSBM}]}

%%BEGIN MZ edits
In Step 2, with the estimate $\hat{\bm{\beta}}_{\rho_0}(t)$ from Step 1 (and, again, a pre-fixed $\rho=\rho_0$), we estimate the remaining parameters $\beta_0$, $\pi$, and $\phi$, as well as make inferences about the latent label $c$. 
%%MZ moved this disclaimer down below ... 
%%the first time L(...) appears in this section, MZ still wants the full argument list (rather than the shortened list) to appear inside, since the notation L(...) with a shortened list has never appeared before this section 
%%\JJ{For simplicity, in the subsequent discussion, we will omit the notation of $\bm{\beta}(t)$ and $\rho$ as arguments in the expressions, as $\rho=\rho_0$ and $\bm{\beta}(t)=\hat{\bm{\beta}}_{\rho_0}(t)$ are both fixed and not being estimated within this step.}
%%END MZ move

%%Ideally, 
%for a given $\rho_0$ and the estimated $\hat{\beta(t)}$, to infer parameters $\beta_0$, $\pi$ and $\phi$, 
%%we need to optimize the \textcolor{blue}{marginal} 
If we directly optimized the
likelihood function \eqref{TVlikelihood} using the EM algorithm, the E-step would require us to compute $\E_{c|D}(\cdot)$ but, here, the conditional distribution of the latent variable $c$ given $D$ is complicated because $c_i$ and $c_j$ are not conditionally independent in general. We will use a variational approach instead. 

To proceed, it will be more natural for us to emphasize the fact that \eqref{TVlikelihood} is really just the joint distribution of $(D,c)$. Thus, instead of writing it as
$L(\beta_0,\bm{\beta}(t),\pi,\phi,\rho;D,c)$, in this section we will write it simply as 
$\P(D,c;\beta_0,\pi,\phi)$,
%$\P(D,c;\beta_0,\pi,\phi, \hat{\bm{\beta}}_{\rho_0}(t), \rho_0)$
where we have also dropped $\bm{\beta}(t)$ and $\rho$ to keep the notation short because, within this step, $\rho=\rho_0$ and $\bm{\beta}(t)=\hat{\bm{\beta}}_{\rho_0}(t)$ are both fixed and not being estimated. 

Ideally, since the latent variable $c$ is not observable, one may want to work with the marginal distribution of $D$ and estimate $(\beta_0,\pi,\phi)$ as: 
\begin{align}
%\left(\hat{\beta}_0,\hat{\pi},\hat{\phi} \right)&=\argmax\limits_{\beta_0,\pi,\phi} \log\P(D;\beta_0,\pi,\phi,\hat{\bm{\beta}}_{\rho_0}(t), \rho_0) \nonumber \\
%&=\argmax\limits_{\beta_0,\pi,\phi} \ \log\scsum{c\in [K]^n}{} \P (D,c;\beta_0,\pi,\phi,\hat{\bm{\beta}}_{\rho_0}(t), \rho_0),
\left(\hat{\beta}_0,\hat{\pi},\hat{\phi} \right)&=\argmax\limits_{\beta_0,\pi,\phi} 
\log\P(D;\beta_0,\pi,\phi) \nonumber \\
&=\argmax\limits_{\beta_0,\pi,\phi} \ \log\scsum{c\in [K]^n}{} 
\P (D,c;\beta_0,\pi,\phi),
\label{eq:TWsbmOptimization}
\end{align}
but this is difficult due to the summation over $K^n$ terms.
%even for successive maximization along each variable at each iteration.
% as the optimization over the latent variables $z$'s is intractable due to the NP-hardness and the local optima. 
%%MZ has moved EM discussion ahead, so taking out here below
%%Moreover, the complicated conditional distribution of the latent variable $z$ on $D,\beta_0,\hat{\bm{\beta}}(t),\pi,\phi,\rho$ makes EM algorithm and MCMC challenging. 
%%We consider the variational inference as a practical method 
%to approximate the likelihood in~\eqref{eq:TWsbmOptimization} 
%%to circumvent the NP-hardness.
%, and use multiple initial community labels parallelly to mitigate the issue of local optima. 
The key idea of variational inference is to approximate 
%$\P(c|D;\beta_0,\pi,\phi,\hat{\bm{\beta}}_{\rho_0}(t),\rho_0)$ 
$\P(c|D;\beta_0,\pi,\phi)$
with a distribution $q(c)$ from a more tractable family---also referred to as the ``variational distribution'' in this context---and to decompose the \myred{objective} function in~\eqref{eq:TWsbmOptimization}  into two terms:
\begin{align}
% & \log \P(D;\beta_0,\pi,\phi,\hat{\bm{\beta}}_{\rho_0}(t), \rho_0) \nonumber \\
% =&\scsum{c\in [K]^n}{} \left( \log \P(D;\beta_0,\pi,\phi,\hat{\bm{\beta}}_{\rho_0}(t), \rho_0) \right) \cdot q(c) \nonumber \\
% =&\scsum{c\in [K]^n}{} \left( \log \dfrac{ \P(D;\beta_0,\pi,\phi,\hat{\bm{\beta}}_{\rho_0}(t), \rho_0) \cdot q(c)}{\P(D,c; \beta_0,\pi,\phi,\hat{\bm{\beta}}_{\rho_0}(t), \rho_0)} + \log \dfrac{\P(D,c; \beta_0,\pi,\phi,\hat{\bm{\beta}}_{\rho_0}(t), \rho_0)}{q(c)} \right) \cdot q(c) \nonumber \\
% =&\underbrace{\E_q \left[ \log \dfrac{q(c)}{\P(c|D; \beta_0,\pi,\phi,\hat{\bm{\beta}}_{\rho_0}(t), \rho_0)} \right]}_{\text{KL}}  + \underbrace{\E_q \left[ \log \dfrac{\P(D,c; \beta_0,\pi,\phi,\hat{\bm{\beta}}_{\rho_0}(t), \rho_0)}{q(c)} \right]}_{\text{ELBO}}. \label{ELBO}
& \log \P(D;\beta_0,\pi,\phi) \nonumber \\
=&\scsum{c\in [K]^n}{} \left( \log \P(D;\beta_0,\pi,\phi) \right) \cdot q(c) \nonumber \\
=&\scsum{c\in [K]^n}{} \left( \log \dfrac{ \P(D;\beta_0,\pi,\phi) \cdot q(c)}{\P(D,c; \beta_0,\pi,\phi)} + \log \dfrac{\P(D,c; \beta_0,\pi,\phi)}{q(c)} \right) \cdot q(c) \nonumber \\
=&\underbrace{\E_q \left[ \log \dfrac{q(c)}{\P(c|D; \beta_0,\pi,\phi)} \right]}_{\text{KL}}  + \underbrace{\E_q \left[ \log \dfrac{\P(D,c; \beta_0,\pi,\phi)}{q(c)} \right]}_{\text{ELBO}}. \label{ELBO}
\end{align}
The first term in \eqref{ELBO} can be recognized as the Kullback–Leibler (KL) divergence between $q(c)$ and $\P(c|D;\cdot)$, which is non-negative. This makes the second term in~\eqref{ELBO} a lower bound of objective function. It is referred to in the literature as the ``evidence lower bound'' (ELBO), and is equal to the objective function itself when the first term is zero, i.e., when $q(c)=\P(c|D;\cdot)$.  

So, instead of maximizing \eqref{eq:TWsbmOptimization} directly, one maximizes the ELBO term---not only over $(\beta_0, \pi, \phi)$, but also over $q$. Since the original objective function---that is, the left-hand side of \eqref{ELBO}---does not depend on $q$, maximizing the ELBO term over $q$ is also equivalent to minimizing the KL term. And when the KL term is small, not only is the variational distribution $q(c)$ close to $\P(c|D;\cdot)$, but the ELBO term is also automatically close to the original objective, which justifies why this approach often gives a good approximate solution to the otherwise intractable problem \eqref{eq:TWsbmOptimization} and why the variational distribution $q(c)\approx \P(c|D;\cdot)$ can be used to make approximate inferences about $c$.  

%% BEGIN MZ comment out
%%We use the variational expectation-maximisation (VEM) algorithm to estimate $\beta_0$, $\pi$, $\phi$ and the parameters of $q(c)$ by iterating between solving an expectation with variational distributions (VE step) and maximizing the expectation (M step). In the VE step, we update $q(c)$ and plug it in the ELBO in~\eqref{ELBO}. As the left-hand side of~\eqref{ELBO} is a constant of $q(c)$, when $q(c)$ is updated to approximate $\P(c|D;\beta_0,\pi,\phi,\hat{\bm{\beta}}(t))$ better leading to a smaller Kullback–Leibler divergence, the ELBO will be closer to the left-hand side of~\eqref{ELBO}. In the M step, we maximize the ELBO to estimate the remaining parameters.
%end of the box
%%END MZ comment out

Since the decomposition \eqref{ELBO} holds for any $q$, in practice one usually chooses it from a ``convenient'' family of distributions so that $\E_q(\cdot)$ is easy to compute. In particular, we can choose $$q(c) = \scprod{i=1}{n} q_i (c_i)$$ to be a completely factorizable distribution; here, each $q_i$ is simply a standalone multinomial distribution with probability vector $(\tau_{i1},\cdots,\tau_{iK})$. 
%Variational inference approximates the conditional distribution of the latent variables $ L_{\text{TV-TSBM}} (z|D;\beta_0,\pi,\beta(t),\phi,\rho)$ with the mean field approximation, and aims at optimizing a lower bound of \eqref{eq:TWsbmOptimization}. 
Under this choice, $\E_q[\indCi]=\tau_{ik}$, $\E_q[\indCiCj]=\tau_{ik}\tau_{jl}$, and the ELBO term in~\eqref{ELBO} is simply
%\begin{align}
\begin{multline}
%   \text{ELBO}(\tau,\beta_0,\pi,\phi; \hat{\bm{\beta}}_{\rho_0}(t),\rho_0, D) =
\text{ELBO}(\tau,\beta_0,\pi,\phi; D) =
 \scsum{i=1}{n} \scsum{k=1}{K}  \tau_{ik}\cdot \log(\pi_k) +\scsum{\nu=1}{T} \scsum{1\leq i < j\leq n}{} \log a(y_{ij}(t_{\nu}),\phi,\rho_0)  \\ %\nonumber \\
 +\scsum{\nu=1}{T} \scsum{1\leq i < j\leq n}{} \scsum{k,l=1}{K} \dfrac{\tau_{ik}\tau_{jl}}{\phi} \biggl( y_{ij}(t_\nu) \dfrac{\exp{[(1-\rho_0)\{\beta_0^{kl}+ x_{ij}^\top \hat{\bm{\beta}}_{\rho_0}(t_\nu)\} ]}}{1-\rho_0}  \\ %\nonumber \\
 - \dfrac{\exp{[(2-\rho_0)\{\beta_0^{kl}+ x_{ij}^\top \hat{\bm{\beta}}_{\rho_0}(t_\nu)\} ]}}{2-\rho_0}\biggr) 
 -\scsum{i=1}{n} \scsum{k=1}{K}  \tau_{ik}\cdot \log(\tau_{ik}), \label{ELBO-TWSBM}
\end{multline}
%\end{align}
which is easy to maximize in a coordinate-wise fashion, i.e., successively over $ \tau $, $\beta_0$, $\pi$ and $\phi$.
%%END MZ edits 
%%\myred{[MZ stopped editing here, skipping over the rest of this subsection for now]}

The maxima of \eqref{ELBO-TWSBM} with respect to $ \tau$ and $\pi$ is found by the method of Lagrange multipliers respectively, as the according optimization problem is subject to equality constraints $\sum_{k=1}^{K} \tau_{ik} =1$ and $\sum_{k=1}^{K} \pi_k =1$ for any $i$ respectively. Specifically, at iteration step $h$
%Maximization of \eqref{ELBO-TWSBM} with respect to $\beta_0$ and $\beta$ is implemented by Newton Raphson method. 
\begin{align*}
 \tau^{(h)}_{ik}=\dfrac{f_{ik}}{\scsum{k=1}{K}f_{ik}},\ k=1,\cdots,K\text{ and }i=1,\cdots,n, 
\end{align*}
where 
\begin{align*}
    f_{ik}=&\pi_k^{(h-1)} \cdot \exp\biggl[\scsum{\nu=1}{T}\scsum{j\neq i}{} \scsum{l=1}{K} 
\dfrac{\tau^{(h-1)}_{jl}}{\phi^{(h-1)}}  \biggl\{y_{ij}(t) \dfrac{\exp[(1-\rho_0)\{( \beta_0^{kl})^{(h-1)} +x_{ij}^\top \hat{\bm{\beta}}_{\rho_0}(t_\nu)\}]}{1-\rho_0}  \\
& -\dfrac{\exp[(2-\rho_0)\{( \beta_0^{kl})^{(h-1)} +x_{ij}^\top \hat{\bm{\beta}}_{\rho_0}(t_\nu)\}]}{2-\rho_0}  \biggr\} \biggr],
\end{align*}
and
\begin{align*}
\pi_k^{(h)}=\dfrac{\scsum{i=1}{n} \tau^{(h)}_{ik}}{n},\ k=1,\cdots,K.
\end{align*}

The objective function~\eqref{ELBO-TWSBM} is concave down in $\beta_0^{kl}$ for each community label pair $k-l$, which allows the zeros of the first derivative of~\eqref{ELBO-TWSBM} to be its maxima. We update $\beta_0^{(h-1)}$ to $\beta_0^{(h)}$ by solving the equation $\frac{\partial }{\partial \beta_0}\text{ELBO}(\tau^{(h)},\beta_0,\pi^{(h)},\phi;D)=0$ analytically for each pair $k$-$l$, which can be implemented in one step:
\begin{align*}
 ( \beta_0^{kl})^{(h)}=\log \dfrac{\scsum{\nu=1}{T} \scsum{1\leq i < j \leq n}{}y_{ij}(t) \exp [(1-\rho_0)x_{ij}^\top \hat{\bm{\beta}}_{\rho_0}(t_\nu) ]\cdot \tau^{(h)}_{ik}\tau^{(h)}_{jl} }{\scsum{\nu=1}{T} \scsum{1\leq i < j \leq n}{} \exp [(2-\rho_0) x_{ij}^\top \hat{\bm{\beta}}_{\rho_0}(t_\nu)] \cdot \tau^{(h)}_{ik}\tau^{(h)}_{jl} }.
\end{align*}

%%MZ revision
With $\pi_k^{(h)}$, $\tau^{(h)}_{ik}$ and $\beta_0^{(h)}$ fixed, we can now directly maximize the ELBO term \eqref{ELBO-TWSBM} over $\phi$ to update it in principle. However, the function $a(y_{ij}(t_{\nu}),\phi,\rho_0)$ is ``a bit of a headache'' to compute, so we use 
the R package \texttt{tweedie} by \cite{dunn2005series,dunn2008evaluation} that computes \eqref{Twdist} for us,
%%we calculate the Tweedie mean $$
%%\mu_{ij}^{(h)}(t_\nu)=\exp\left\{ \beta_0^{(h)}[c_i^{(h)},c_j^{(h)}]+x_{ij}^\top \hat{\bm{\beta}}_{\rho_0}(t_\nu) \right\},$$ and 
and update $\phi$ by letting $c_i^{(h)} = \argmax_{k} \tau^{(h)}_{ik}$ and maximizing over the original log-likelihood function instead, i.e.,
%%by
    \begin{align}
        \phi^{(h)}&=\argmax\limits_{\phi} 
%%        L(\beta_0^{(h)},\pi^{(h)},\phi; D,c^{(h)})
\ \log \P(D,c^{(h)}; \beta_0^{(h)},\pi^{(h)},\phi).
%%\nonumber \\
%%         & =\argmax\limits_{\phi} \scprod{\nu=1 }{T} \scprod{i<j }{}
%%         f(y_{ij}(t_\nu) \mid \mu_{ij}^{(h)}(t_\nu),\phi,\rho_0), 
         \label{updatePhi}
    \end{align}
%%    where $f$ in~\eqref{updatePhi} is the Tweedie density function~\eqref{Twdist}. 
We do this directly using the R function \texttt{optim}.
%%END MZ revision

\begin{comment}
\begin{align*}
  &  l(Z,\beta_0,\beta(t),\phi,\rho|D)  =  \scsum{i=1}{n} \scsum{k=1}{K}  z_{ik}\cdot \log(\pi_k) +\scsum{t=1}{T} \scsum{1\leq i < j\leq n}{} \scsum{k,l=1}{K} \log a(y_{ij}(t),\phi,\rho) +  \nonumber \\
    &\scsum{t=1}{T} \scsum{1\leq i < j\leq n}{} \scsum{k,l=1}{K} \dfrac{1}{\phi} \left( y_{ij} \dfrac{\exp{[(1-\rho)(\beta_0^{kl} + x_{ij} \beta(t))]}}{1-\rho} - \dfrac{\exp{[(2-\rho)(\beta_0^{kl} + x_{ij} \beta(t))]}}{2-\rho}\right)\cdot z_{ik}z_{jl} .
\end{align*}

\begin{align*}
  & \text{ELBO}(q,\beta_0,\beta,\phi,\rho;D)= \scsum{i=1}{n} \scsum{k=1}{K}  \tau_{ik}\cdot \log(\pi_k) -\scsum{i=1}{n} \scsum{k=1}{K}  \tau_{ik}\cdot \log(\tau_{ik}) \\
  &  +\scsum{t=1}{T} \scsum{1\leq i < j\leq n}{} \scsum{k,l=1}{K} \log a(y_{ij}(t),\phi,\rho)
  \\
    &+\scsum{t=1}{T}  \scsum{1\leq i < j\leq n}{} \scsum{k,l=1}{K} \dfrac{\tau_{ik}\tau_{jl}}{\phi} \left( y_{ij} \dfrac{\exp{[(1-\rho)(\beta_0^{kl} + x_{ij} \beta(t))]}}{1-\rho} - \dfrac{\exp{[(2-\rho)(\beta_0^{kl} + x_{ij} \beta(t))]}}{2-\rho}\right). 
\end{align*}
\end{comment}

\subsection{Tuning Parameter Selection}
\label{subsec:tuning}

We adapt the leave-one-out cross validation to choose the tuning parameter $\lambda$ when fitting our model. 
%We adapt the idea of time series cross-validation \citep{hyndman2018forecasting} and propose the tuning scheme in~\ref{alg:cv}. 
In particular, each time 
we utilize observations made at $T - 1$ time points to train the model and then test
%calculate the loss of 
the trained model on the observations made at the remaining time points. 
%To avoid boundary effects, we repeat the above procedure $T-2$ times by always retaining the observations made at the first and the last time points in the training set. 
%\JJ{To avoid boundary effects, we exclude the first and last time points as testing data. More specifically, we repeat the above procedure $T-2$ times by taking the observations made from the second to the last second time points as the test fold respectively. }
%%MZ sees PS and JJ appear to disagree on how to write the sentence above. MZ will hereby rewrite it as follows:
\myred{To avoid boundary effects, our leave-one-out procedure is repeated for only $T-2$ times (as opposed to the usual $T$ times), because we always retain the observations at times $t_1$ and $t_T$ in the training set---only those at times $t_2, \dots, t_{T-1}$ are used (one at a time) as test points.} 
%%END MZ rewrite
In our implementations, the loss is defined as the negative log-likelihood of the fitted model, and the overall loss 
is taken as the average across the $T-2$ repeats. We select the ``optimal" $\lambda$ that gives rise to the smallest loss.

\section{Simulation}
\label{sec:simulation}

%%BEGIN MZ rewrite
%%This section presents simulation studies for three objectives: 1) to illustrate the usage of Tweedie distribution in the classic problem of Stochastic Block Models (Vanilla Tweedie SBM) in Section~\ref{subsec:simulationVanillaTWsbm}; 2) how considering the nodal-pair covariates helps uncover the latent communities in Vanilla Tweedie SBM (Tweedie SBM with covariates) in Section~\ref{subsec:simulationTWsbm}, and 3) the time-varying effect of the nodal-pair covariates in dynamic weighted SBM problem (Time-varying Tweedie SBM) in Section~\ref{subsec:simulationTVTWsbm}. 
In this section, we present simulation results to validate the performance of \myred{our restricted Tweedie SBM}. We do so in successive steps---from the vanilla model (Section~\ref{subsec:simulationVanillaTWsbm}), to the static model with covariates (Section~\ref{subsec:simulationTWsbm}), and finally, the most general, time-varying version of the model (Section~\ref{subsec:simulationTVTWsbm}). 
%%END MZ rewrite

%%Minor edits by MZ before Sec 5.1, details not noted.
We mainly focus on two aspects of the results, the clustering quality and the accuracy of the estimated covariate effects. 
%%Above, MZ changed ``accuracy of the estimated parameters'' to ``accuracy of the estimated covariate effects'' because we never really measured the accuracy of $\hat{\phi},\hat{\rho}$, right?
We measure the latter by the mean squared error, and the former by a metric called ``normalized mutual information'' (NMI) \citep{danon2005comparing}, which ranges in $[0,1]$, with values closer to $1$ indicating better agreements between the estimated community labels and the true ones.

\begin{comment}
defined as the mutual information between the estimated clustering results $C_{est}$ and the ground true clustering $C_{true}$ normalized by the sum of their entropy,
\begin{align*}
    NMI(C_{est},C_{true})&=
    %\dfrac{2I(C_{est},C_{true})}{H(C_{est})+H(C_{true})}=
    \dfrac{2\sum\limits_{i=1}^{K}\sum\limits_{j=1}^{K} \frac{m_{ij}}{N} \log \frac{m_{ij} N}{ m_{i\cdot} m_{\cdot j} } }{-\sum\limits_{i=1}^{K}\frac{m_{i\cdot}}{N}\log \frac{m_{i\cdot}}{N} -\sum\limits_{j=1}^{K}\frac{m_{\cdot j}}{N}\log \frac{m_{\cdot j}}{N}},
\end{align*}
where 
$M$ is a $K$-by-$K$ contingency table denoting the number of nodes the overlap between the two clustering. Namely, $M_{ij}$ denoting the number of nodes clustered as group $i$ in $C_{est}$ while clustered as group $j$ in $C_{true}$. 
\end{comment}

For all simulations, we fix the true number of communities to be $K=3$, with prior probabilities $\pi=(0.2,0.3,0.5)$. For the true matrix $\beta_0$, we set all diagonal entries $\beta_0^{kk}$ to be equal, and all off-diagonal entries $\beta_0^{kl}$ to be equal as well---so the entire matrix is completely specified by just two numbers. 

To avoid getting stuck at poor local optima, we use multiple initial values in each run. 

%\subsection{Vanilla Tweedie SBM}
\subsection{Simulation of vanilla model}
\label{subsec:simulationVanillaTWsbm}

First, we assess the performance of \myred{our vanilla model (Section~\ref{subsec-vanillaTWsbm})}, and compare it with the Poisson SBM and spectral clustering. 
The Poisson SBM assumes the edges follow Poisson distributions; we simply round each $y_{ij}$ into an integer and use the function \texttt{estimateSimpleSBM} in R package \texttt{sbm} to fit it. 
%%BEGIN MZ comment out
%%Spectral clustering, another popular method for weighted community detection, can leverage the eigenvectors of the weighted graph Laplacian matrix to identify clusters of nodes. 
%%END MZ comment out
To run spectral clustering, we use the function \texttt{reg.SSP} from the R package \texttt{randnet}.
The function \texttt{estimateSimpleSBM} uses results from a bipartite SBM as its initial values. To make a more informative comparison, we use two different initialization strategies to fit our model: (i) starting from 30 sets of randomly drawn community labels and picking the best solution afterwards, and (ii) starting from the Poisson SBM result itself.

We generate $Y$ using nine different combinations of  $(\phi,\rho)$ with $\phi=0.5, 1, 2$ and $\rho=1.2, 1.5, 1.8$, and three different $\beta_0$ matrices:
% \begin{equation*}
% \begin{array}{llcl}
% \text{scenario 1,} & (\beta_0^{kk},\beta_0^{kl})=(1.0,\phantom{-}0.0) & \Rightarrow & \exp(\beta_0^{kk})-\exp(\beta_0^{kl}) \approx 1.72;\\
% \text{scenario 2,} & (\beta_0^{kk},\beta_0^{kl})=(0.5,-0.5) & \Rightarrow & \exp(\beta_0^{kk})-\exp(\beta_0^{kl})\approx 1.04;\\
% \text{scenario 3,} & (\beta_0^{kk},\beta_0^{kl})=(0.0,-1.0) & \Rightarrow & \exp(\beta_0^{kk})-\exp(\beta_0^{kl})\approx 0.63.
% \end{array}
% \end{equation*}
\begin{align*}
\text{scenario 1, }  (\beta_0^{kk},\beta_0^{kl})=(1.0,\phantom{-}0.0)  \Rightarrow & \exp(\beta_0^{kk})-\exp(\beta_0^{kl}) \approx 1.72;\\
\text{scenario 2, }  (\beta_0^{kk},\beta_0^{kl})=(0.5,-0.5)  \Rightarrow & \exp(\beta_0^{kk})-\exp(\beta_0^{kl})\approx 1.04;\\
\text{scenario 3, }  (\beta_0^{kk},\beta_0^{kl})=(0.0,-1.0)  \Rightarrow & \exp(\beta_0^{kk})-\exp(\beta_0^{kl})\approx 0.63.
\end{align*}

According to the discrepancy in $\mu_{ij}$ between $(i,j)$-pairs belonging to the same group and those belonging to different groups, the clustering difficulty of the three designs can be roughly ordered as: scenario 1 $<$ scenario 2 $<$ scenario 3.

Table~\ref{simulation1_tab1},~\ref{simulation1_tab2}, and~\ref{simulation1_tab3} summarize the averages and the standard errors of the NMI metric for different methods over 50 simulation runs, respectively for scenarios 1, 2 and 3. As expected, all methods perform the best in scenario 1 and the worst in scenario 3. Their performances also improve when the sample size $n$ increases, and as the parameter $\phi$ decreases---as the dispersion parameter, a smaller $\phi$ means a reduced variance and an easier problem.

Overall, our \myred{restricted Tweedie SBM} and the Poisson SBM tend to outperform spectral clustering. Among all 54 sets of simulation results, 
%%MZ says "OMG, such a cumbersome way of doing it! Can you imagine how he readers will feel??!!"
%%i.e., (scenarios 1, 2, 3) $\times$ $(\phi=2, 1, 0.5) \times (\rho=1.2, 1.5, 1.8) \times (n=50, 100)$, 
%%END MZ revision
our model with random initialization compares favorably with other methods in 50 of them. 
In the remaining four sets \myred{(marked by a superscript ``$\dagger$'' in the tables)}, 
%%MZ says "OMG, such a cumbersome way of doing it! Can you imagine how he readers will feel??!!"
%%i.e. (scenario 2, $\phi=0.5$, $\rho=1.2$, $n=50$), (scenario 2, $\phi=0.5$, $\rho=1.5$, $n=50$), (scenario 2, $\phi=0.5$, $\rho=1.8$, $n=50$), and (scenario 3, $\phi=0.5$, $\rho=1.2$, $n=50$), 
%%\myred{[MZ: one cannot easily see these from the tables]} 
%%END MZ revision
the Poisson SBM is slightly better, but we could still outperform it in three of them and match it in the other if we initialized our algorithm with the Poisson SBM result itself. It is evident in all cases that \myred{our restricted Tweedie SBM} can further improve the clustering result of the Poisson SBM.

\begin{table}[ht] %htp
\centering
\begin{tabular}{|c|c|c|cccc|}
\hline
 & & & \multicolumn{2}{c}{Restricted Tweedie SBM} & Poisson & Spectral \\
 $\phi$ & $\rho$ & $n$ & Random Init.  &  Poisson Init.       &   SBM     & Clustering    \\ \hline 
 \hline
    \multirow{6}{*}{$2$} &\multirow{2}{*}{$1.2$} & 50 & \shortstack{\\ 0.9097 \\(0.016)}  & \shortstack{\\ 0.8275 \\(0.023)}  & \shortstack{\\ 0.8099 \\(0.022)}  & \shortstack{\\ 0.5547 \\(0.012)}  \\ \cline{3-7}
&& 100 & \shortstack{\\ 0.9958 \\(0.002)}  &  \shortstack{\\ 0.9958 \\(0.002)} & \shortstack{ 0.9950\\(0.002)}  &  \shortstack{\\0.9185 \\(0.019)} \\ \cline{2-7}
&\multirow{2}{*}{$1.5$} & 50 & \shortstack{\\ 0.8647 \\(0.019)}  &  \shortstack{\\ 0.7780 \\(0.02)} &  \shortstack{\\ 0.7275 \\(0.02)} &  \shortstack{\\ 0.5152 \\(0.012)} \\ \cline{3-7}
&& 100 & \shortstack{\\0.9878 \\(0.003)}& \shortstack{\\0.9878 \\(0.003)} & \shortstack{\\0.9865 \\(0.003)} & \shortstack{\\0.769 \\(0.025)}\\ \cline{2-7}
&\multirow{2}{*}{$1.8$} & 50 & \shortstack{\\ 0.7644 \\(0.017)} & \shortstack{\\ 0.7180 \\(0.020)}  & \shortstack{\\ 0.6539 \\(0.02)}  & \shortstack{\\ 0.4857 \\(0.015)} \\  \cline{3-7}
&& 100 &\shortstack{\\0.9828 \\(0.004)}& \shortstack{\\0.9828 \\(0.004)}  & \shortstack{\\0.9826 \\(0.004)} & \shortstack{\\0.6597 \\(0.015)} \\
\hline
 \hline
    \multirow{6}{*}{$1$} &\multirow{2}{*}{$1.2$} & 50 & \shortstack{\\ 0.9918 \\(0.005)}  & \shortstack{\\ 0.9946 \\(0.004)}  & \shortstack{\\ 0.9880 \\(0.004)}  & \shortstack{\\ 0.7529 \\(0.027)}  \\ \cline{3-7}
&& 100 &  \shortstack{\\ 1 \\( 0 )} & \shortstack{\\ 1 \\( 0 )}  & \shortstack{\\ 1 \\( 0 )}  & \shortstack{\\ 1 \\( 0 )} \\ \cline{2-7}
&\multirow{2}{*}{$1.5$} & 50 & \shortstack{\\ 0.9778 \\(0.008)} & \shortstack{\\ 0.9859 \\(0.006)}  & \shortstack{\\ 0.9745 \\(0.008)}  & \shortstack{\\ 0.7034 \\(0.023)} \\ \cline{3-7}
&& 100 &  \shortstack{\\ 1 \\( 0 )} &  \shortstack{\\ 1 \\( 0 )} & \shortstack{\\ 1 \\( 0 )}  & \shortstack{\\0.9991 \\(0.001)} \\ \cline{2-7}
&\multirow{2}{*}{$1.8$} & 50 & \shortstack{\\ 0.9653 \\(0.01)}  & \shortstack{\\ 0.9644 \\(0.01)}  & \shortstack{\\ 0.9512 \\(0.012)}  & \shortstack{\\ 0.6702 \\(0.019)} \\  \cline{3-7}
&& 100 & \shortstack{\\0.9992 \\(0.001)}  &  \shortstack{\\0.9992 \\(0.001)} & \shortstack{\\0.9992 \\(0.001)}  &  \shortstack{\\0.9656 \\(0.013)}\\
\hline
 \hline
\multirow{6}{*}{$0.5$} &\multirow{2}{*}{$1.2$} & 50 & \shortstack{\\ 1 \\( 0 )}  &  \shortstack{\\ 1 \\( 0 )} & \shortstack{\\ 1 \\( 0 )}  & \shortstack{\\ 0.9934 \\(0.007)} \\ \cline{3-7}
&& 100 & \shortstack{\\ 1 \\( 0 )}  &  \shortstack{\\ 1 \\( 0 )} & \shortstack{\\ 1 \\( 0 )}  & \shortstack{\\ 1 \\( 0 )} \\ \cline{2-7}
&\multirow{2}{*}{$1.5$} & 50 &  \shortstack{\\ 1 \\( 0 )} &  \shortstack{\\ 1 \\( 0 )} & \shortstack{\\ 1 \\( 0 )}  & \shortstack{\\ 0.9297 \\(0.019)} \\ \cline{3-7}
&& 100 & \shortstack{\\ 1 \\( 0 )} & \shortstack{\\ 1 \\( 0 )}  & \shortstack{\\ 1 \\( 0 )} & \shortstack{\\ 1 \\( 0 )} \\ \cline{2-7}
&\multirow{2}{*}{$1.8$} & 50 & \shortstack{\\ 1 \\( 0 )}  &  \shortstack{\\ 1 \\( 0 )} &  \shortstack{\\ 0.9985 \\(0.001)} & \shortstack{\\ 0.8307 \\(0.025)} \\  \cline{3-7}
&& 100 &  \shortstack{\\ 1 \\( 0 )} &  \shortstack{\\ 1 \\( 0 )} &  \shortstack{\\ 1 \\( 0 )} & \shortstack{\\ 1 \\( 0 )} \\
\hline
\end{tabular}
\caption{Summary of the NMI in scenario 1, $(\beta_0^{kk},\beta_0^{kl})=(1,0)$, over 50  simulation runs.}
\label{simulation1_tab1}
\end{table}

\begin{table}[ht] %htp
\centering
\begin{tabular}{|c|c|c|cccc|}
\hline
 & & & \multicolumn{2}{c}{Restricted Tweedie SBM} & Poisson & Spectral \\
 $\phi$ & $\rho$ & $n$ & Random Init.  &  Poisson Init.       &   SBM     & Clustering    \\ \hline 
 \hline
    \multirow{6}{*}{$2$} &\multirow{2}{*}{$1.2$} & 50 &\shortstack{\\ 0.7490 \\(0.023)}&   \shortstack{\\ 0.6713 \\(0.024)} & \shortstack{\\ 0.640 \\(0.021)} &  \shortstack{\\ 0.4515 \\(0.014)} \\ \cline{3-7}
&& 100 & \shortstack{\\ 0.9698 \\(0.007)} & \shortstack{\\ 0.9592 \\(0.011)} & \shortstack{\\ 0.9603 \\(0.011)} & \shortstack{\\ 0.6936 \\(0.023)}\\ \cline{2-7}
&\multirow{2}{*}{$1.5$} & 50 &\shortstack{\\ 0.6921 \\(0.023)}&\shortstack{\\ 0.6327 \\(0.021)}& \shortstack{\\ 0.6031 \\(0.021)} & \shortstack{\\ 0.4596 \\(0.018)} \\ \cline{3-7}
&& 100 & \shortstack{\\ 0.9568 \\(0.009)} & \shortstack{\\ 0.9650 \\(0.007)} & \shortstack{\\ 0.9430 \\(0.011)} & \shortstack{\\ 0.6133 \\(0.014)} \\ \cline{2-7}
&\multirow{2}{*}{$1.8$} & 50 &\shortstack{\\ 0.7052 \\(0.022)} & \shortstack{\\ 0.6315 \\(0.023)} & \shortstack{\\ 0.5727 \\(0.020)} &\shortstack{\\ 0.4174 \\(0.02)}\\  \cline{3-7}
&& 100 & \shortstack{\\ 0.9803 \\(0.004)}  & \shortstack{\\ 0.9539 \\(0.013)} & \shortstack{\\ 0.9362 \\(0.013)} & \shortstack{\\ 0.6433 \\(0.012)}\\
\hline
\hline
    \multirow{6}{*}{$1$} &\multirow{2}{*}{$1.2$} & 50 & \shortstack{\\ 0.9490 \\(0.013)} & \shortstack{\\ 0.9284 \\(0.014)}  & \shortstack{\\ 0.9037 \\(0.013)} & \shortstack{\\ 0.6489 \\(0.021)}  \\ \cline{3-7}
&& 100 & \shortstack{\\ 0.9992 \\(0.001)}&\shortstack{\\ 0.9992 \\(0.001)}& \shortstack{\\ 0.9984 \\(0.001)} & \shortstack{\\ 0.9918 \\(0.003)}  \\ \cline{2-7}
&\multirow{2}{*}{$1.5$} & 50 & \shortstack{\\ 0.9330 \\(0.014)} & \shortstack{\\ 0.9193 \\(0.014)} & \shortstack{\\ 0.9127 \\(0.014)} & \shortstack{\\ 0.6304 \\(0.018)} \\ \cline{3-7}
&& 100 & \shortstack{\\ 1 \\( 0 )} &\shortstack{\\ 1 \\( 0 )}& \shortstack{\\ 0.9976 \\(0.001)} &\shortstack{\\ 0.9926 \\(0.003)}\\ \cline{2-7}
&\multirow{2}{*}{$1.8$} & 50 & \shortstack{\\ 0.9288 \\(0.013)} & \shortstack{\\ 0.9235 \\(0.014)} & \shortstack{\\ 0.9103 \\(0.014)} & \shortstack{\\ 0.6437 \\(0.015)} \\  \cline{3-7}
&& 100 &\shortstack{\\ 0.9992 \\(0.001)} & \shortstack{\\ 0.9992 \\(0.001)} &\shortstack{\\ 0.9967 \\(0.002)}&\shortstack{\\ 0.9375 \\(0.017)}\\
\hline
 \hline
\multirow{6}{*}{$0.5$} &\multirow{2}{*}{$1.2$} & 50$^{\dagger}$ & \shortstack{\\ 0.9961 \\(0.004)}  & \shortstack{\\ 1 \\( 0 )} &  \shortstack{\\ 1 \\( 0 )} &  \shortstack{\\ 0.8504 \\(0.027)}  \\ \cline{3-7}
&& 100 & \shortstack{\\ 1 \\( 0 )}  &  \shortstack{\\ 1 \\( 0 )} & \shortstack{\\ 1 \\( 0 )}  & \shortstack{\\ 0.9991 \\(0.001)} \\ \cline{2-7}
&\multirow{2}{*}{$1.5$} & 50$^{\dagger}$ & \shortstack{\\ 0.9847 \\(0.009)}  & \shortstack{\\ 1 \\( 0 )}  & \shortstack{\\ 1 \\( 0 )}  & \shortstack{\\ 0.8193 \\(0.0260)} \\ \cline{3-7}
&& 100 &  \shortstack{\\ 1 \\( 0 )} &  \shortstack{\\ 1 \\( 0 )}& \shortstack{\\ 1 \\( 0 )}& \shortstack{\\ 1 \\( 0 )}\\ \cline{2-7}
&\multirow{2}{*}{$1.8$} & 50$^{\dagger}$ & \shortstack{\\ 0.9879 \\(0.007)}  & \shortstack{\\ 1 \\( 0 )}  &  \shortstack{\\ 0.9973 \\(0.002)} & \shortstack{\\ 0.7947 \\(0.026)}   \\  \cline{3-7}
&& 100 & \shortstack{\\ 1 \\( 0 )}& \shortstack{\\ 1 \\( 0 )}& \shortstack{\\ 1 \\( 0 )}&  \shortstack{\\ 1 \\( 0 )}\\
\hline
\end{tabular}
\caption{Summary of NMI in scenario 2, $(\beta_0^{kk},\beta_0^{kl})=(0.5,-0.5)$, over 50 runs. A superscript ``$\dagger$'' denotes a case in which (restricted Tweedie SBM with random initialization) $<$ (Poisson SBM) $\leq$ (restricted Tweedie SBM with Poisson initialization) in their respective clustering performances.}
\label{simulation1_tab2}
\end{table}

\begin{table}[ht] %htp
\centering
\begin{tabular}{|c|c|c|cccc|}
\hline
 & & & \multicolumn{2}{c}{Restricted Tweedie SBM} & Poisson & Spectral \\
 $\phi$ & $\rho$ & $n$ & Random Init.  &  Poisson Init.       &   SBM     & Clustering    \\ \hline 
 \hline
    \multirow{6}{*}{$2$} &\multirow{2}{*}{$1.2$} & 50 & \shortstack{\\ 0.4385 \\(0.032)} & \shortstack{\\ 0.4340 \\(0.027)} & \shortstack{\\ 0.4243 \\(0.025)} & \shortstack{\\ 0.2889 \\(0.022)} \\ \cline{3-7}
&& 100 & \shortstack{\\ 0.8497 \\(0.013)} & \shortstack{\\ 0.8025 \\(0.019)} & \shortstack{\\ 0.774 \\(0.020)} & \shortstack{\\ 0.5134 \\( 0.016 )}\\ \cline{2-7}
&\multirow{2}{*}{$1.5$} & 50 & \shortstack{\\ 0.5611 \\(0.023)} &\shortstack{\\ 0.5226 \\(0.023)} & \shortstack{\\ 0.5071 \\(0.022)} &\shortstack{\\ 0.3462 \\(0.018)} \\ \cline{3-7}
&& 100 & \shortstack{\\ 0.9097 \\(0.012)} & \shortstack{\\ 0.8606 \\(0.016)} & \shortstack{\\ 0.8146 \\(0.017)} & \shortstack{\\ 0.5737 \\(0.012)} \\ \cline{2-7}
&\multirow{2}{*}{$1.8$} & 50 &\shortstack{\\ 0.6179 \\(0.022)}&\shortstack{\\ 0.5771 \\(0.024)}&\shortstack{\\ 0.522 \\(0.021)}&\shortstack{\\ 0.4102 \\(0.018)}\\  \cline{3-7}
&& 100 &\shortstack{\\ 0.9567 \\(0.009)}&\shortstack{\\ 0.8736 \\(0.02)}&\shortstack{\\ 0.8377 \\(0.020)}& \shortstack{\\ 0.5985 \\(0.013)} \\
\hline
\hline
    \multirow{6}{*}{$1$} &\multirow{2}{*}{$1.2$} & 50 &\shortstack{\\ 0.8710 \\(0.016)}
& \shortstack{\\ 0.7404 \\(0.017)} & \shortstack{\\ 0.7325 \\(0.016)} &\shortstack{\\ 0.5379 \\(0.011)} \\ \cline{3-7}
&& 100 &\shortstack{\\ 0.9893 \\(0.006)}&\shortstack{\\ 0.9967 \\(0.002)} &\shortstack{\\ 0.9842 \\(0.003)}&\shortstack{\\ 0.862 \\(0.022)}\\ \cline{2-7}
&\multirow{2}{*}{$1.5$} & 50 & \shortstack{\\ 0.8709 \\(0.016)} & \shortstack{\\ 0.7763 \\(0.017)} & \shortstack{\\ 0.7684 \\(0.016)} & \shortstack{\\ 0.5601 \\(0.012)} \\ \cline{3-7}
&& 100 & \shortstack{\\ 0.9950 \\(0.004)}& \shortstack{\\ 0.9992 \\(0.001)} & \shortstack{\\ 0.9876 \\(0.003)}& \shortstack{\\ 0.8311 \\(0.022)}  \\ \cline{2-7}
&\multirow{2}{*}{$1.8$} & 50 & \shortstack{\\ 0.8806 \\(0.017)} &\shortstack{\\ 0.8039 \\(0.019)}&\shortstack{\\ 0.7901 \\(0.018)}& \shortstack{\\ 0.6092 \\(0.013)} \\  \cline{3-7}
&& 100 &\shortstack{\\ 0.9992 \\(0.001)}&\shortstack{\\ 0.9992 \\(0.001)}&\shortstack{\\ 0.9934 \\(0.002)}&\shortstack{\\ 0.8876 \\(0.022)}\\
\hline
 \hline
\multirow{6}{*}{$0.5$} &\multirow{2}{*}{$1.2$} & 50 & \shortstack{\\ 0.9414 \\(0.014)} & \shortstack{\\ 0.8998 \\(0.017)} & \shortstack{\\ 0.8817 \\(0.016)} & \shortstack{\\ 0.7379 \\(0.028)}   \\ \cline{3-7}
&& 100$^{\dagger}$ &\shortstack{\\ 0.9956 \\(0.004)} & \shortstack{\\ 1 \\( 0 )} &\shortstack{\\ 1 \\( 0 )}&\shortstack{\\ 0.9983 \\(0.001)}\\ \cline{2-7}
&\multirow{2}{*}{$1.5$} & 50 & \shortstack{\\ 0.9591 \\(0.012)} &\shortstack{\\ 0.9112 \\(0.015)}& \shortstack{\\ 0.8999 \\(0.015)} &\shortstack{\\ 0.7354 \\(0.026)} \\ \cline{3-7}
&& 100 &\shortstack{\\ 1 \\( 0 )}&\shortstack{\\ 1 \\( 0 )}& \shortstack{\\ 1 \\( 0 )}  & \shortstack{\\ 1 \\( 0 )} \\ \cline{2-7}
&\multirow{2}{*}{$1.8$} & 50 & \shortstack{\\ 1 \\(0)} &  \shortstack{\\ 0.9727 \\(0.01)} &\shortstack{\\ 0.9550 \\(0.01)}&\shortstack{\\ 0.7549 \\(0.022)}\\  \cline{3-7}
&& 100 & \shortstack{\\ 1 \\( 0 )} &\shortstack{\\ 1 \\( 0 )}& \shortstack{\\ 1 \\( 0 )}& \shortstack{\\ 1 \\( 0 )}\\
\hline
\end{tabular}
\caption{Summary of NMI in scenario 3, $(\beta_0^{kk},\beta_0^{kl})=(0,-1)$, over 50 simulation runs. A superscript ``$\dagger$'' denotes a case in which (restricted Tweedie SBM with random initialization) $<$ (Poisson SBM) $\leq$ (restricted Tweedie SBM with Poisson initialization) in their respective clustering performances.}
\label{simulation1_tab3}
\end{table}

%\subsection{Simulation of Tweedie SBM with Covariates}
\subsection{Simulation of model with covariates}
\label{subsec:simulationTWsbm}

Next, we study \myred{our static model with covariates (Section~\ref{subsec-TWsbm})}. 
We use exactly the same combination of $\phi$, $\rho$ and $n$ as we did previously in Section~\ref{subsec:simulationVanillaTWsbm}, but only scenario 2 for the matrix $\beta_0$---the one with medium difficulty---for conciseness. 

For the covariates, we take $p=1$ so there is just one scalar covariate $x_{ij}$, which we generate independently for each $(i,j)$-pair from the uniform distributions on $(-1,1)$. The true covariate effect $\beta$ is simulated to be either weak ($\beta=1$) or strong ($\beta=2$). 

%%BEGIN MZ comment out
%%To emphasize that including the covariate effect in modelling Tweedie SBM can improve the clustering performance and the interpretability, we compare the Tweedie SBM with Covariates in Section~\ref{subsec-TWsbm} and the Vanilla Tweedie SBM in Section~\ref{subsec-vanillaTWsbm}. 
%%END MZ comment out

Table~\ref{simulation2_tab1} summarizes the results.
%%BEGIN MZ comment out/rewrite
%%One can see that these two methods have a notable difference in terms of NMI, especially when the underlying true covariate coefficient $\beta$ increases from $1$ to $2$. The performance of the Vanilla Tweedie SBM is less satisfactory.  
%%END MZ comment out/rewrite
Clearly, if there is a covariate $x_{ij}$ affecting the outcome $y_{ij}$, not taking it into account (and simply fitting a vanilla model) will significantly affect the clustering result, as measured by the NMI metric.
On the other hand, the mean and standard error of the estimate $\hat{\beta}$ over repeated simulation runs clearly validate the \myred{correctness of Theorem~\ref{theorem1_1} and the} effectiveness of our two-step algorithm\myred{---the covariate effects can indeed be estimated quite well with arbitrarily assigned community labels}. 

\begin{table}[ht] %htp
\centering
\begin{tabular}{|c|c|c|ccc|ccc|}
\hline
\multicolumn{3}{|c|}{} & \multicolumn{3}{c|}{Weak Effect ($\beta=1$)} & \multicolumn{3}{c|}{Strong Effect ($\beta=2$)} \\ \hline
 $\phi$ & $\rho$ & $n$ & \shortstack{\\NMI\\(excl.~$x_{ij}$)}   &  \shortstack{\\NMI\\(incl.~$x_{ij}$)}  & $\hat{\beta}$ &  \shortstack{\\NMI\\(excl.~$x_{ij}$)}   &  \shortstack{\\NMI\\(incl.~$x_{ij}$)}  & $\hat{\beta}$  \\ \hline
\multirow{6}{*}{$0.5$} &\multirow{2}{*}{$1.2$} & 50 & \shortstack{\\0.9804\\(0.009)}  & \shortstack{\\0.9794\\(0.01)} & \shortstack{\\1.0015\\(0.006)} & \shortstack{\\0.9289\\(0.015)} &  \shortstack{\\1\\(0)}& \shortstack{\\2.0067\\(0.006)}  \\ \cline{3-9}
&& 100 &\shortstack{\\1\\(0)} & \shortstack{\\1\\(0)} & \shortstack{\\0.9979\\(0.002)} &\shortstack{\\0.9976\\(0.001)} &\shortstack{\\1\\(0)}&\shortstack{\\1.9986\\(0.002)}\\ \cline{2-9}
&\multirow{2}{*}{$1.5$} & 50 & \shortstack{\\0.9626\\(0.013)} & \shortstack{\\0.9908\\(0.006)} & \shortstack{\\1.0128\\(0.005)} & \shortstack{\\0.9017\\(0.017)} & \shortstack{\\0.9986\\(0.001)} &  \shortstack{\\1.9969\\(0.005)} \\ \cline{3-9}
&& 100 &\shortstack{\\1\\(0)} &\shortstack{\\1\\(0)} & \shortstack{\\0.9994\\(0.003)} & \shortstack{\\0.9742\\(0.009)}&\shortstack{\\1\\(0)}&\shortstack{\\1.9995\\(0.002)} \\ \cline{2-9}
&\multirow{2}{*}{$1.8$} & 50 & \shortstack{\\0.9667\\(0.013)}  & \shortstack{\\0.9793\\(0.009)} & \shortstack{\\1.0083\\(0.005)} &\shortstack{\\0.8300\\(0.023)} & \shortstack{\\0.9883\\(0.007)} & \shortstack{\\2.0013\\(0.006)}\\  \cline{3-9}
&& 100 &\shortstack{\\1\\(0)} & \shortstack{\\1\\(0)} & \shortstack{\\0.9943\\(0.003)} &\shortstack{\\0.9731\\(0.007)}&\shortstack{\\1\\(0)}&\shortstack{\\1.9940\\(0.003)}\\
\hline \hline
\multirow{6}{*}{$1$} &\multirow{2}{*}{$1.2$} & 50 & \shortstack{\\0.9234\\(0.015)} & \shortstack{\\0.9344\\(0.016)} & \shortstack{\\0.9948\\(0.008)} & \shortstack{\\0.8335\\(0.022)} &  \shortstack{\\0.9846\\(0.008)} & \shortstack{\\1.9889\\(0.007)}  \\ \cline{3-9}
&& 100 & \shortstack{\\0.9984\\(0.001)}& \shortstack{\\1\\(0)} &\shortstack{\\1.0026\\(0.004)} &\shortstack{\\0.9597\\(0.009)}   & \shortstack{\\1\\(0)} &\shortstack{\\1.9974\\(0.003)}\\ \cline{2-9}
&\multirow{2}{*}{$1.5$} & 50 & \shortstack{\\0.8811\\(0.019)} &  \shortstack{\\0.9304\\(0.015)} & \shortstack{\\1.0039\\(0.006)} & \shortstack{\\0.7092\\(0.022)} & \shortstack{\\0.9687\\(0.011)} & \shortstack{\\1.9936\\(0.007)}\\ \cline{3-9}
&& 100 & \shortstack{\\0.9930\\(0.003)}& \shortstack{\\0.9992\\(0.001)} &\shortstack{\\0.9984\\(0.004)}&\shortstack{\\0.9225\\(0.011)}   & \shortstack{\\1\\(0)} &\shortstack{\\1.9948\\(0.004)} \\ \cline{2-9}
&\multirow{2}{*}{$1.8$} & 50 & \shortstack{\\0.8861\\(0.016)} & \shortstack{\\0.9404\\(0.012)} &  \shortstack{\\1.0176\\(0.007)} &\shortstack{\\0.5655\\(0.027)} &\shortstack{\\0.9234\\(0.015)}  & \shortstack{\\2.0155\\(0.008)}\\  \cline{3-9}
&& 100 &\shortstack{\\0.9877\\(0.007)} & \shortstack{\\0.9922\\(0.005)} &\shortstack{\\0.9946\\(0.004)}&\shortstack{\\0.8724\\(0.013)} & \shortstack{\\0.9945\\(0.004)} & \shortstack{\\1.9955\\(0.003)} \\
\hline \hline
\multirow{6}{*}{$2$} &\multirow{2}{*}{$1.2$} & 50 & \shortstack{\\0.7058\\(0.018)} & \shortstack{\\0.7699\\(0.018)} & \shortstack{\\0.9994\\(0.011)} & \shortstack{\\0.6262\\(0.022)} &  \shortstack{\\0.8621\\(0.018)} & \shortstack{\\2.0066\\(0.009)}  \\ \cline{3-9}
&& 100 &\shortstack{\\0.9542\\(0.009)} & \shortstack{\\0.976\\(0.008)} &\shortstack{\\0.9960\\(0.005)} & \shortstack{\\0.9022\\(0.011)}&\shortstack{\\0.9887\\(0.004)}&\shortstack{\\1.9902\\(0.004)}\\ \cline{2-9}
&\multirow{2}{*}{$1.5$} & 50 & \shortstack{\\0.6203\\(0.023)} & \shortstack{\\0.7028\\(0.021)} & \shortstack{\\1.0070\\(0.012)} & \shortstack{\\0.4602\\(0.022)}& \shortstack{\\0.7610\\(0.019)} & \shortstack{\\2.0025\\(0.012)} \\ \cline{3-9}
&& 100 &\shortstack{\\0.9015\\(0.012)} &\shortstack{\\0.9609\\(0.009)} & \shortstack{\\0.9868\\(0.005)}&\shortstack{\\0.7827\\(0.015)}   & \shortstack{\\0.9735\\(0.008)} & \shortstack{\\1.9885\\(0.006)} \\ \cline{2-9}
&\multirow{2}{*}{$1.8$} & 50 &\shortstack{\\0.5353\\(0.025)} &\shortstack{\\0.7114\\(0.024)}& \shortstack{\\1.0068\\(0.012)} & \shortstack{\\0.2581\\(0.028)}& \shortstack{\\0.7250\\(0.022)} & \shortstack{\\2.0236\\(0.011)} \\  \cline{3-9}
&& 100 & \shortstack{\\0.8477\\(0.016)} &\shortstack{\\0.9562\\(0.01)} & \shortstack{\\0.9892\\(0.005)} & \shortstack{\\0.6376\\(0.018)} & \shortstack{\\0.9403\\(0.013)} & \shortstack{\\1.9910\\(0.006)} \\
\hline
\end{tabular}
\caption{Summary of clustering and estimation performance from the static model with covariates over 50 simulation runs, with $(\beta_0^{kk},\beta_0^{kl})=(0.5,-0.5)$. }
\label{simulation2_tab1}
\end{table}

%\subsection{Simulation of the TV-Tweedie SBM}
\subsection{Simulation of time-varying model}
\label{subsec:simulationTVTWsbm}

We now study \myred{the most general, time-varying version of our model (Section~\ref{subsec-TVTWsbm}), having already established empirical evidence for the usefulness of the restricted Tweedie model in its vanilla form (Section~\ref{subsec:simulationVanillaTWsbm}) and the importance of taking covariates into account in a static setting (Section~\ref{subsec:simulationTWsbm}).} 

Instead of different combinations of $(\phi, \rho, n)$, these are now fixed at $\phi=1$, $\rho=1.5$, and $n=50$. 
%%BEGIN MZ rewrite: MZ rewrote below
%%The diagonal and off-diagonal entries, $(\beta_0^{kk},\beta_0^{kl})$, which characterizes the block matrix is generated in six ways: (i) $(1,0)$, (ii) $(0.5,-0.5)$, (iii) $(0,-1)$, (iiii) $(0.5,0)$, (v) $(0.25,-0.25)$, (vi) $(0,-0.5)$. 
%%to:
But we introduce three more scenarios for the true matrix $\beta_0$:
% \begin{equation*}
% \begin{array}{llcl}
% \text{scenario 4,} & (\beta_0^{kk},\beta_0^{kl})=(0.50,\phantom{-}0.00) & \Rightarrow & \exp(\beta_0^{kk})-\exp(\beta_0^{kl}) \approx 0.65;\\
% \text{scenario 5,} & (\beta_0^{kk},\beta_0^{kl})=(0.25,-0.25) & \Rightarrow & \exp(\beta_0^{kk})-\exp(\beta_0^{kl})\approx 0.51;\\
% \text{scenario 6,} & (\beta_0^{kk},\beta_0^{kl})=(0.00,-0.50) & \Rightarrow & \exp(\beta_0^{kk})-\exp(\beta_0^{kl})\approx 0.39.
% \end{array}
% \end{equation*}
\begin{align*}
\text{scenario 4, }  (\beta_0^{kk},\beta_0^{kl})=(0.50,\phantom{-}0.00)  \Rightarrow & \exp(\beta_0^{kk})-\exp(\beta_0^{kl}) \approx 0.65;\\
\text{scenario 5, }  (\beta_0^{kk},\beta_0^{kl})=(0.25,-0.25)  \Rightarrow & \exp(\beta_0^{kk})-\exp(\beta_0^{kl})\approx 0.51;\\
\text{scenario 6, }  (\beta_0^{kk},\beta_0^{kl})=(0.00,-0.50)  \Rightarrow & \exp(\beta_0^{kk})-\exp(\beta_0^{kl})\approx 0.39.
\end{align*}
\myred{These are similar to the earlier scenarios 1, 2 and 3, but respectively more difficult to cluster.} 
%%BEGIN MZ rewrite: MZ rewrote below

We generate one scalar covariate $x_{ij}$ in exactly the same way as we did in Section~\ref{subsec:simulationTWsbm}, except that its effect is now time-varying, with coefficient $\beta(t)$ generated in six different ways: (i) $\beta(t)=2t-1$, (ii) $\beta(t)=\sin (2\pi t)$, (iii) $\beta(t)=2t$, (iiii) $\beta(t)=\sin (2\pi t)+1$, (v) $\beta(t)=0.5(2t-1)$, and (vi) $\beta(t)= 0.5 \sin (2\pi t)$. Finally, the data sets are simulated in such a way that the network is observed at $T=20$ equally spaced time points on $[0, 1]$.

We use 10 different sets of initial values for each simulation run. To evaluate the performance of the estimated $\hat{\beta}(t)$, we calculate the estimation error as
\begin{align*}
    \text{Err}(\hat{\beta}(t))= \frac{1}{20} \scsum{\nu=1}{20} [\hat{\beta}(t_{\nu}) - \beta(t_{\nu})]^2.
\end{align*}

In general, the tuning parameter $\lambda$ is to be selected by cross-validation (see Section~\ref{subsec:tuning}). To reduce computational cost, we simply fix it at $\lambda=0.5$ for the current simulation study. Appendix~\ref{Appendix3:SensitivityTuning} contains a small sensitivity analysis using $\lambda=0.1 < 0.5$ and $\lambda=1.0 > 0.5$, from which one can see that it makes little difference whether $\lambda=0.1$, $0.5$ or $1.0$ is used in this study.

For all simulated cases with different combinations of $\beta_0$ and $\beta(t)$, Table~\ref{sbm-tab:sim-TVTweedie} summarizes the two metrics, NMI and $\text{Err}(\hat{\beta}(t))$, while Figure~\ref{fig:sim_TV_betat} displays the true function $\beta(t)$ together with the pointwise mean and standard deviation of $\hat{\beta}(t)$, over repeated simulation runs. The standard deviation is hard to visualize because it is very small at all $t$.

\myred{Theorem~\ref{theorem1_1} again explains why the varying-coefficient $\beta(t)$ can be estimated so well. Once $\beta(t)$ has been estimated, the community structure is actually easier to detect with time-varying data than it is with static data because, for each pair $(i,j)$, observations at all time points, $\{y_{ij}(t_\nu)\}_{\nu=1}^T$, contain this information, not just a single observation $y_{ij}$.}

\begin{table}[hp] %htp
\centering
\begin{tabular}{c|c|cccccc}
\hline
\multirow{2}{*}{$(\beta_0^{kk},\beta_0^{kl})$} &
&\multicolumn{6}{c}{$\beta(t)$} \\ 
  & &    $2t-1$     &   $\sin (2\pi t)$     &    $2t$   &    $\sin (2\pi t)$ +1       &   $0.5(2t-1)$      &   $ 0.5 \sin (2\pi t)$  \\ \hline \hline
Scenario 1 & NMI & \shortstack{\\1\\(0)} & \shortstack{\\1\\(0)} & \shortstack{\\1\\(0)} & \shortstack{\\1\\(0)} & \shortstack{\\0.996\\(0.004)} & \shortstack{\\1\\(0)} \\ \cline{2-8}
$(1,0)$ &$\text{Err}(\hat{\beta}(t))$ &\shortstack{\\0.004\\(0)}&\shortstack{\\0.026\\(0)} & \shortstack{\\0.004\\(0)} & \shortstack{\\0.025\\(0)} & \shortstack{\\0.004\\(0)} &  \shortstack{\\0.013\\(0)} \\ \hline \hline
Scenario 2 &NMI & \shortstack{\\1\\(0)} & \shortstack{\\1\\(0)} & \shortstack{\\1\\(0)} & \shortstack{\\1\\(0)} & \shortstack{\\1\\(0)} & \shortstack{\\1\\(0)} \\ \cline{2-8}
$(0.5,-0.5)$ &$\text{Err}(\hat{\beta}(t))$ & \shortstack{\\0.005\\(0)} &\shortstack{\\0.031\\(0)} &\shortstack{\\0.005\\(0)}&\shortstack{\\0.029\\(0)}&\shortstack{\\0.005\\(0)}&\shortstack{\\0.016\\(0)}\\ \hline \hline
Scenario 3&NMI & \shortstack{\\1\\(0)} & \shortstack{\\1\\(0)} & \shortstack{\\1\\(0)} & \shortstack{\\0.996\\(0.004)} & \shortstack{\\1\\(0)} & \shortstack{\\1\\(0)} \\ \cline{2-8}
$(0,-1)$&$\text{Err}(\hat{\beta}(t))$ &\shortstack{\\0.005\\(0)}&\shortstack{\\0.037\\(0)} &\shortstack{\\0.005\\(0)}&\shortstack{\\0.035\\(0)}&\shortstack{\\0.006\\(0)}& \shortstack{\\0.019\\(0)} \\ \hline \hline
Scenario 4 &NMI &\shortstack{\\0.996\\(0.004)}&\shortstack{\\1\\(0)}&\shortstack{\\1\\(0)}&\shortstack{\\1\\(0)}&\shortstack{\\1\\(0)}& \shortstack{\\1\\(0)} \\ \cline{2-8}
$(0.5,0)$ &$\text{Err}(\hat{\beta}(t))$ &\shortstack{\\0.004\\(0)}&\shortstack{\\0.029\\(0)}&\shortstack{\\0.005\\(0)}& \shortstack{\\0.027\\(0)} & \shortstack{\\0.005\\(0)} & \shortstack{\\0.015\\(0)} \\ \hline \hline
Scenario 5&NMI & \shortstack{\\1\\(0)} & \shortstack{\\1\\(0)} & \shortstack{\\1\\(0)} & \shortstack{\\1\\(0)} & \shortstack{\\1\\(0)} & \shortstack{\\1\\(0)} \\ \cline{2-8}
$(0.25,-0.25)$&$\text{Err}(\hat{\beta}(t))$ &\shortstack{\\0.005\\(0)}&\shortstack{\\0.031\\(0)}&\shortstack{\\0.005\\(0)} &\shortstack{\\0.03\\(0)}&\shortstack{\\0.005\\(0)}&\shortstack{\\0.016\\(0)}\\ \hline \hline
Scenario 6 &NMI & \shortstack{\\1\\(0)} &  \shortstack{\\0.996\\(0.004)} & \shortstack{\\1\\(0)} &  \shortstack{\\0.996\\(0.004)} & \shortstack{\\1\\(0)} &  \shortstack{\\0.996\\(0.004)} \\ \cline{2-8}
$(0,-0.5)$ & $\text{Err}(\hat{\beta}(t))$ & \shortstack{\\0.005\\(0)} &  \shortstack{\\0.034\\(0)} & \shortstack{\\0.005\\(0)} & \shortstack{\\0.033\\(0)} & \shortstack{\\0.005\\(0)} & \shortstack{\\0.017\\(0)} \\ \hline 
\end{tabular}
\caption{Summary of clustering and estimation performance (using $\lambda=0.5$) from the time-varying model over 50 simulation runs, with $\phi=1$, $\rho=1.5$ and $n=50$.}
\label{sbm-tab:sim-TVTweedie}
\end{table}

\begin{figure}[h]
    \centering
\includegraphics[height=0.32\paperheight]{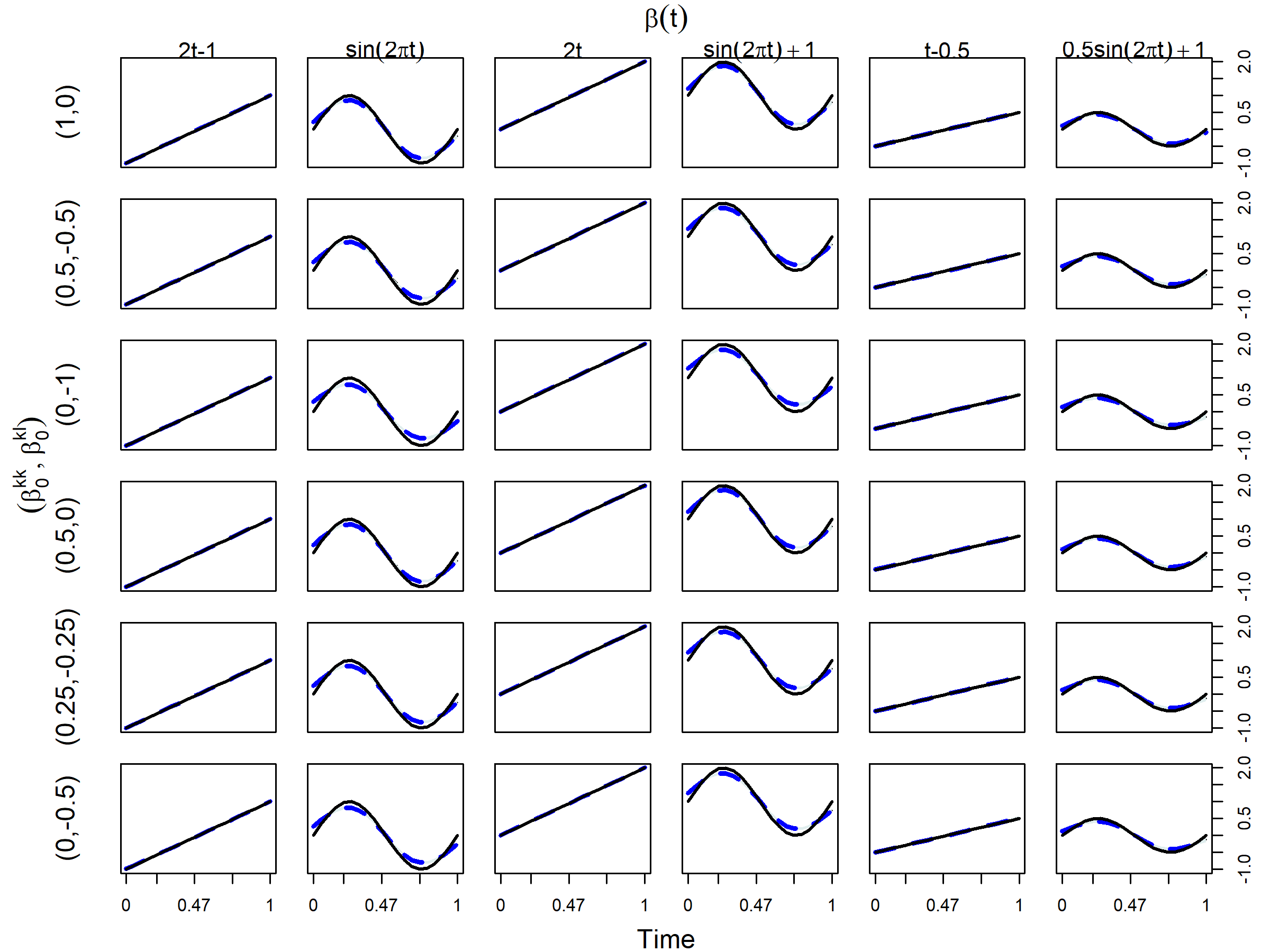}
    \caption{Estimations of $\beta(t)$ using a tuning parameter of $\lambda=0.5$ for all simulated cases with different combinations of $\beta_0$ and $\beta(t)$. In each panel, the black solid line is the true function $\beta(t)$; the blue dashed line is the pointwise mean of $\hat{\beta}(t)$; and the light blue shadow (hardly visible) marks the corresponding pointwise confidence band.}
    \label{fig:sim_TV_betat}
\end{figure}

\begin{comment}
\begin{table}[ht] %htp
\centering
\begin{tabular}{c|c|cccccc}
\hline
\multirow{2}{*}{$(\beta_0^{kk},\beta_0^{kl})$} & &\multicolumn{6}{c}{$\beta(t)$} \\ 
                &  &    $2t-1$     &   $\sin (2\pi t)$     &    $2t$   &    $\sin (2\pi t)$ +1       &   $0.5(2t-1)$      &   $ 0.5 \sin (2\pi t)$  \\ \hline
 \multirow{2}{*}{$1,0$} & NMI & & &  &  &  &  \\ \cline{2-8}
 & $\text{e}(\hat{\beta}(t))$ & & &  &  &  &  \\ \hline
\multirow{2}{*}{$0.5,-0.5$}& &  &  &  &  &  &   \\ \cline{2-8}
& &  &  &  &  &  &   \\ \hline
\multirow{2}{*}{$0,-1$} & & &  &  &   &  &   \\ \cline{2-8}
& &  &  &  &  &  &   \\ \hline
\hline  
\multirow{2}{*}{$0.5,0$}  & & & &  &  &  &  \\ \cline{2-8}
& &  &  &  &  &  &   \\ \hline
\multirow{2}{*}{$1/4,-1/4$} & &  &  &  &  &  &   \\ \cline{2-8}
& &  &  &  &  &  &   \\ \hline
\multirow{2}{*}{$0,-0.5$} & & & &  &  &  &  \\ \cline{2-8} 
& &  &  &  &  &  &   \\ \hline
\end{tabular}
\caption{\textbf{Summary of NMI and the estimation error of $\beta(t)$ in Time-varying Two-Steps Tweedie SBM.} The above statistics are calculated over 50 simulations, with $\phi=1$, $\rho=1.5$, sample size 50 and $\lambda=0.5$. }
\label{sbm-tab:sim-TVTweedie0}
\end{table}
\end{comment}

%\section{Application: Apple Trading Community and Distance Effects}

\section{Application: International Trading}
\label{sec:application}

%%\myred{[MZ has now edited this section]}

In this section, we apply the restricted Tweedie SBM to study international trading relationships among different countries and how these relationships are influenced by geographical distances. 
%%BEGIN MZ comment out
%%International agricultural trade is essential for ensuring food security, promoting economic growth, and identifying environmental impact. Studying the relationships of international agricultural trading between countries and clustering different countries into different groups based on their international agricultural trading intensity can provide valuable insights for market analysis, policy analysis, and risk analysis.
%%END MZ comment out
As an example, we focus on the trading of apples---not only are these data readily available from the World Bank \citep{wits2023},
%%BEGIN MZ comment out
%%one of the most widely consumed fruits worldwide, and international apple trading is essential for ensuring food security and promoting economic growth. By using our methods, we would be able to cluster the countries into different groups so that countries in the same group have similar apple trading patterns, implying their similar demand patterns or consumption preferences. This provides valuable information for market analysis. For example, countries in the same community can collaborate in policy making and risk management. 
%While the trading community is significant to apple trading, 
%%END MZ comment out
but one can also surmise {\it a priori} that geographical distances will likely have a substantial impact on the trading due to the heavyweight and perishable nature of this product.

From the international trading data sets provided by the World Bank \citep{wits2023}, we have collected annual import and export values of edible and fresh apples among $n=66$ countries from $t_1=2002$ to $t_{20}=2021$. In each given year $t_{\nu}$, we observe a 66-by-66 matrix $Y(t_{\nu})$ where each cell $y_{ij}(t_{\nu})$ represents the trading value from country $i$ to country $j$ in thousands of US dollars during that year. We then average $Y(t_{\nu})$ with its transpose to ensure symmetry. 
%Finally a very small proportion entries whose values are in $(0,1)$ are thresholded to $0$, and logarithm is taken for the rest entries. 
Finally, a small number of entries with values ranging from 0 to 1 (i.e., total trading values less than \$1,000) are thresholded to 0, and the remaining entries are logarithmically transformed. For the covariate $x_{ij}$, we use the shortest geographical distance between the two trading countries based on their borders, which we calculate using the R packages \texttt{maps} and \texttt{geosphere}.

We employ the cross-validation procedure outlined in Section~\ref{subsec:tuning} to choose the tuning parameter $\lambda$. Figure~\ref{fig:application_cv} displays the CV error, showing the optimal tuning parameter to be $\lambda^*=0.1$.

Table~\ref{tab:trading-community} shows how the $66$ countries are clustered into three communities by our method.
Figure~\ref{fig:application_country} displays the aggregated matrix, $Y(2002) + Y(2003) + \dots + Y(2021)$, with rows and columns having been permuted according to the inferred community labels. Clearly, countries in the first community trade intensively with each other and with countries in the third community. While both the second and third communities consist of countries that mainly trade with countries in the first community (rather than among themselves or between each other), the trading intensity with the first community is lot higher for the third community than it is for the second.

Figure~\ref{fig:application_betat} displays $\hat{\beta}(t)$, the estimated effect of geographical distances on apple trading over time. We can make three prominent observations. First, the function $\hat{\beta}(t)$ is negative over the entire time period being studied---%,which indicates that a larger distance between two regions will decrease the intensity of the apple trading between them. This is 
not surprising since longer distances can only increase the cost and time of transportation, and negatively impact fresh apple trading. Next, generally speaking the magnitude of $\hat{\beta}(t)$ is decreasing over the twenty-year period, implying that the negative effect of geographical distances is diminishing. This may be attributed to more efficient method and reduced cost of shipment overtime. Finally, two relatively big ``dips'' in $\hat{\beta}(t)$ are clearly visible---one after the financial crisis in 2008, and another after the onset of the Covid-19 pandemic in 2020.

\begin{figure}[ht]
    \centering
    \includegraphics[width=0.6\textwidth]{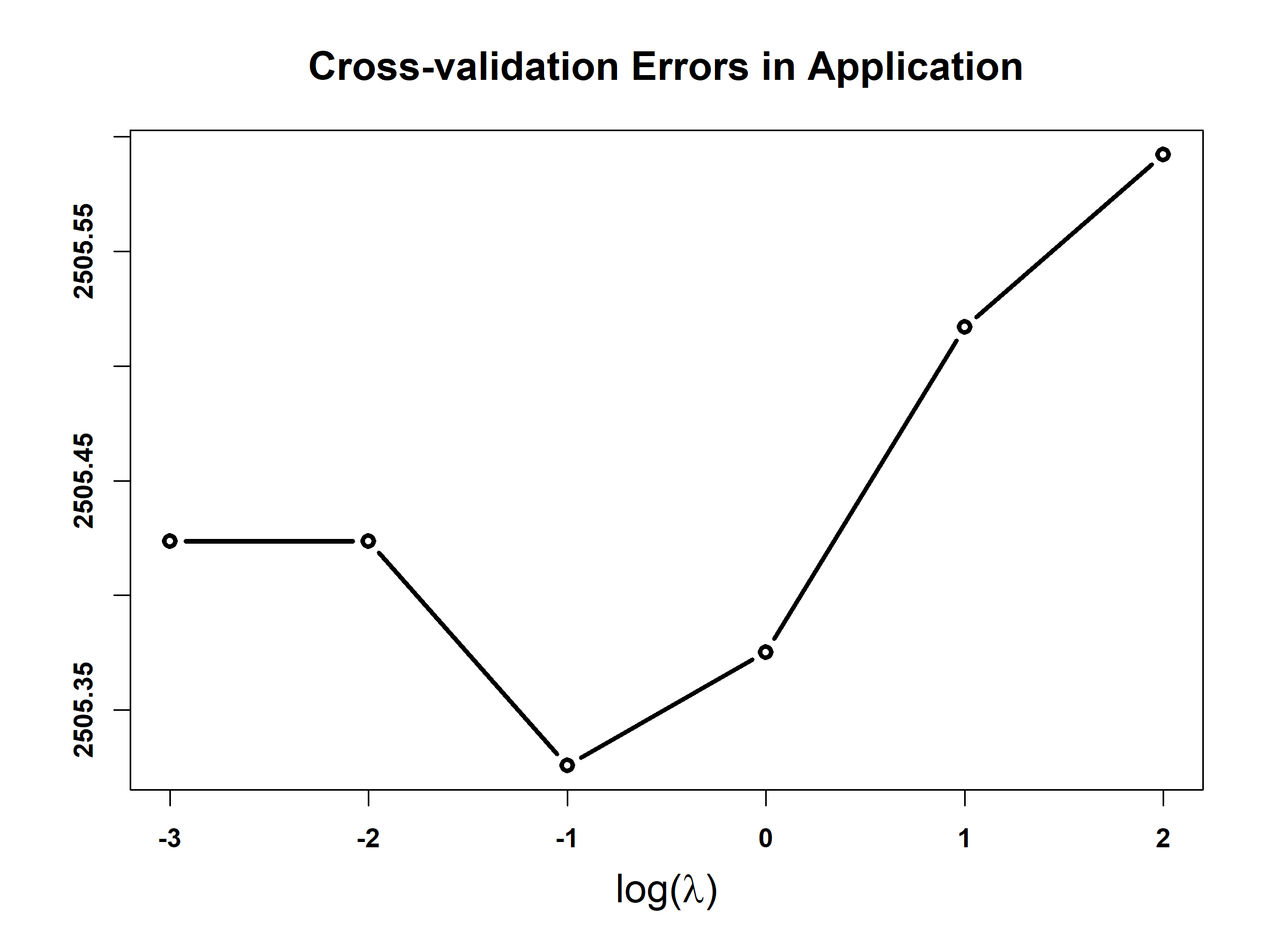}
    \caption{Cross validation errors change across a range
of plausible values for the tuning parameter $\lambda$.}
    \label{fig:application_cv}
\end{figure}

\begin{figure}[ht]
    \centering
    \includegraphics[width=1\textwidth]{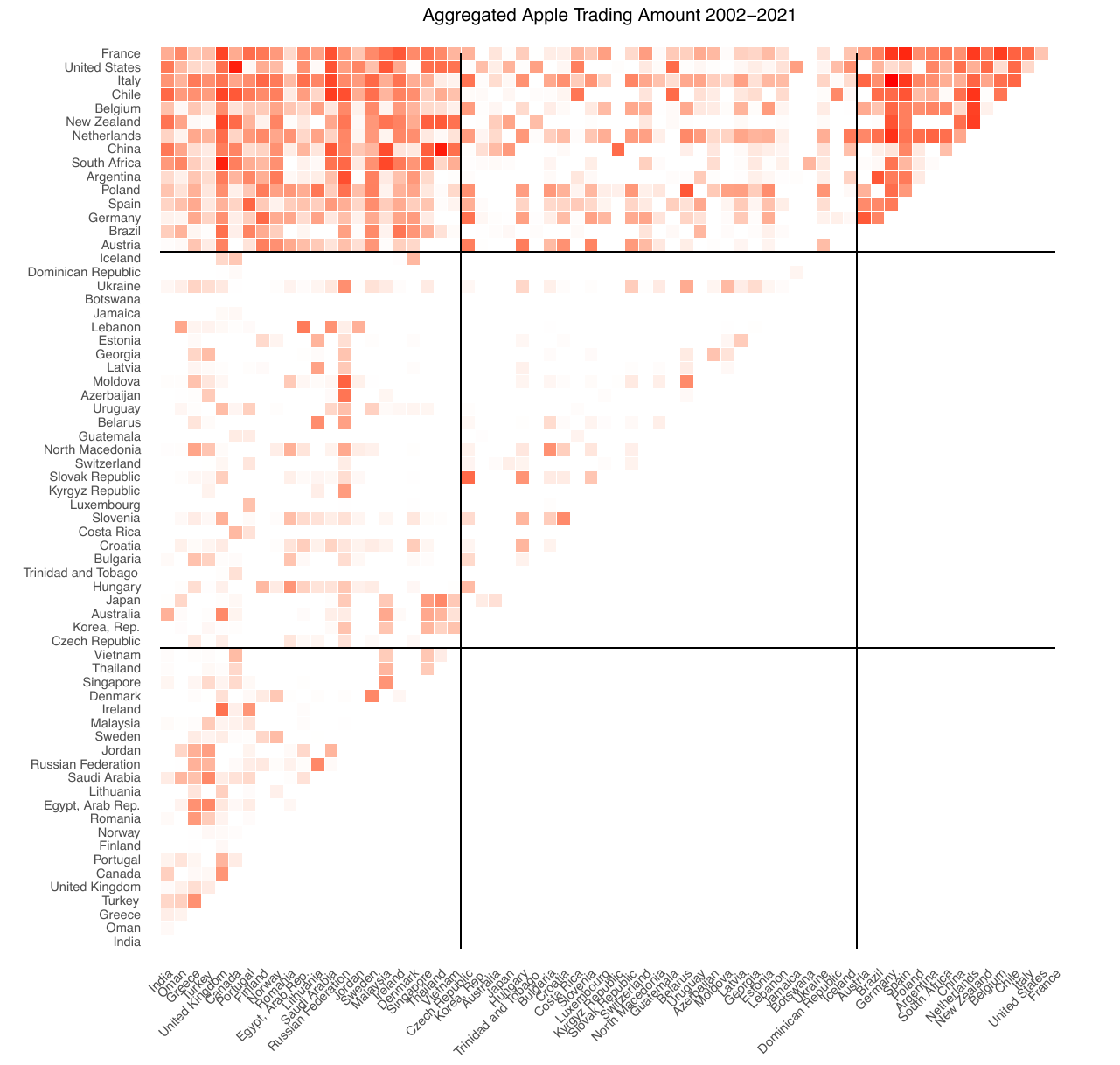}
    \caption{The (aggregated) matrix, $Y(2002)+\dots+Y(2021)$, with rows and columns having been permuted according to the inferred community labels. Due to symmetry, only the upper half of the matrix is shown, with color shadings being proportional to each entry's respective magnitude.}
    \label{fig:application_country}
\end{figure}

\begin{figure}[ht]
    \centering
    \includegraphics[width=0.8\textwidth,angle=270]{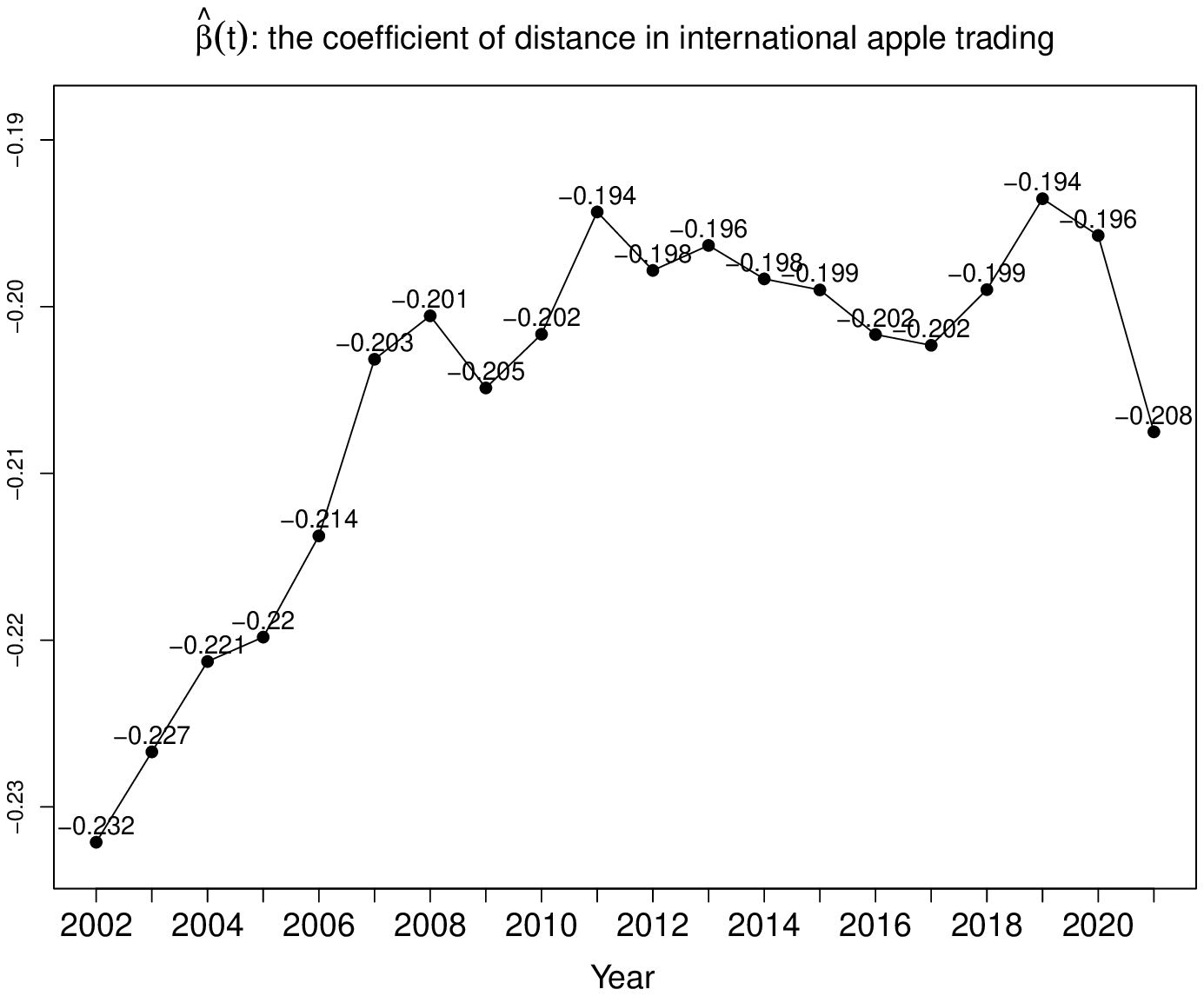}
    \caption{Estimated covariate coefficient $\hat{\beta}(t)$ for $\lambda^* = 0.1$.}
    \label{fig:application_betat}
\end{figure}

\begin{table}[ht] %htp
\centering
\begin{tabular}{p{2cm}|p{9cm} } 
  \hline
Community & \multicolumn{1}{c}{Country} \\ [0.5ex] 
\hline
 1 & France, United States, Italy, Chile, Belgium, New Zealand, Netherlands, China, South Africa, Argentina, Poland, Spain, Germany, Brazil, Austria\\ 
 \hline 
 2 & Iceland, Dominican Republic, Ukraine, Botswana, Jamaica, Lebanon, Estonia, Georgia, Latvia, Moldova, Azerbaijan, Uruguay, Belarus, Guatemala, North, Macedonia, Switzerland, Slovak, Republic, Kyrgyz Republic, Luxembourg, Slovenia, Costa Rica, Croatia, Bulgaria, Trinidad and Tobago, Hungary, Japan, Australia, Korea Rep, Czech Republic\\ 
  \hline 
  3 & Vietnam, Thailand, Singapore, Denmark, Ireland, Malaysia, Sweden, Jordan, Russian, Federation, Saudi, Arabia, Lithuania, Egypt Arab Rep., Romania, Norway, Finland, Portugal, Canada, United Kingdom, Turkey, Greece, Oman, India
  \\ 
  \hline 
\end{tabular}
\caption{Community detection results of 66 countries.}
\label{tab:trading-community}
\end{table}

\section{Discussion}

This paper generalizes the vanilla SBM by replacing the Bernoulli distribution with the restricted Tweedie distribution to accommodate non-negative zero-inflated continuous edge weights. Moreover, our model accounts for dynamic effects of nodal information. We show that as the number of nodes diverges to infinity, estimating the covariates coefficients is asymptotically irrelevant to the community labels when we maximize the likelihood function. This startling finding leads to the efficient two-step algorithm. Applying our framework  to the international apple trading data provides insight into the dynamic effect of the geographic distance between countries in the trading network. 
%Our analysis of the international apple trading data set also leads to some interesting discoveries, as the trading participants are clustered into different communities and the dynamic effect of the distance on the trading is also estimated. 

Moreover, simulation studies in Section \ref{sec:simulation} demonstrates the appealing performance of the proposed framework in clustering. 
This can be attributed to time independent community labels for each node, as the temporal data provide sufficient information for inferring the community labels. However, in many real world dynamic networks, the community label of each node is time dependent; it renders our current framework inapplicable. 
\cite{xu2014dynamic} and \cite{Matias2017} proposed to use a Markov chain to address this problem, but there exist idenfitiability issues for parameters to be resolved in future work. 

\bibliographystyle{plainnat}
\bibliography{sbm}

\begin{thebibliography}{41}
\providecommand{\natexlab}[1]{#1}
\providecommand{\url}[1]{\texttt{#1}}
\expandafter\ifx\csname urlstyle\endcsname\relax
  \providecommand{\doi}[1]{doi: #1}\else
  \providecommand{\doi}{doi: \begingroup \urlstyle{rm}\Url}\fi

\bibitem[Aicher et~al.(2013)Aicher, Jacobs, and Clauset]{aicher2013adapting}
Christopher Aicher, Abigail~Z Jacobs, and Aaron Clauset.
\newblock {Adapting the stochastic block model to edge-weighted networks}.
\newblock In \emph{ICML Workshop on Structured Learning (SLG)}, 2013.

\bibitem[Aicher et~al.(2015)Aicher, Jacobs, and Clauset]{aicher2015learning}
Christopher Aicher, Abigail~Z Jacobs, and Aaron Clauset.
\newblock Learning latent block structure in weighted networks.
\newblock \emph{Journal of Complex Networks}, 3\penalty0 (2):\penalty0
  221--248, 2015.

\bibitem[Bakhthemmat and Izadi(2021)]{bakhthemmat2021communities}
Ali Bakhthemmat and Mohammad Izadi.
\newblock Communities detection for advertising by futuristic greedy method
  with clustering approach.
\newblock \emph{Big Data}, 9\penalty0 (1):\penalty0 22--40, 2021.

\bibitem[Bedi and Sharma(2016)]{bedi2016community}
Punam Bedi and Chhavi Sharma.
\newblock Community detection in social networks.
\newblock \emph{Wiley Interdisciplinary Reviews: Data Mining and Knowledge
  Discovery}, 6\penalty0 (3):\penalty0 115--135, 2016.

\bibitem[Bhattacharjee et~al.(2020)Bhattacharjee, Banerjee, and
  Michailidis]{Bhattacharjee2020}
Monika Bhattacharjee, Moulinath Banerjee, and George Michailidis.
\newblock Change point estimation in a dynamic stochastic block model.
\newblock \emph{Journal of Machine Learning Research}, 21\penalty0
  (107):\penalty0 1--59, 2020.

\bibitem[Chen and Mo(2022)]{chen2022community}
Yan Chen and Dongxu Mo.
\newblock Community detection for multilayer weighted networks.
\newblock \emph{Information Sciences}, 595:\penalty0 119--141, 2022.

\bibitem[Choi et~al.(2012)Choi, Wolfe, and Airoldi]{choi2012stochastic}
David~S Choi, Patrick~J Wolfe, and Edoardo~M Airoldi.
\newblock Stochastic blockmodels with a growing number of classes.
\newblock \emph{Biometrika}, 99\penalty0 (2):\penalty0 273--284, 2012.

\bibitem[Danon et~al.(2005)Danon, Diaz-Guilera, Duch, and
  Arenas]{danon2005comparing}
Leon Danon, Albert Diaz-Guilera, Jordi Duch, and Alex Arenas.
\newblock Comparing community structure identification.
\newblock \emph{Journal of Statistical Mechanics: Theory and Experiment},
  2005\penalty0 (09):\penalty0 P09008, 2005.

\bibitem[Dunn and Smyth(2005)]{dunn2005series}
Peter~K Dunn and Gordon~K Smyth.
\newblock Series evaluation of {Tweedie} exponential dispersion model
  densities.
\newblock \emph{Statistics and Computing}, 15\penalty0 (4):\penalty0 267--280,
  2005.

\bibitem[Dunn and Smyth(2008)]{dunn2008evaluation}
Peter~K Dunn and Gordon~K Smyth.
\newblock Evaluation of tweedie exponential dispersion model densities by
  {Fourier} inversion.
\newblock \emph{Statistics and Computing}, 18\penalty0 (1):\penalty0 73--86,
  2008.

\bibitem[Fu et~al.(2009)Fu, Song, and Xing]{fu2009dynamic}
Wenjie Fu, Le~Song, and Eric~P Xing.
\newblock Dynamic mixed membership blockmodel for evolving networks.
\newblock In \emph{Proceedings of the 26th Annual International Conference on
  Machine Learning}, pages 329--336, 2009.

\bibitem[Gasparetti et~al.(2021)Gasparetti, Sansonetti, and
  Micarelli]{gasparetti2021community}
Fabio Gasparetti, Giuseppe Sansonetti, and Alessandro Micarelli.
\newblock Community detection in social recommender systems: a survey.
\newblock \emph{Applied Intelligence}, 51\penalty0 (6):\penalty0 3975--3995,
  2021.

\bibitem[Green and Silverman(1993)]{green1993nonparametric}
Peter~J Green and Bernard~W Silverman.
\newblock \emph{Nonparametric Regression and Generalized Linear Models: a
  Roughness Penalty Approach}.
\newblock Chapman and Hall, New York, 1993.

\bibitem[Guerrero-Sol{\'e}(2017)]{guerrero2017community}
Frederic Guerrero-Sol{\'e}.
\newblock Community detection in political discussions on twitter: An
  application of the retweet overlap network method to the {Catalan} process
  toward independence.
\newblock \emph{Social Science Computer Review}, 35\penalty0 (2):\penalty0
  244--261, 2017.

\bibitem[Haj et~al.(2022)Haj, Slaoui, Louis, and Khraibani]{haj2022estimation}
Abir~El Haj, Yousri Slaoui, Pierre-Yves Louis, and Zaher Khraibani.
\newblock Estimation in a binomial stochastic blockmodel for a weighted graph
  by a variational expectation maximization algorithm.
\newblock \emph{Communications in Statistics-Simulation and Computation},
  51\penalty0 (8):\penalty0 4450--4469, 2022.

\bibitem[Hastie et~al.(2009)Hastie, Tibshirani, Friedman, and
  Friedman]{hastie2009elements}
Trevor Hastie, Robert Tibshirani, Jerome~H Friedman, and Jerome~H Friedman.
\newblock \emph{The Elements of Statistical Learning: Data Mining, Inference,
  and Prediction}.
\newblock Springer, New York, 2 edition, 2009.

\bibitem[Hoff et~al.(2002)Hoff, Raftery, and Handcock]{hoff2002latent}
Peter~D Hoff, Adrian~E Raftery, and Mark~S Handcock.
\newblock Latent space approaches to social network analysis.
\newblock \emph{Journal of the American Statistical Association}, 97\penalty0
  (460):\penalty0 1090--1098, 2002.

\bibitem[Holland et~al.(1983)Holland, Laskey, and
  Leinhardt]{holland1983stochastic}
Paul~W Holland, Kathryn~Blackmond Laskey, and Samuel Leinhardt.
\newblock Stochastic blockmodels: First steps.
\newblock \emph{Social Networks}, 5\penalty0 (2):\penalty0 109--137, 1983.

\bibitem[Huang et~al.(2023)Huang, Sun, and Feng]{huang2023pairwise}
Sihan Huang, Jiajin Sun, and Yang Feng.
\newblock {PCABM}: Pairwise covariates-adjusted block model for community
  detection.
\newblock \emph{Journal of the American Statistical Association}, 0\penalty0
  (0):\penalty0 1--26, 2023.

\bibitem[Lee and Wilkinson(2019)]{lee2019review}
Clement Lee and Darren~J Wilkinson.
\newblock A review of stochastic block models and extensions for graph
  clustering.
\newblock \emph{Applied Network Science}, 4\penalty0 (1):\penalty0 1--50, 2019.

\bibitem[Lian et~al.(2023)Lian, Yang, Wang, Shi, and Platt]{lian2023tweedie}
Yi~Lian, Archer~Yi Yang, Boxiang Wang, Peng Shi, and Robert~William Platt.
\newblock A {Tweedie} compound {Poisson} model in reproducing kernel {Hilbert}
  space.
\newblock \emph{Technometrics}, 65\penalty0 (2):\penalty0 281--295, 2023.

\bibitem[Lu et~al.(2018)Lu, Wahlstr{\"o}m, and Nehorai]{lu2018community}
Zhenqi Lu, Johan Wahlstr{\"o}m, and Arye Nehorai.
\newblock Community detection in complex networks via clique conductance.
\newblock \emph{Scientific Reports}, 8\penalty0 (1):\penalty0 1--16, 2018.

\bibitem[Ludkin(2020)]{ludkin2020inference}
Matthew Ludkin.
\newblock Inference for a generalised stochastic block model with unknown
  number of blocks and non-conjugate edge models.
\newblock \emph{Computational Statistics \& Data Analysis}, 152:\penalty0
  107051, 2020.

\bibitem[Ma et~al.(2020)Ma, Ma, and Yuan]{ma2017exploration}
Zhuang Ma, Zongming Ma, and Hongsong Yuan.
\newblock Universal latent space model fitting for large networks with edge
  covariates.
\newblock \emph{Journal of Machine Learning Research}, 21\penalty0
  (4):\penalty0 1--67, 2020.

\bibitem[MacDonald et~al.(2022)MacDonald, Levina, and Zhu]{macdonald2022latent}
Peter~W MacDonald, Elizaveta Levina, and Ji~Zhu.
\newblock Latent space models for multiplex networks with shared structure.
\newblock \emph{Biometrika}, 109\penalty0 (3):\penalty0 683--706, 2022.

\bibitem[Mariadassou et~al.(2010)Mariadassou, Robin, and
  Vacher]{mariadassou2010uncovering}
Mahendra Mariadassou, St{\'e}phane Robin, and Corinne Vacher.
\newblock Uncovering latent structure in valued graphs: a variational approach.
\newblock \emph{The Annals of Applied Statistics}, 4\penalty0 (2):\penalty0
  715--742, 2010.

\bibitem[Matias and Miele(2017)]{Matias2017}
Catherine Matias and Vincent Miele.
\newblock Statistical clustering of temporal networks through a dynamic
  stochastic block model.
\newblock \emph{Journal of the Royal Statistical Society. Series B.},
  79\penalty0 (4):\penalty0 1119--1141, 2017.
\newblock ISSN 1369-7412.

\bibitem[Motalebi et~al.(2021)Motalebi, Stevens, and
  Steiner]{NathanielStevens2021}
Narges Motalebi, Nathaniel~T. Stevens, and Stefan~H. Steiner.
\newblock Hurdle blockmodels for sparse network modeling.
\newblock \emph{The American Statistician}, 75\penalty0 (4):\penalty0 383--393,
  2021.

\bibitem[Ng and Murphy(2021)]{ng2021weighted}
Tin Lok~James Ng and Thomas~Brendan Murphy.
\newblock Weighted stochastic block model.
\newblock \emph{Statistical Methods \& Applications}, 30:\penalty0 1365--1398,
  2021.

\bibitem[Peixoto(2018)]{peixoto2018nonparametric}
Tiago~P Peixoto.
\newblock Nonparametric weighted stochastic block models.
\newblock \emph{Physical Review E}, 97\penalty0 (1):\penalty0 012306, 2018.

\bibitem[Roy et~al.(2019)Roy, Atchad{\'e}, and Michailidis]{roy2019likelihood}
Sandipan Roy, Yves Atchad{\'e}, and George Michailidis.
\newblock Likelihood inference for large scale stochastic blockmodels with
  covariates based on a divide-and-conquer parallelizable algorithm with
  communication.
\newblock \emph{Journal of Computational and Graphical Statistics}, 28\penalty0
  (3):\penalty0 609--619, 2019.

\bibitem[Silverman(1985)]{silverman1985some}
Bernhard~W Silverman.
\newblock Some aspects of the spline smoothing approach to non-parametric
  regression curve fitting.
\newblock \emph{Journal of the Royal Statistical Society: Series B
  (Methodological)}, 47\penalty0 (1):\penalty0 1--21, 1985.

\bibitem[Tallberg(2004)]{tallberg2004bayesian}
Christian Tallberg.
\newblock A {Bayesian} approach to modeling stochastic blockstructures with
  covariates.
\newblock \emph{Journal of Mathematical Sociology}, 29\penalty0 (1):\penalty0
  1--23, 2004.

\bibitem[Tweedie(1984)]{tweedie1984index}
Maurice~CK Tweedie.
\newblock An index which distinguishes between some important exponential
  families.
\newblock In \emph{Statistics: Applications and New Directions: Proc. Indian
  Statistical Institute Golden Jubilee International Conference}, volume 579,
  pages 579--604, 1984.

\bibitem[Vu et~al.(2013)Vu, Hunter, and Schweinberger]{vu2013model}
Duy~Q Vu, David~R Hunter, and Michael Schweinberger.
\newblock Model-based clustering of large networks.
\newblock \emph{The Annals of Applied Statistics}, 7\penalty0 (2):\penalty0
  1010, 2013.

\bibitem[{World Integrated Trade Solution}(2023)]{wits2023}
{World Integrated Trade Solution}.
\newblock {International merchandise trade, tariff and non-tariff measures
  (NTM) data}.
\newblock \url{https://wits.worldbank.org/}, 2023.

\bibitem[Xin et~al.(2017)Xin, Zhu, and Chipman]{xin2017continuous}
Lu~Xin, Mu~Zhu, and Hugh Chipman.
\newblock A continuous-time stochastic block model for basketball networks.
\newblock \emph{The Annals of Applied Statistics}, 11\penalty0 (2):\penalty0
  553--597, 2017.

\bibitem[Xing et~al.(2010)Xing, Fu, and Song]{xing2010state}
Eric~P Xing, Wenjie Fu, and Le~Song.
\newblock A state-space mixed membership blockmodel for dynamic network
  tomography.
\newblock \emph{The Annals of Applied Statistics}, 4\penalty0 (2):\penalty0
  535--566, 2010.

\bibitem[Xu and Hero(2014)]{xu2014dynamic}
Kevin~S Xu and Alfred~O Hero.
\newblock Dynamic stochastic blockmodels for time-evolving social networks.
\newblock \emph{IEEE Journal of Selected Topics in Signal Processing},
  8\penalty0 (4):\penalty0 552--562, 2014.

\bibitem[Yang et~al.(2011)Yang, Chi, Zhu, Gong, and Jin]{yang2011detecting}
Tianbao Yang, Yun Chi, Shenghuo Zhu, Yihong Gong, and Rong Jin.
\newblock Detecting communities and their evolutions in dynamic social
  networks—a bayesian approach.
\newblock \emph{Machine Learning}, 82\penalty0 (2):\penalty0 157--189, 2011.

\bibitem[Zhang et~al.(2020)Zhang, Sun, and Li]{zhang2020TVN}
Jingfei Zhang, Will~Wei Sun, and Lexin Li.
\newblock Mixed-effect time-varying network model and application in brain
  connectivity analysis.
\newblock \emph{Journal of the American Statistical Association}, 115\penalty0
  (532):\penalty0 2022--2036, 2020.
\newblock ISSN 0162-1459.

\end{thebibliography}

\newpage

\appendix
\label{appendix}

%\subsection{Profile Log-likelihood w.r.t. Covariates Coefficients}
\section{MLE of $\hat{\beta}_0$}
\label{Appendix0}

In this section, we provide the detailed derivation of $\hat{\beta}_0(\bm{\beta}(t))$ as defined in~\eqref{eq:profileloglik}.
Subsequently, we substitute this resulting maximum likelihood estimate of $\hat{\beta}_0$ back into~\eqref{eq:profileloglik}, demonstrating how the equation presented in the first line of~\eqref{eq:theorem} is established.

%In this section, we provide the details of the calculation of \eqref{profilelog}, the profile log-likelihood with respect to $\beta(t)$. First, we take logarithm of~\eqref{TVlikelihood} to obtain its log-likelihood 

%The log-likelihood based on \eqref{TVlikelihood}  and focusing exclusively on the terms involving $\beta_0$ is given by

\begin{comment}
\begin{align}
 & l (\beta_0,\bm{\beta}(t),\pi,\phi,\rho;D,z) =
\scsum{\nu =1}{T} \scsum{1\leq i < j\leq n}{} \scsum{k,l=1}{K} \ind \times \\ \nonumber
& \biggl\{  \dfrac{1}{\phi} \left( \dfrac{y_{ij}(t_{\nu}) \exp[(1-\rho)\{\beta_0^{kl}+ \bm{x}_{ij}^\top \bm{\beta}(t_{\nu})\} ]}{1-\rho} - \dfrac{\exp[(2-\rho)\{\beta_0^{kl} + \bm{x}_{ij}^\top \bm{\beta}(t_{\nu})\} ]}{2-\rho}\right) \biggr\}. \label{loglikelihood}
\end{align}
\end{comment}

The derivative of~\eqref{eq:profileloglik} with respect to $\beta_0^{kl}$ is given by 
\begin{multline}
    \dfrac{\partial \ell_n(\beta_0,\bm{\beta}(t), \phi_0, \rho_0; D, z) }{\partial \beta_0^{kl}}=\frac{1}{\binom{n}{2}}
\scsum{\nu =1}{T} \scsum{1\leq i < j\leq n}{} \dfrac{\ind}{\phi_0} \times\\
 \Bigl\{ y_{ij}(t_{\nu}) \cdot \exp [(1-\rho_0)\{\beta_0^{kl} + \bm{x}_{ij}^\top \bm{\beta}(t_{\nu})\}]
     -\exp[(2-\rho_0)\{\beta_0^{kl} + \bm{x}_{ij}^\top \bm{\beta}(t_{\nu})\}]  \Bigl\}. \label{appendeq:derivativebeta0}
\end{multline}
Thus the second-order derivative is 
\begin{multline*}
  \dfrac{\partial^2 \ell_n(\beta_0,\bm{\beta}(t), \phi_0, \rho_0; D, z) }{\partial [\beta_0^{kl}]^2}= \\
  \frac{1}{\binom{n}{2}}
\scsum{\nu =1}{T} \scsum{1\leq i < j\leq n}{} \dfrac{\ind}{\phi_0} \times  
    \Bigl\{ (1-\rho_0)\cdot y_{ij}(t_{\nu}) \cdot \exp [(1-\rho_0)\{\beta_0^{kl} + \bm{x}_{ij}^\top \bm{\beta}(t_{\nu})\}]- \\(2-\rho_0) \cdot \exp[(2-\rho_0)\{\beta_0^{kl} + \bm{x}_{ij}^\top \bm{\beta}(t_{\nu})\}]  \Bigl\} <0.
\end{multline*}
Therefore, the MLE of ${\beta}_0^{kl}$ is given by the zero of~\eqref{appendeq:derivativebeta0} as 
\begin{align*}
 \hat{\beta}_0^{kl}(\bm{\beta}(t))&=\log \dfrac{\scsum{\nu =1}{T} \scsum{1\leq i <j \leq n}{} y_{ij}(t_{\nu}) \exp[(1-\rho)\bm{x}_{ij}^\top \bm{\beta}(t_{\nu}) ]\ind }{\scsum{\nu =1}{T} \scsum{1\leq i <j \leq n}{} \exp[(2-\rho)\bm{x}_{ij}^\top \bm{\beta}(t_{\nu}) ] \ind }\\
 &=\log \dfrac{ \hat{\theta}_{kl}}{\hat{\gamma}_{kl}}.
\end{align*}
Plugging $\hat{\beta}_0^{kl}(\bm{\beta}(t))=\log \hat{\theta}_{kl}/\hat{\gamma}_{kl}$ into~\eqref{eq:profileloglik}, we obtain the first line of the equation presented in~\eqref{eq:theorem}:
\begin{align*}
& \ell_n(\bm{\beta}(t), \phi_0, \rho_0; D, z) =\frac{1}{\binom{n}{2}} \scsum{\nu =1}{T} \scsum{1\leq i < j\leq n}{} \scsum{k,l=1}{K}\dfrac{\ind}{\phi_0} \times\\
&\Bigl[ \dfrac{y_{ij}(t_{\nu}) \exp[(1-\rho_0)\{\log \hat{\theta}_{kl}/\hat{\gamma}_{kl}+ \bm{x}_{ij}^\top \bm{\beta}(t_{\nu})\} ]}{1-\rho_0} -  \dfrac{\exp[(2-\rho_0)\{\log \hat{\theta}_{kl}/\hat{\gamma}_{kl} + \bm{x}_{ij}^\top \bm{\beta}(t_{\nu})\} ]}{2-\rho_0} \Bigr]\\
 =&\frac{1}{\binom{n}{2}} \scsum{k,l=1}{K} \frac{1}{1-\rho_0}(\frac{\hat{\theta}_{kl}}{\hat{\gamma}_{kl}})^{1-\rho_0} \scsum{\nu =1}{T} \scsum{1\leq i < j\leq n}{}  \dfrac{\ind}{\phi_0} \cdot
y_{ij}(t_{\nu}) \exp[(1-\rho_0)\{ \bm{x}_{ij}^\top \bm{\beta}(t_{\nu})\} ] - \nonumber \\ 
 & \frac{1}{\binom{n}{2}} \scsum{k,l=1}{K} \frac{1}{2-\rho_0}(\frac{\hat{\theta}_{kl}}{\hat{\gamma}_{kl}})^{2-\rho_0} \scsum{\nu =1}{T} \scsum{1\leq i < j\leq n}{} \dfrac{\ind}{\phi_0} \cdot \exp[(2-\rho_0)\{\bm{x}_{ij}^\top \bm{\beta}(t_{\nu})\} ] \\
  =& \dfrac{1}{\phi_0}\scsum{k,l=1}{K} \frac{1}{1-\rho_0}(\frac{\hat{\theta}_{kl}}{\hat{\gamma}_{kl}})^{1-\rho_0} \cdot \hat{\theta}_{kl} -  \dfrac{1}{\phi_0}\scsum{k,l=1}{K} \frac{1}{2-\rho_0}(\frac{\hat{\theta}_{kl}}{\hat{\gamma}_{kl}})^{2-\rho_0} \cdot \hat{\gamma}_{kl} \\
 %
%=&\dfrac{1}{\phi} \dfrac{1}{(1-\rho)(2-\rho)} \Bigl[ \scsum{\nu =1}{T} \scsum{1\leq i <j \leq n}{} z_{ik} z_{jl} \cdot y_{ij}(t_{\nu}) \exp\{(1-\rho)\bm{x}_{ij}^{\top} \bm{\beta}(t_{\nu}) \} \Bigr]^{2-\rho} \cdot \\
%& \Bigl[ \scsum{\nu =1}{T} \scsum{1\leq i <j \leq n}{} z_{ik} z_{jl}\cdot \exp\{(2-\rho)\bm{x}_{ij}^{\top} \bm{\beta}(t_{\nu}) \} \Bigr]^{\rho-1} \\
= &   \dfrac{1}{\phi_0} \dfrac{1}{(1-\rho_0)(2-\rho_0)} \scsum{k,l=1}{K}  \hat{\theta}_{kl}^{2-\rho_0}\cdot \hat{\gamma}_{kl}^{\rho_0-1}.
\end{align*}

\section{Proof of Theorem~\ref{theorem1_1}}
\label{Appendix1}
In this section, we prove Theorem~\ref{theorem1_1}. Before laying out the main proof, we introduce several lemmas first. 

\begin{Lemma}\label{coro1_1}
Under Conditions \ref{cond1} to \ref{cond2}, 
\begin{align*}
\frac{\hat{\gamma}_{kl}}{\scsum{k,l}{} \hat{\gamma}_{kl}}  = p_k p_l + o_p(1).
\end{align*}
\end{Lemma}

\begin{proof}{Proof}{}
According to Conditions~\ref{cond1} and~\ref{cond2}, $\exp [(2-\rho) x_{ij}^\top \beta(t)]$ and $\ind$ are iid random variables, with mean $\gamma$ and $p_k p_l$ respectively. Specifically, $\gamma$ is a positive constant. 
By the weak law of large numbers, we have 
\begin{align*}
\frac{\hat{\gamma}_{kl}}{\scsum{k,l}{} \hat{\gamma}_{kl}} & =  \frac{2\scsum{\nu =1}{T} \scsum{1\leq i < j \leq n}{} \exp [(2-\rho) \bm{x}_{ij}^\top \bm{\beta}(t_{\nu})]\ind /\{n(n-1)\}}{2\scsum{\nu =1}{T} \scsum{1\leq i < j \leq n}{} \exp [(2-\rho) \bm{x}_{ij}^\top \bm{\beta}(t_{\nu})] /\{n(n-1)\}} \\
    &= \frac{\gamma \cdot p_k p_l +o_p(1)}{\gamma+o_p(1)}\\
     &= p_k p_l + o_p(1).
\end{align*}
\end{proof}

\begin{Lemma}\label{coro2_1}
Under Conditions \ref{cond1} to \ref{cond2}, 
\begin{align*}
 \frac{\hat{\theta}_{kl}}{\scsum{k,l}{} \hat{\theta}_{kl}}  = p_k p_l + o_p(1).
\end{align*}
\end{Lemma}

\begin{proof}{Proof}{}
The proof is similar to that of Lemma \ref{coro1_1}. If we can show that, at each time point $t_{\nu}$, $\nu = 1, \ldots, T$,  $y_{ij}(t_{\nu}) \exp [(1-\rho)\bm{x}_{ij}^\top \bm{\beta}(t_{\nu})]$ for $i, j = 1, \ldots, n$ are iid with a nonzero mean, we complete the proof. 
For each node pair $(i, j)$, both their pairwise covariate $\bm{x}_{ij}$ and community labels $c_i$ and $c_j$ are iid. Moreover, $y_{ij}(t_{\nu})$ conditional on $x_{ij}$, $c_i$ and $c_j$ are iid as well. Therefore, $y_{ij}(t_{\nu}) \exp [(1-\rho)\bm{x}_{ij}^\top \bm{\beta}(t_{\nu})]$ for $i, j = 1, \ldots, n$ are iid, with mean
\begin{align*}
  \E[y_{ij}(t_{\nu}) \exp \{(1-\rho)\bm{x}_{ij}^\top \bm{\beta}(t_{\nu})\}] &=   \E \Bigl(  \E [y_{ij}(t_{\nu}) \exp \{(1-\rho)\bm{x}_{ij}^\top \bm{\beta}(t_{\nu})\} \mid \bm{x},\bm{c} ]  \Bigr)  \\
  &=\E \Bigl[  \E \{y_{ij}(t_{\nu})  | \bm{x},\bm{c} \} \cdot \exp \{(1-\rho)\bm{x}_{ij}^\top \bm{\beta}(t_{\nu}) 
 \}\Bigr] \\
  %  &=\E \Bigl[  [\mu_{ij}(t)  | \bm{x},\bm{c} ] \cdot \exp [(1-\rho)x_{ij}^\top \beta(t)]\Bigr] \\
 &=\E \Bigl[  \exp\{\beta_0^{c_i c_j}+\bm{x}_{ij}^\top \bm{\beta}(t_{\nu})\} \cdot \exp \{(1-\rho)\bm{x}_{ij}^\top \bm{\beta}(t_{\nu})\}\Bigr] \\
  &= \scsum{k,l}{} \E \Bigl[  \exp \left\{ \beta_0^{kl}+(2-\rho) \bm{x}_{ij}^\top \bm{\beta}(t_{\nu}) \right\}\Bigr]\cdot p_k p_l.
\end{align*}
Therefore, the expectation is a nonzero constant.
\end{proof}

Next, we prove Theorem~\ref{theorem1_1}. 
\begin{proof}{Proof of Theorem \ref{theorem1_1}}{}
By Lemmas~\ref{coro1_1} and \ref{coro2_1} and the continuous mapping theorem, 
\begin{align}
&  \scsum{k,l}{}  \hat{\theta}_{kl}^{2-\rho}\cdot \hat{\gamma}_{kl}^{\rho-1} \nonumber \\
=&   \Bigl[ \scsum{k,l}{} \Bigl(\dfrac{\hat{\theta}_{kl}}{\scsum{k,l}{}\hat{\theta}_{kl}} \Bigr)^{2-\rho}\cdot \Bigl(\dfrac{\hat{\gamma}_{kl}}{\scsum{k,l}{}\hat{\gamma}_{kl}} \Bigr)^{\rho-1} \Bigr] \cdot \bigl(\scsum{k,l}{}\hat{\theta}_{kl} \bigr)^{2-\rho} \bigl(\scsum{k,l}{}\hat{\gamma}_{kl} \bigr)^{\rho-1} \nonumber \\
=& \Bigl[ \scsum{k,l}{} \Bigl\{(p_k p_l)^{2-\rho} + o_p(1) \Bigr\}\cdot \Bigl\{ (p_k p_l)^{\rho-1} + o_p(1) \Bigr\} \Bigr] \cdot \bigl(\scsum{k,l}{}\hat{\theta}_{kl} \bigr)^{2-\rho} \bigl(\scsum{k,l}{}\hat{\gamma}_{kl} \bigr)^{\rho-1}\nonumber\\
=&   \Bigl[ \scsum{k,l}{} \Bigl(p_k p_l + o_p(1) \Bigr) \Bigr] \cdot \bigl(\scsum{k,l}{}\hat{\theta}_{kl} \bigr)^{2-\rho} \bigl(\scsum{k,l}{}\hat{\gamma}_{kl} \bigr)^{\rho-1}  \nonumber \\
=& \Bigl[ 1 + \frac{K(K+1)}{2}o_p(1)  \Bigr] \cdot \Bigl( \theta + o_p (1) \Bigr)^{2-\rho} \Bigl( \gamma + o_p (1)  \Bigr)^{\rho-1}  
%\nonumber \\=& \Bigl[ 1 + \frac{K(K+1)}{2}o_p(1)  \Bigr] \cdot \Bigl( \theta^{2-\rho} + o_p (1) \Bigr) \Bigl( \gamma^{\rho-1} + o_p (1) \Bigr) 
\label{proof1_wlln_1}  \\
=& \theta^{2-\rho} \gamma^{\rho-1}  + o_p(1). \nonumber 
\end{align}
\eqref{proof1_wlln_1} holds because $\sum_{k,l}{}\hat{\theta}_{kl}=\hat{\theta} = \theta+o_p(1)$ and $\sum_{k,l}{}\hat{\gamma}_{kl} = \hat{\gamma}= \gamma+o_p(1)$ by the weak law of large numbers. Therefore, we have
\begin{align*}
\dfrac{2}{n(n-1)} l_z (\bm{\beta}(t))& =\dfrac{1}{\phi} \dfrac{1}{(1-\rho)(2-\rho)} \scsum{k,l}{}  \hat{\theta}_{kl}^{2-\rho}\cdot \hat{\gamma}_{kl}^{\rho-1}\\
  &= \dfrac{1}{\phi} \dfrac{1}{(1-\rho)(2-\rho)} \Bigl( \theta^{2-\rho}\cdot \gamma^{\rho-1}  +o_p(1) \Bigr)\\
  &=\dfrac{1}{\phi} \dfrac{1}{(1-\rho)(2-\rho)} \theta^{2-\rho}\cdot \gamma^{\rho-1}  +o_p(1).
\end{align*}
\end{proof}

\begin{Remark} By H\"older's inequality, we have
%in~\eqref{proof1_1} 
\begin{align*}
    \scsum{k,l}{} \left(\dfrac{\hat{\theta}_{kl}}{\scsum{k,l}{}\hat{\theta}_{kl}} \right)^{2-\rho}\cdot \left(\dfrac{\hat{\gamma}_{kl}}{\scsum{k,l}{}\hat{\gamma}_{kl}} \right)^{\rho-1} \leq \left( \scsum{k,l}{} \dfrac{\hat{\theta}_{kl}}{\scsum{k,l}{}\hat{\theta}_{kl}} \right)^{2-\rho}\cdot \left( \scsum{k,l}{} \dfrac{\hat{\gamma}_{kl}}{\scsum{k,l}{}\hat{\gamma}_{kl}} \right)^{\rho-1}=1.
\end{align*}
Then, it follows
\begin{align}
\dfrac{2}{n(n-1)}  l_z (\bm{\beta}(t)) \leq \dfrac{1}{\phi} \dfrac{1}{(1-\rho)(2-\rho)} \theta^{2-\rho}\cdot \gamma^{\rho-1}. \label{inequality}
\end{align}
In fact, Lemmas~\ref{coro1_1} and~\ref{coro2_1} establish the asymptotic equality conditions, which sharpen~\eqref{inequality} and lead to the conclusion in Theorem~\ref{theorem1_1}. 
\end{Remark}

\section{Bspline estimation in Step 1}
\label{Appendix2}

In this section, we present the details of estimating the time-varying covariate coefficient $ \hat{\bm{\beta}}(t) $ in accordance to~\eqref{eq:hat_betat} as part of Step 1 in our two-step estimation process.

According to \cite{silverman1985some} and \cite{green1993nonparametric}, each optimal $\bm{\beta}_u (t), u=1,\cdots,p$ is a natural cubic spline with knots at time points where temporal data is observed. 
In practice, 
we use B-spline in the computations of smoothing splines \citep{hastie2009elements}. We use $T+4$ B-spline basis functions $\{ B_m (t) \}_{m=1}^{T+4}$, so we can represent the scalar $\bm{\beta}_u (t_\nu)$ as the $(\nu,u)-$th element in the $T-$by$-p$ matrix $\bm{B} \bm{\eta}$, where 
% \begin{align*}
% \bm{B}_{T\times (T+4) }=
% \begin{bmatrix}
% B_1 (t_1) & \cdots & B_{T+4} (t_1)  \\
% \vdots & \ddots & \vdots \\
% B_1 (t_T) & \cdots & B_{T+4} (t_T) \\
% \end{bmatrix}.
% \end{align*}
\begin{equation*}
\bm{B}_{T\times (T+4) }= 
    	\left[ \begin{aligned}
B_1 (t_1) & \quad \cdots \quad B_{T+4} (t_1)  \\
\vdots & \quad  \ddots \quad \quad \vdots \\
B_1 (t_T) & \quad \cdots \quad B_{T+4} (t_T) \\
\end{aligned} \right]. 
\end{equation*}
and $\bm{\eta}\in \mathbb{R}^{(T+4)\times p}$ is the coefficient matrix that needs to be estimated. 
%\st{The $p$ dimensional vector $\bar{\bm{\beta}} (t)=(\bm{B}_{t\cdot } \bm{\eta})^\top$, where $\bm{B}_{t\cdot } $ represents the $t^{th}$ row of the matrix $\bm{B}$.} 
The $p$ dimensional vector $\bm{\beta} (t_\nu)=(\bm{B}_{\nu \cdot } \bm{\eta})^\top$, where $\bm{B}_{\nu \cdot } $ represents the $\nu^{th}$ row of the matrix $\bm{B}$.

If we define $\bm{\Omega} \in \mathbb{R}^{(T+4)\times(T+4)}$ where $\bm{\Omega}_{ij}=\int  B_i'' (t) B_j'' (t)  dt$ and $\bm{\lambda} = (\lambda_1,\cdots,\lambda_p)^\top$, we can solve the $(T+4) \times p$ matrix $\bm{\eta}$ by plugging $\bm{\beta} (t_\nu)=(\bm{B}_{\nu \cdot } \bm{\eta})^\top$ in~\eqref{eq:hat_betat}:

\begin{comment}
\begin{align}
\bm{\eta}  = \argmax_{\bm{\eta}} \dfrac{1}{\phi} \dfrac{1}{(1-\rho)(2-\rho)} \scsum{k,l}{}  \hat{\theta}_{kl}^{2-\rho}\cdot \hat{\gamma}_{kl}^{\rho-1} - \frac{1}{2} \bm{\lambda}^\top \cdot diag(\bm{\eta}^T \bm{\Omega} \bm{\eta}),
\end{align}
where the $\hat{\theta}_{kl}$ and $\hat{\gamma}_{kl}$ are adjusted as
\begin{align*}
 \hat{\theta}_{kl}  &= \scsum{\nu=1}{T} \scsum{1\leq i < j \leq n}{}y_{ij}(t) \exp [(1-\rho) \bm{B}_{\nu\cdot} \bm{\eta} x_{ij}]\ind,\\
  \hat{\gamma}_{kl} &=\scsum{\nu=1}{T} \scsum{1\leq i < j \leq n}{} \exp [(2-\rho) \bm{B}_{\nu \cdot} \bm{\eta} x_{ij}]\ind.
\end{align*}
\JJ{Or}
\end{comment}

\begin{multline*}
   \hat{\bm{\eta}}  = \argmax_{\bm{\eta}}  \\
    \dfrac{1}{(1-\rho_0)(2-\rho_0)} \scsum{k,l}{}  \left(\scsum{\nu=1}{T} \scsum{1\leq i < j \leq n}{}y_{ij}(t) \exp [(1-\rho_0) \bm{B}_{\nu \cdot} \bm{\eta} x_{ij}]\ind \right)^{2-\rho_0}\times \\
  \left( \scsum{\nu=1}{T} \scsum{1\leq i < j \leq n}{} \exp [(2-\rho_0) \bm{B}_{\nu \cdot} \bm{\eta} x_{ij}]\ind \right)^{\rho_0-1} - \frac{1}{2} \bm{\lambda}^\top \cdot diag(\bm{\eta}^T \bm{\Omega} \bm{\eta})  
\end{multline*}

\begin{comment}
\begin{align*}
& \hat{\bm{\eta}}  = \argmax_{\bm{\eta}} \dfrac{1}{(1-\rho_0)(2-\rho_0)} \scsum{k,l}{}  \left(\scsum{\nu=1}{T} \scsum{1\leq i < j \leq n}{}y_{ij}(t) \exp [(1-\rho_0) \bm{B}_{\nu \cdot} \bm{\eta} x_{ij}]\ind \right)^{2-\rho_0}\cdot \\
& \left( \scsum{\nu=1}{T} \scsum{1\leq i < j \leq n}{} \exp [(2-\rho_0) \bm{B}_{\nu \cdot} \bm{\eta} x_{ij}]\ind \right)^{\rho_0-1} - \frac{1}{2} \bm{\lambda}^\top \cdot diag(\bm{\eta}^T \bm{\Omega} \bm{\eta}),
\end{align*}
\end{comment}

Once we have obtain $\hat{\bm{\eta}}$, we can calculate the estimated $\hat{\bm{\beta}}(t)$ in Step 1 by $\hat{\bm{\beta}}(t) = \bm{B} \hat{\bm{\eta}}$.

%In the vanilla Tweedie without the covariate coefficient, we skip this step. For static Tweedie with covariate, we estimate the static covariate coefficient $\bm{\beta}$ by maximize the log-likelihood of~\eqref{TVlikelihood}, which can be implemented by R function \texttt{optim} easily as the optimization problem is convex.

\section{Additional Simulation Results}
\label{Appendix3}

\subsection{Tweedie Parameters Estimated in Simulation}
\label{Appendix3:Tweedie}

Although our primary interest is to estimate the covariate coefficients and infer the community labels in our model, their estimation is affected by the Tweedie parameters $\phi$ and $\rho$. In this section, we provide the simulation results regarding $\phi$ and $\rho$ in the Section~\ref{sec:simulation}. We report the estimated bias and standard error (SE) of the estimates of $\phi$ and $\rho$ over 50 simulation runs in Table~\ref{Appendix_tab2}, ~\ref{Appendix_tab1}, ~\ref{simulation_covariate_phirho}, ~\ref{sbm-tab:sim-TVTweedie-phi} and ~\ref{sbm-tab:sim-TVTweedie-rho}. To be more specific, we calculate the bias of the estimate $\hat{\phi}$ of $\phi$ with true value $\phi_0$ by $\text{bias}(\hat{\phi})=\sum_{i=1}^{50}(\hat{\phi}_i-\phi_0)/50$ and $\text{SE}(\hat{\phi})=\sqrt{\sum_{i=1}^{50}(\hat{\phi}_i-\bar{\hat{\phi}})^2/49}$ where $\bar{\hat{\phi}}$ is the average of $\hat{\phi}$ over 50 simulation runs. In summary, the simulation results indicate that the estimates of $\phi$ and $\rho$ are highly accurate.

\begin{table}[ht] %htp
\centering
\begin{tabular}{|c|c|c|cc|cc|cc|}
\hline
 & & & \multicolumn{2}{c|}{Scenario 1} & \multicolumn{2}{c|}{Scenario 2} & \multicolumn{2}{c|}{Scenario 3} \\
 $\phi$ & $\rho$ & $n$ & Bias  &  SE &   Bias    & SE  &   Bias    & SE  \\ \hline 
 \hline
    \multirow{6}{*}{$2$} &\multirow{2}{*}{$1.2$} & 50 & 0.012&0.054 &0.014&0.09&0.014&0.082 \\ \cline{3-9}
&& 100 & -0.003&0.027&0.001&0.026&0.004&0.032  \\ \cline{2-9}
&\multirow{2}{*}{$1.5$} & 50 &0.014&0.072&0.018&0.085&0.011&0.085  \\ \cline{3-9}
&& 100 & 0.004&0.025 &0.004&0.034 &0&0.032  \\ \cline{2-9}
&\multirow{2}{*}{$1.8$} & 50 &-0.003&0.062&-0.004&0.07&-0.001&0.056  \\  \cline{3-9}
&& 100 &0&0.03&0.001&0.028&-0.002&0.028  \\
\hline
 \hline
    \multirow{6}{*}{$1$} &\multirow{2}{*}{$1.2$} & 50 & 0.008&0.034&0.008&0.028&0.013&0.048   \\ \cline{3-9}
&& 100 & 0.002&0.014&-0.002&0.015&-0.001&0.012 \\ \cline{2-9}
&\multirow{2}{*}{$1.5$} & 50 &0.008&0.039&0.006&0.033&0.01&0.033 \\ \cline{3-9}
&& 100 & -0.001&0.018&0&0.017&0.002&0.014\\ \cline{2-9}
&\multirow{2}{*}{$1.8$} & 50 & 0.001&0.037&0.008&0.034&0.007&0.034 \\  \cline{3-9}
&& 100 & 0.001&0.016 &0.001&0.016&-0.001&0.016 \\
\hline
 \hline
\multirow{6}{*}{$0.5$} &\multirow{2}{*}{$1.2$} & 50 & 0.004&0.015&0&0.022&0.007&0.023 \\ \cline{3-9}
&& 100 & 0.002&0.007 &0&0.008&0&0.006 \\ \cline{2-9}
&\multirow{2}{*}{$1.5$} & 50 &  0.005&0.021&0.002&0.022&0.009&0.03 \\ \cline{3-9}
&& 100 & 0&0.01 &0&0.01&-0.001&0.01 \\ \cline{2-9}
&\multirow{2}{*}{$1.8$} & 50 &-0.003&0.016&0.008&0.023&0.002&0.031 \\  \cline{3-9}
&& 100 & -0.002&0.009 &-0.001&0.009&-0.001&0.010 \\
\hline
\end{tabular}
\caption{Summary of estimated bias and standard error (SE) of estimated $\phi$ in scenario 1, 2 and 3 over 50 simulation runs.}
\label{Appendix_tab2}
\end{table}

\begin{table}[ht] %htp
\centering
\begin{tabular}{|c|c|c|cc|cc|cc|}
\hline
 & & & \multicolumn{2}{c|}{Scenario 1} & \multicolumn{2}{c|}{Scenario 2} & \multicolumn{2}{c|}{Scenario 3} \\
 $\phi$ & $\rho$ & $n$ & Bias  &  SE &   Bias    & SE  &   Bias    & SE  \\ \hline 
 \hline
    \multirow{6}{*}{$2$} &\multirow{2}{*}{$1.2$} & 50 & 0&0&0.002&0.014&0.002&0.014 \\ \cline{3-9}
&& 100 & 0&0&0&0&0&0  \\ \cline{2-9}
&\multirow{2}{*}{$1.5$} & 50 &0&0&0&0&0&0  \\ \cline{3-9}
&& 100 & 0&0 &0&0&0&0  \\ \cline{2-9}
&\multirow{2}{*}{$1.8$} & 50 & 0&0&0&0&0&0  \\  \cline{3-9}
&& 100 &0&0&0&0&0&0  \\
\hline
 \hline
    \multirow{6}{*}{$1$} &\multirow{2}{*}{$1.2$} & 50 & 0.060&0.120&0&0&0.002&0.014   \\ \cline{3-9}
&& 100 &0&0&0&0&0&0 \\ \cline{2-9}
&\multirow{2}{*}{$1.5$} & 50 & -0.002&0.014&0&0&0&0 \\ \cline{3-9}
&& 100 & 0&0&0&0&0&0 \\ \cline{2-9}
&\multirow{2}{*}{$1.8$} & 50 & -0.002&0.014&0&0&0&0 \\  \cline{3-9}
&& 100 &0&0&0&0&0&0 \\
\hline
 \hline
\multirow{6}{*}{$0.5$} &\multirow{2}{*}{$1.2$} & 50 & 0&0&-0.002&0.014&0.002&0.014 \\ \cline{3-9}
&& 100 & 0&0&0&0&0&0 \\ \cline{2-9}
&\multirow{2}{*}{$1.5$} & 50 &  0.002&0.037&-0.002&0.014&0.002&0.032 \\ \cline{3-9}
&& 100 & 0&0 &0&0&0&0 \\ \cline{2-9}
&\multirow{2}{*}{$1.8$} & 50 & 0.006&0.054 &0.008&0.039&0.002&0.037 \\  \cline{3-9}
&& 100 &  0.002&0.014&0&0&0&0 \\
\hline
\end{tabular}
\caption{Summary of estimated bias and standard error (SE) of estimated $\rho$ in scenario 1, 2 and 3 over 50 simulation runs.}
\label{Appendix_tab1}
\end{table}

\begin{table}[ht] %htp
\centering
\begin{tabular}{|c|c|c|cc|cc|cc|cc|}
\hline
\multicolumn{3}{|c|}{} & \multicolumn{4}{c|}{Weak Effect ($\beta=1$)} & \multicolumn{4}{c|}{Strong Effect ($\beta=2$)} \\ \hline
 $\phi$ & $\rho$ & $n$ & \multicolumn{2}{c|}{$\hat{\phi}$} &\multicolumn{2}{c|}{$\hat{\rho}$}&\multicolumn{2}{c|}{$\hat{\phi}$}& \multicolumn{2}{c|}{$\hat{\rho}$} \\ \cline{4-11}
 & & & Bias   &  SE & Bias & SE &  Bias   &  SE & Bias & SE \\ \hline
\multirow{6}{*}{$2$} &\multirow{2}{*}{$1.2$} & 50 & -0.002&0.063 & 0&0 & 0.002&0.073   & 0.002&0.014 \\ \cline{3-11}
&& 100 & 0.001&0.026  &0&0 & -0.006&0.031 & 0&0  \\ \cline{2-11}
&\multirow{2}{*}{$1.5$} & 50 & 0.016&0.076  &0&0&   0.019&0.072 & 0&0 \\ \cline{3-11}
&& 100 &0.005&0.033 &  0&0  &0.005&0.039 &  0&0 \\ \cline{2-11}
&\multirow{2}{*}{$1.8$} & 50 & 0.003&0.065  & 0&0 &  -0.007&0.066 & 0&0  \\  \cline{3-11}
&& 100 & -0.002&0.032   &  0&0& 0.004&0.035 & 0&0  \\
\hline \hline
\multirow{6}{*}{$1$} &\multirow{2}{*}{$1.2$} & 50 & 0.003&0.03  &0&0 &  0.002&0.032  & 0&0 \\ \cline{3-11}
&& 100 & 0&0.017  &  0&0 & -0.002&0.013 & 0&0  \\ \cline{2-11}
&\multirow{2}{*}{$1.5$} & 50 & 0.003&0.042 & -0.002&0.014  &  0.009&0.052  &  0&0.02  \\ \cline{3-11}
&& 100 & -0.001&0.013 & 0&0  & 0.001&0.016 &  0&0  \\ \cline{2-11}
&\multirow{2}{*}{$1.8$} & 50 &0.001&0.037 &  0&0 & 0.006&0.04 &  0&0  \\  \cline{3-11}
&& 100 & -0.001&0.027  & -0.002&0.014 & 0.003&0.02 & 0&0  \\
\hline \hline
\multirow{6}{*}{$0.5$} &\multirow{2}{*}{$1.2$} & 50 & 0.009&0.026  & 0.004&0.02  &  0&0.014  & 0&0 \\ \cline{3-11}
&& 100 & 0.001&0.008 &  0&0 &0.001&0.007 & 0&0  \\ \cline{2-11}
&\multirow{2}{*}{$1.5$} & 50 & 0.003&0.027 &  -0.006&0.031 & 0.002&0.019  & -0.006&0.024  \\ \cline{3-11}
&& 100 &0&0.009 &  0&0 & 0&0.009 & 0&0  \\ \cline{2-11}
&\multirow{2}{*}{$1.8$} & 50 & 0.005&0.028 & -0.01&0.054 & -0.002&0.025 & -0.006&0.042 \\  \cline{3-11}
&& 100 &-0.002&0.01  & 0&0 & 0.001&0.009 & 0&0  \\
\hline 
\end{tabular}
\caption{Summary of estimated bias and standard error (SE) of estimated $\phi$ and $\rho$ in the static model with covariates over 50 simulation runs, with $(\beta_0^{kk},\beta_0^{kl})=(0.5,-0.5)$. }
\label{simulation_covariate_phirho}
\end{table}

\begin{table}[hp] %htp
\centering
\begin{tabular}{c|c|cccccc}
\hline
\multirow{2}{*}{$(\beta_0^{kk},\beta_0^{kl})$}&\multirow{2}{*}{$\phi$}&\multicolumn{6}{c}{$\beta(t)$} \\ 
 & &    $2t-1$     &   $\sin (2\pi t)$     &    $2t$   &    $\sin (2\pi t)$ +1       &   $0.5(2t-1)$      &   $ 0.5 \sin (2\pi t)$  \\ \hline \hline
Scenario 1  & Bias & 0.002 & 0.004  &  -0.001 & 0.007  & 0.003  & 0.004 \\ \cline{2-8}
$(1,0)$ &   SE & 0.008 &  0.008 & 0.014  &  0.017 & 0.01  & 0.017 \\  \hline \hline
Scenario 2  & Bias & 0.002 & 0.005  &  0.004 &  0.005 & 0.002  & 0.001 \\ \cline{2-8}
$(0.5,-0.5)$ &   SE & 0.008 & 0.006  & 0.026  & 0.008  & 0.008  & 0.007 \\  \hline \hline
Scenario 3  & Bias & 0.002 & 0.005  &  0 & 0.005  & 0.001  & 0.002 \\ \cline{2-8}
$(0,-1)$ &   SE & 0.007 &  0.007 &  0.007 & 0.009  & 0.007  & 0.009 \\  \hline \hline
Scenario 4  & Bias & 0.001 & 0.003  & 0.002  & 0.004  & 0.001  & 0.001 \\ \cline{2-8}
$(0.5,0)$ &   SE & 0.007 & 0.007  & 0.007  & 0.006  & 0.007  &  0.007\\  \hline \hline
Scenario 5  & Bias & 0.001 & 0.004  &  0.002 & 0.005  & 0.002  & 0.001 \\ \cline{2-8}
$(0.25,-0.25)$ &   SE & 0.007 & 0.008  & 0.007  & 0.008  & 0.008  & 0.007 \\  \hline \hline
Scenario 6  & Bias & 0 &  0.004 & 0.001  & 0.006  & 0  & 0.006 \\ \cline{2-8}
$(0,-0.5)$ &   SE & 0.008 & 0.007  & 0.009  & 0.007  &  0.007 & 0.032 \\  \hline \hline
\end{tabular}
\caption{Summary of estimated bias and standard error (SE) of estimated $\phi$ (with true value $1$) in the time-varying model over 50 simulation runs (using $\lambda=0.5$).}
\label{sbm-tab:sim-TVTweedie-phi}
\end{table}

%Scenario 2  $(0.5,-0.5)$
%Scenario 3  $(0,-1)$
%Scenario 4  $(0.5,0)$
%Scenario 5  $(0.25,-0.25)$
%Scenario 6 $(0,-0.5)$

\begin{table}[hp] %htp
\centering
\begin{tabular}{c|c|cccccc}
\hline
\multirow{2}{*}{$(\beta_0^{kk},\beta_0^{kl})$}&\multirow{2}{*}{$\rho$}&\multicolumn{6}{c}{$\beta(t)$} \\ 
 & &    $2t-1$     &   $\sin (2\pi t)$     &    $2t$   &    $\sin (2\pi t)$ +1       &   $0.5(2t-1)$      &   $ 0.5 \sin (2\pi t)$  \\ \hline \hline
Scenario 1  & Bias & 0 &  0 &  -0.002 & 0.002  & 0  & 0.002 \\ \cline{2-8}
$(1,0)$ &   SE & 0 & 0  & 0.014  &  0.014 & 0  & 0.014 \\  \hline \hline
Scenario 2  & Bias & 0 &  0 &  0.002 &  0 & 0  &  0\\ \cline{2-8}
$(0.5,-0.5)$ &   SE & 0 &  0 & 0.014  &  0 &  0 & 0 \\  \hline \hline
Scenario 3  & Bias & 0 &  0 &  0 &  0 &  0 & 0 \\ \cline{2-8}
$(0,-1)$ &   SE & 0 & 0  &  0 &  0 & 0  & 0 \\  \hline \hline
Scenario 4  & Bias & 0 &  0 &  0 &  0 &  0 & 0 \\ \cline{2-8}
$(0.5,0)$ &   SE & 0 &  0 &  0 &  0 &  0 & 0 \\  \hline \hline
Scenario 5  & Bias & 0 &  0 &  0 & 0  & 0  & 0 \\ \cline{2-8}
$(0.25,-0.25)$ &   SE & 0 & 0  &  0 & 0  & 0  & 0 \\  \hline \hline
Scenario 6  & Bias & 0 & 0  &  0 & 0  &  0 & 0.002 \\ \cline{2-8}
$(0,-0.5)$ &   SE & 0 &  0 &  0 &  0 &  0 & 0.014 \\  \hline \hline
\end{tabular}
\caption{Summary of estimated bias and standard error (SE) of estimated $\rho$ (with true value $1.5$) in the time-varying model over 50 simulation runs (using $\lambda=0.5$).}
\label{sbm-tab:sim-TVTweedie-rho}
\end{table}

\subsection{Sensitivity Analysis of Tuning Parameters in TV-TSBM}
\label{Appendix3:SensitivityTuning}

In this section, we apply the TV-TSBM on two $\lambda$ values of 1 and 0.1 respectively to conduct the sensitivity analysis of the simulation in Section~\ref{subsec:simulationTVTWsbm}.

%%MZ moved details below to here in the appendix, but not really happy with the text yet 
By and large, the clustering outcomes across the three distinct tuning parameters measured by the NMI are relatively close, and all indicate high-quality clustering. With increasing $\lambda$ values, the curvature of the estimated $\hat{\beta}(t)$ diminishes. Consequently, when the true curve $\beta(t)$ is linear, larger values of $\lambda$ yield smaller errors in estimating $\hat{\beta}(t)$. Vice versa, a smaller $\lambda$ leads to a better estimation of $\hat{\beta}(t)$ when the underlying curve is a sine function. In summary, the consistent clustering outcomes across various distinct $\lambda$ values, coupled with the choice of a moderately penalized smoothness, substantiates the rationale behind adopting $0.5$ as the preferred value for $\lambda$.

\begin{table}[ht] %htp
\centering
\begin{tabular}{c|c|cccccc}
\hline
 \multirow{2}{*}{$(\beta_0^{kk},\beta_0^{kl})$} &  &\multicolumn{6}{c}{$\beta(t)$} \\ \cline{3-8} 
  &          &    $2t-1$     &   $\sin (2\pi t)$     &    $2t$   &    $\sin (2\pi t)$ +1       &   $0.5(2t-1)$      &   $ 0.5 \sin (2\pi t)$  \\ \hline \hline
Scenario 1 &NMI & \shortstack{\\1\\(0)} & \shortstack{\\1\\(0)} & \shortstack{\\1\\(0)} & \shortstack{\\1\\(0)} & \shortstack{\\0.996\\(0.004)} & \shortstack{\\1\\(0)} \\ \cline{2-8}
$(1,0)$&$\text{Err}(\hat{\beta}(t))$ & \shortstack{\\0.004\\(0)} & \shortstack{\\0.041\\(0)} & \shortstack{\\0.004\\(0)} & \shortstack{\\0.040\\(0)} & \shortstack{\\0.004\\(0)} & \shortstack{\\0.021\\(0)} \\ \hline \hline
Scenario 2 &NMI  & \shortstack{\\1\\(0)} & \shortstack{\\1\\(0)} & \shortstack{\\1\\(0)} & \shortstack{\\1\\(0)} & \shortstack{\\1\\(0)} & \shortstack{\\1\\(0)} \\ \cline{2-8}
$(0.5,-0.5)$ &$\text{Err}(\hat{\beta}(t))$  & \shortstack{\\0.004\\(0)} & \shortstack{\\0.048\\(0)} & \shortstack{\\0.005\\(0)} & \shortstack{\\0.046\\(0)} & \shortstack{\\0.005\\(0)} & \shortstack{\\0.024\\(0)} \\ \hline \hline
Scenario 3 &NMI & \shortstack{\\1\\(0)} & \shortstack{\\1\\(0)} & \shortstack{\\1\\(0)} & \shortstack{\\1\\(0)} & \shortstack{\\1\\(0)} & \shortstack{\\1\\(0)} \\ \cline{2-8}
$(0,-1)$ &$\text{Err}(\hat{\beta}(t))$  & \shortstack{\\0.005\\(0)} & \shortstack{\\0.055\\(0)} & \shortstack{\\0.005\\(0)} & \shortstack{\\0.053\\(0)} & \shortstack{\\0.006\\(0)} & \shortstack{\\0.028\\(0)} \\ \hline \hline
Scenario 4 &NMI  & \shortstack{\\0.996\\(0.004)} & \shortstack{\\1\\(0)}&\shortstack{\\1\\(0)}&\shortstack{\\1\\(0)}&\shortstack{\\1\\(0)}&\shortstack{\\1\\(0)}\\ \cline{2-8}
$(0.5,0)$ &$\text{Err}(\hat{\beta}(t))$  & \shortstack{\\0.004\\(0)} &\shortstack{\\0.045\\(0)}&\shortstack{\\0.004\\(0)}&\shortstack{\\0.043\\(0)}&\shortstack{\\0.004\\(0)}&\shortstack{\\0.023\\(0)}  \\ \hline \hline
Scenario 5 &NMI  & \shortstack{\\1\\(0)} & \shortstack{\\1\\(0)} & \shortstack{\\1\\(0)} & \shortstack{\\1\\(0)} & \shortstack{\\1\\(0)} & \shortstack{\\1\\(0)} \\ \cline{2-8}
$(0.25,-0.25)$ &$\text{Err}(\hat{\beta}(t))$  & \shortstack{\\0.004\\(0)} & \shortstack{\\0.048\\(0)} & \shortstack{\\0.004\\(0)} & \shortstack{\\0.046\\(0)} & \shortstack{\\0.004\\(0)} & \shortstack{\\0.024\\(0)} \\ \hline \hline
Scenario 6 &NMI  & \shortstack{\\1\\(0)} & \shortstack{\\0.996\\(0.004)} & \shortstack{\\1\\(0)} & \shortstack{\\1\\(0)} & \shortstack{\\1\\(0)} & \shortstack{\\0.996\\(0.004)} \\ \cline{2-8}
$(0,-0.5)$ &$\text{Err}(\hat{\beta}(t))$  & \shortstack{\\0.005\\(0)} & \shortstack{\\0.052\\(0)} & \shortstack{\\0.004\\(0)} & \shortstack{\\0.050\\(0)} & \shortstack{\\0.005\\(0)} & \shortstack{\\0.026\\(0)} \\ \hline 
\end{tabular}
\caption{Summary of clustering and estimation performance (using $\lambda=1$) from the time-varying model over 50 simulation runs, with $\phi=1$, $\rho=1.5$ and $n=50$.}
\label{sbm-tab:TVTweedie-sensitivity_1}
\end{table}

\begin{figure}[ht]
    \centering
\includegraphics[height=0.32\paperheight]{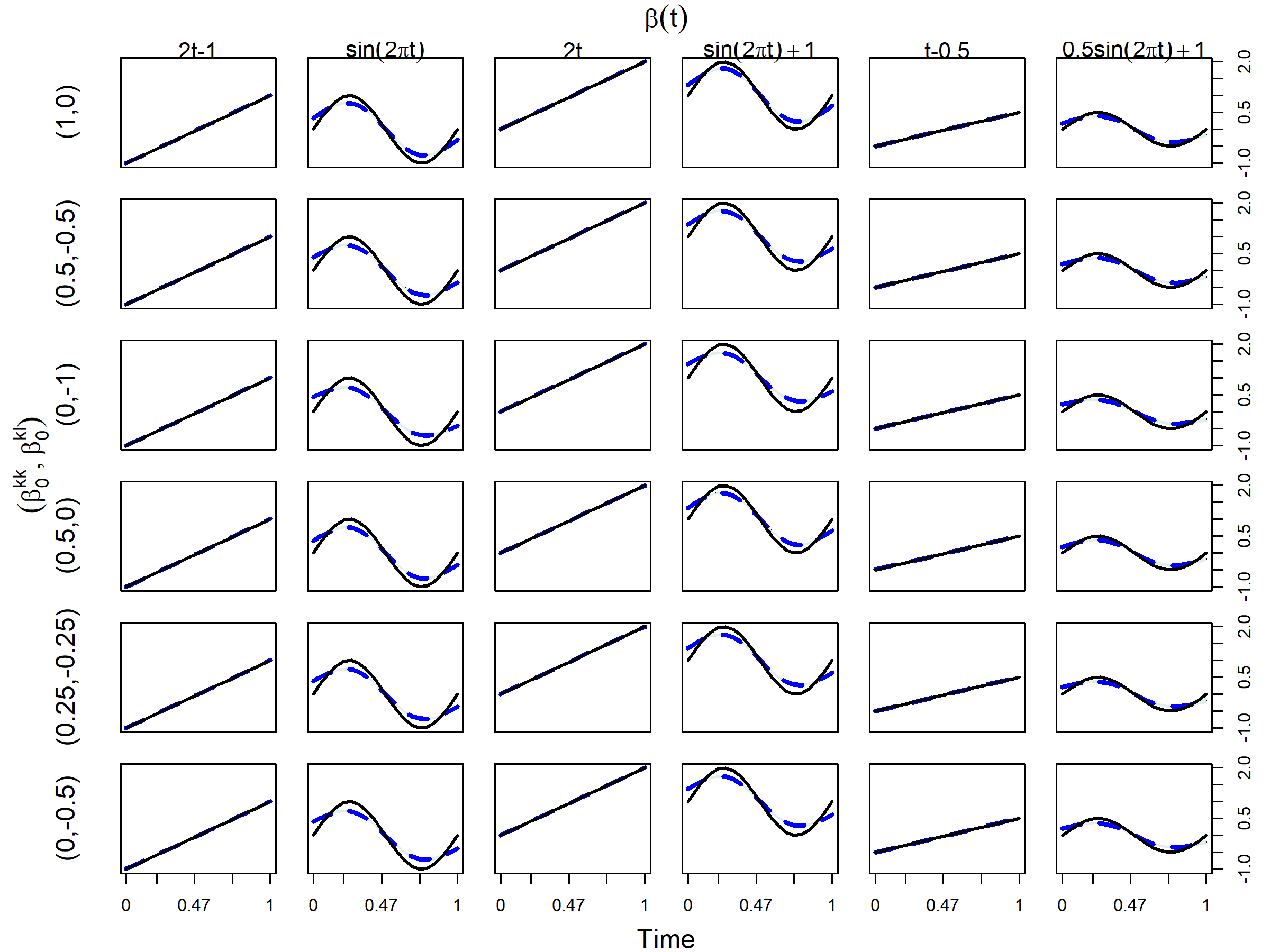}
    \caption{Estimations of the time-varying coefficients for 36 designs with $\lambda=1$, i.e. six block matrices by six functions for $\beta(t)$. In each panel, the black solid line represents the true $\beta(t)$ while the blue dashed line denotes the mean curve of $\hat{\beta}(t)$ and the light blue shadow marks the corresponding confidence band. }
    \label{fig:sim_TV_betat_lam1}
\end{figure}

\begin{table}[ht] %htp
\centering
\begin{tabular}{c|c|cccccc}
\hline
\multirow{2}{*}{$(\beta_0^{kk},\beta_0^{kl})$} &  &\multicolumn{6}{c}{$\beta(t)$} \\ 
  &           &    $2t-1$     &   $\sin (2\pi t)$     &    $2t$   &    $\sin (2\pi t)$ +1       &   $0.5(2t-1)$      &   $ 0.5 \sin (2\pi t)$  \\\hline \hline
Scenario 1 & NMI & \shortstack{\\1\\(0)} & \shortstack{\\1\\(0)} & \shortstack{\\1\\(0)} & \shortstack{\\1\\(0)} & \shortstack{\\0.996\\(0.004)} & \shortstack{\\1\\(0)} \\ \cline{2-8}
$(1,0)$ & $\text{Err}(\hat{\beta}(t))$ &\shortstack{\\0.005\\(0)} & \shortstack{\\0.008\\(0)}&\shortstack{\\0.005\\(0)} &\shortstack{\\0.008\\(0)} & \shortstack{\\0.005\\(0)}& \shortstack{\\0.006\\(0)}\\ \hline \hline
Scenario 2 & NMI  & \shortstack{\\1\\(0)}& \shortstack{\\1\\(0)}& \shortstack{\\1\\(0)}& \shortstack{\\1\\(0)}& \shortstack{\\1\\(0)}& \shortstack{\\1\\(0)} \\ \cline{2-8}
$(0.5,-0.5)$ & $\text{Err}(\hat{\beta}(t))$  &\shortstack{\\0.005\\(0)} &\shortstack{\\0.009\\(0)} & \shortstack{\\0.006\\(0)} &\shortstack{\\0.009\\(0)} &\shortstack{\\0.006\\(0)} & \shortstack{\\0.007\\(0)}\\ \hline \hline
Scenario 3 & NMI &\shortstack{\\1\\(0)} &\shortstack{\\1\\(0)} &\shortstack{\\1\\(0)} &\shortstack{\\0.996\\(0.004)} &\shortstack{\\1\\(0)} &\shortstack{\\1\\(0)}  \\ \cline{2-8}
$(0,-1)$ & $\text{Err}(\hat{\beta}(t))$  &\shortstack{\\0.006\\(0)} &\shortstack{\\0.012\\(0)} &\shortstack{\\0.006\\(0)} &\shortstack{\\0.011\\(0)} & \shortstack{\\0.007\\(0)}&\shortstack{\\0.008\\(0)} \\ \hline \hline
Scenario 4 & NMI  & \shortstack{\\0.996\\(0.004)}&\shortstack{\\1\\(0)} &\shortstack{\\1\\(0)} &\shortstack{\\1\\(0)} &\shortstack{\\1\\(0)} &\shortstack{\\1\\(0)} \\ \cline{2-8}
$(0.5,0)$ & $\text{Err}(\hat{\beta}(t))$  &\shortstack{\\0.005\\(0)} &\shortstack{\\0.009\\(0)} & \shortstack{\\0.006\\(0)}&\shortstack{\\0.008\\(0)} & \shortstack{\\0.006\\(0)}& \shortstack{\\0.007\\(0)}\\ \hline \hline
Scenario 5 & NMI  &\shortstack{\\1\\(0)} & \shortstack{\\1\\(0)}&\shortstack{\\1\\(0)} &\shortstack{\\1\\(0)} &\shortstack{\\1\\(0)} &\shortstack{\\1\\(0)} \\ \cline{2-8}
$(0.25,-0.25)$ & $\text{Err}(\hat{\beta}(t))$  &\shortstack{\\0.006\\(0)} &\shortstack{\\0.009\\(0)} & \shortstack{\\0.005\\(0)}& \shortstack{\\0.009\\(0)}&\shortstack{\\0.005\\(0)} &\shortstack{\\0.007\\(0)} \\ \hline \hline
Scenario 6 & NMI   &\shortstack{\\1\\(0)} &\shortstack{\\1\\(0)} & \shortstack{\\1\\(0)}& \shortstack{\\0.996\\(0.004)}&\shortstack{\\1\\(0)} & \shortstack{\\0.996\\(0.004)} \\ \cline{2-8}
$(0,-0.5)$ & $\text{Err}(\hat{\beta}(t))$  &\shortstack{\\0.006\\(0)} &\shortstack{\\0.010\\(0)} &\shortstack{\\0.006\\(0)} &\shortstack{\\0.010\\(0)} & \shortstack{\\0.006\\(0)}& \shortstack{\\0.007\\(0)} \\ \hline 
\end{tabular}
\caption{Summary of clustering and estimation performance (using $\lambda=0.1$) from the time-varying model over 50 simulation runs, with $\phi=1$, $\rho=1.5$ and $n=50$.}
\label{sbm-tab:TVTweedie-sensitivity_0p1}
\end{table}

\begin{figure}[ht]
    \centering
\includegraphics[height=0.32\paperheight]{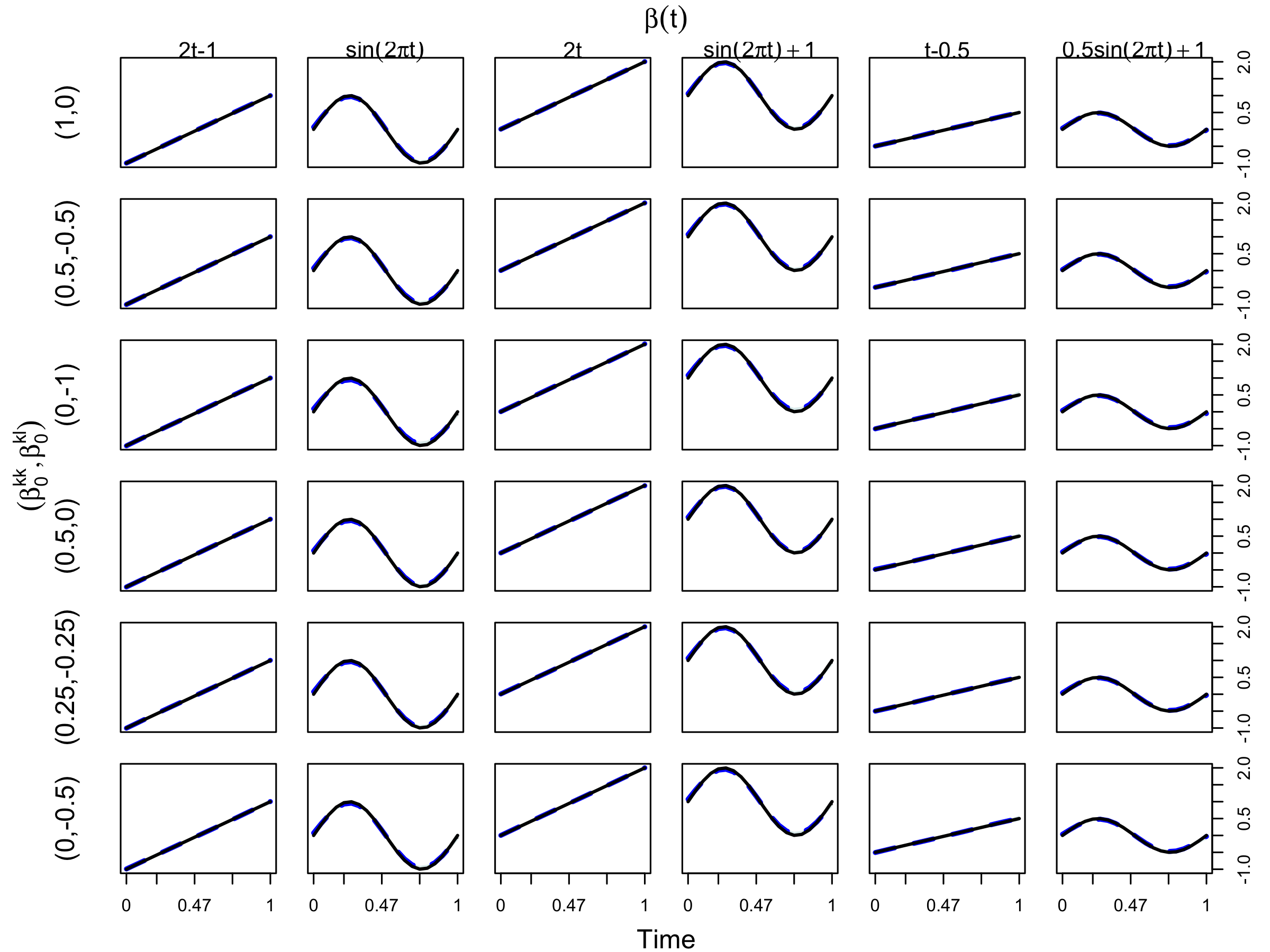}
    \caption{Estimations of the time-varying coefficients for 36 designs with $\lambda=0.1$, i.e. six block matrices by six functions for $\beta(t)$. In each panel, the black solid line represents the true $\beta(t)$ while the blue dashed line denotes the mean curve of $\hat{\beta}(t)$ and the light blue shadow marks the corresponding confidence band. }
    \label{fig:sim_TV_betat_lam0p1}
\end{figure}

\end{document}